\documentclass[twoside]{article}

\usepackage[accepted]{aistats2026}

\usepackage{microtype}

\usepackage[%
pdfstartview=FitH,%
breaklinks=true,%
bookmarks=false,%
colorlinks=true,%
linkcolor= blue,%
anchorcolor=blue,%
citecolor=blue,%
filecolor=blue,%
menucolor=blue,%
urlcolor=blue%
]{hyperref}
\usepackage{url}
\usepackage{booktabs}
\usepackage{amsfonts}
\usepackage{nicefrac}
\usepackage{xcolor}
\usepackage{amsmath}
\usepackage{algorithm,algpseudocode}
\usepackage{amsthm}
\usepackage{amssymb}
\usepackage{xspace}
\usepackage{mathtools}
\usepackage{graphicx}
\usepackage{pdflscape}
\usepackage{subcaption}
\usepackage{natbib}
\usepackage{thm-restate}
\usepackage{bbm}

\usepackage{cleveref}
\usepackage{crossreftools}
\newcounter{algline}

\crefname{algline}{line}{lines}
\Crefname{algline}{Line}{Lines}
  {\setcounter{algline}{0}\begin{algorithmic}[#1]}%
  {\end{algorithmic}}
\makeatletter
\let\ALG@oldState\State
\renewcommand{\State}{\refstepcounter{algline}\ALG@oldState}
\let\ALG@oldIf\If
\renewcommand{\If}{\refstepcounter{algline}\ALG@oldIf}
\let\ALG@oldElsIf\ElsIf
\renewcommand{\ElsIf}{\refstepcounter{algline}\ALG@oldElsIf}
\let\ALG@oldElse\Else
\renewcommand{\Else}{\refstepcounter{algline}\ALG@oldElse}
\let\ALG@oldEndIf\EndIf
\renewcommand{\EndIf}{\refstepcounter{algline}\ALG@oldEndIf}
\let\ALG@oldFor\For
\renewcommand{\For}{\refstepcounter{algline}\ALG@oldFor}
\let\ALG@oldEndFor\EndFor
\renewcommand{\EndFor}{\refstepcounter{algline}\ALG@oldEndFor}
\let\ALG@oldForAll\ForAll
\renewcommand{\ForAll}{\refstepcounter{algline}\ALG@oldForAll}

\definecolor{blue}{HTML}{0077BB}
\definecolor{cyan}{HTML}{33BBEE}
\definecolor{green}{HTML}{009988}
\definecolor{orange}{HTML}{EE7733}
\definecolor{red}{HTML}{CC3311}
\definecolor{magenta}{HTML}{EE3377}
\definecolor{grey}{HTML}{BBBBBB}

\DeclareMathOperator*{\argmin}{arg\,min}
\DeclareMathOperator*{\EE}{\mathbb{E}}
\DeclareMathOperator*{\PP}{\mathbb{P}}

\newcommand{\wrt}{\textit{w.r.t.}\xspace}
\newcommand{\ie}{\textit{i.e.}\xspace}
\newcommand{\iid}{\textit{i.i.d.}\xspace}
\newcommand{\eg}{\textit{e.g.}\xspace}

\newcommand{\KL}{\mathrm{KL}}
\newcommand{\kl}{\mathrm{kl}}
\newcommand{\lnplus}{\ln^{\!+}}
\renewcommand{\L}{L}
\newcommand{\D}{D}
\newcommand{\Lemp}{\hat{\L}}
\renewcommand{\a}{\textnormal{\textsc{a}}}
\newcommand{\La}{\L_{\a}}
\newcommand{\Laemp}{\Lemp_{\S_\a}}
\newcommand{\Laprime}{\L_{\a}}
\newcommand{\Laempprime}{\Lemp_{{\S'}_{\!\!\a}}}
\newcommand{\xbf}{\mathbf{x}}
\newcommand{\Hcal}{\mathcal{H}}
\newcommand{\Acal}{\mathcal{A}}

\newcommand{\Ncal}{\mathcal{N}}
\newcommand{\defeq}{:=}
\renewcommand{\S}{\mathcal{S}}
\newcommand{\Sp}{\S_\P}

\newcommand{\X}{\mathcal{X}}
\newcommand{\x}{\xbf}
\newcommand{\lambdaa}{\lambda_{\a}}
\newcommand{\N}{\mathbb{N}}
\newcommand{\R}{\mathbb{R}}
\newcommand{\Y}{\mathcal{Y}}
\newcommand{\XY}{\X{\times}\Y}
\newcommand{\A}{\Acal}
\newcommand{\cardA}{n}
\newcommand{\T}{\mathcal{T}}
\newcommand{\Z}{\mathcal{Z}}
\newcommand{\Dacond}{D_{|\a}}
\newcommand{\Dac}{\Dacond}
\newcommand{\M}{\mathcal{M}}

\newcommand{\prior}{\pi}
\newcommand{\Q}{Q}
\renewcommand{\P}{P}
\newcommand{\risk}{\mathcal{R}}
\newcommand{\riskempB}{\widehat{\risk}_{\batch}}
\newcommand{\riskemp}{\widehat{\risk}_{\S}}
\newcommand{\riskempprime}{\widehat{\risk}_{\S'}}
\newcommand{\riskempprimej}{\widehat{\risk}_{\S'_j}}
\newcommand{\env}{E}

\newcommand{\ma}{m_{\a}}
\newcommand{\distA}{\prior}

\newcommand{\Eemp}{\widehat{E}}
\newcommand{\empBound}{\mathcal{B}}
\newcommand{\empBoundCorUn}{\mathcal{B}_{\mathrm{cor1}}}
\newcommand{\empBoundCorDeux}{\mathcal{B}_{\mathrm{th3}}}
\newcommand{\empBoundMham}{\mathcal{B}_{\mathrm{th1}}}
\newcommand{\batch}{U}
\newcommand{\hthetatildeP}{h_{\thetatilde_{\P}}}
\newcommand{\configs}{\mathcal{C}}
\newcommand{\config}{c}
\newcommand{\thetatilde}{\tilde{\theta}}
\newcommand{\hthetatilde}{h_{\thetatilde}}
\newcommand{\thetaP}{\theta_{\P}}
\newcommand{\priorset}{\mathcal{P}}

\declaretheorem[name=Assumption]{assumption}

\declaretheorem[name=Lemma]{lemma}

\crefname{theorem}{Th.}{Ths.}
\Crefname{theorem}{Theorem}{Theorems}
\crefname{assumption}{Assum.}{Assums.}
\Crefname{assumption}{Assumption}{Assumptions}
\crefname{equation}{Eq.}{Eqs.}
\Crefname{equation}{Equation}{Equations}
\crefname{definition}{Def.}{Defs.}
\Crefname{definition}{Definition}{Definitions}
\crefname{corollary}{Cor.}{Cors.}
\Crefname{corollary}{Corollary}{Corollaries}
\crefname{algorithm}{Algo.}{Algos.}
\Crefname{algorithm}{Algorithm}{Algorithms}

\newif\ifnotappendix
\notappendixtrue

\begin{document}

\runningauthor{Hind Atbir, Farah Cherfaoui, Guillaume Metzler, Emilie Morvant, Paul Viallard}

\twocolumn[

\aistatstitle{PAC-Bayesian Bounds on Constrained $f$-Entropic Risk Measures}

\aistatsauthor{Hind Atbir$^{1,\diamond}$\\ { hind.atbir@univ-st-etienne.fr} \And Farah Cherfaoui$^{1,\diamond}$\\farah.cherfaoui@univ-st-etienne.fr\And  Guillaume Metzler$^{2}$\\guillaume.metzler@univ-lyon2.fr \AND Emilie Morvant$^{1}$\\emilie.morvant@univ-st-etienne.fr  \And Paul Viallard$^{3}$\\paul.viallard@inria.fr}

\medskip 

\aistatsaddress{
$^1$ Université Jean Monnet, CNRS, Institut d'Optique Graduate School, Laboratoire Hubert Curien UMR 5516,\\ 
Inria$^{\diamond}$,  F-42023, Saint-Etienne, France\\
$^2$ Université Lumière Lyon 2, Universite Claude Bernard Lyon 1, ERIC, 69007, Lyon, France 
\\
$^3$ Univ Rennes, Inria, CNRS IRISA - UMR 6074, F35000 Rennes, France
} ]

\begin{abstract}%
PAC generalization bounds on the risk, when expressed in terms of the expected loss, are often insufficient to capture imbalances between subgroups in the data.
To tackle this limitation, we introduce a new family of risk measures, called \emph{constrained $f$-entropic risk measures}, which enable finer control over distributional shifts and subgroup imbalances via $f$-divergences, and include the Conditional Value at Risk (CVaR), a well-known risk measure.
We derive both classical and disintegrated PAC-Bayesian generalization bounds for this family of risks, providing the first \emph{disintegrated} PAC-Bayesian guarantees beyond standard risks.
Building on this theory, we design a self-bounding algorithm minimizing our bounds directly, yielding models with guarantees at the subgroup level.
We empirically demonstrate the usefulness of our approach.
\end{abstract}%

\section{INTRODUCTION}
A machine learning task is modeled by a fixed but unknown joint probability distribution over $\XY$ denoted by $\D$, 
where~$\X$ is the input space and $\Y$ is the output space.
Given a family of hypotheses~$\Hcal$, consisting of predictors $h\!:\! \X\!\to\!\Y$, the learner aims to find the hypothesis $h\!\in\!\Hcal$ that best captures the relationship between the input space $\X$ and the output space $\Y$.
In other words, the learned hypothesis $h$ must correspond to the one that minimizes the true risk defined by 
\begin{align*}
\L(h) \defeq 
\EE_{(\xbf,y)\sim \D}\ell(y, h(\x)), 
\end{align*}
with $\ell\!:\! \Y {\times} \Y \!\to\! [0,1]$ a (measurable) loss function to assess the quality of $h$.
Since $\D$ is unknown, the true risk cannot be computed, so we need tools to estimate it and to assess the quality of the selected hypothesis $h\!\in\!\Hcal$.
To do so, a learning algorithm relies on a learning set~$\S$ composed of examples drawn \iid from~$\D$, and minimizes the empirical risk defined by 
\begin{align*}
\textstyle \Lemp(h) \defeq \Lemp_{\S}(h) \defeq \frac{1}{|\S|}\sum_{(x,y)\in\S}\ell(y, h(\x)).
\end{align*}
Thus, a central question in statistical learning theory is how well the empirical risk $\Lemp(h)$ approximates the true risk $\L(h)$.
This is commonly captured by the generalization gap defined as a deviation between $\L(h)$ and $\Lemp(h)$, which can be upper-bounded with a Probably Approximately Correct (PAC) generalization bound~\citep{valiant1984theory}. 
Several theoretical frameworks have been developed to provide such bounds, notably uniform-convergence-based bounds~\citep{bartlett2002rademacher,vapnik1971uniform}.
In this paper, we focus on the PAC-Bayesian framework~\citep{shawetaylor1997pac,mcallester1998some}, which is able to provide tight and often easily computable generalization bounds.
As a consequence, a key feature of PAC-Bayesian bounds is that they can be optimized during the learning process, giving rise to self-bounding algorithms~\citep{freund1998self}\footnote{Self-bounding algorithms have recently regained interest in PAC-Bayes (see, \eg, \citet{rivasplata2022pac,viallard2023pac}).}.
Such algorithms not only return a model but also provide its own generalization guarantee: The bound is optimized.

However, when the distribution $\D$ exhibits \mbox{imbalances}, for example, when subgroups of the population may be under (or over) represented, the classical generalization gap generally fails to capture these imbalances.
This issue arises in many practical scenarios, including class imbalance. 
In fact, when the learning set $\S$ is sampled \iid from $\D$, the imbalances are likely to be replicated, resulting in learning a hypothesis with a high error rate for underrepresented subgroups or classes.
A way to address such under-representation is to partition the data into subgroups and compute a re-weighted risk across the subgroups.
We formalize this scenario as follows.
Let $\Acal$ be an arbitrary partition of the data space $\XY$, then $\Dac$ is the conditional distribution on a subset $\a\!\in\!\Acal$, and the associated true risk on $\a$ is
\begin{align*}
\La(h) \defeq \EE_{(\x,y)\sim \Dac}\ell(y, h(\x)).
\end{align*} 
Here, we assume that the learning set is partitioned\footnote{
We assume every subgroup in $\A$ is represented in $\S$.} as $\S\! =\! \{\S_\a\}_{\a\in\A}$.
The empirical risk of a subgroup $\a$ is evaluated on $\S_\a$ of size $\ma$ with
\begin{align*}
\Laemp(h) \defeq \frac{1}{\ma} \sum_{(\x, y) \in \S_\a} \ell(y, h(\x)),
\end{align*}
More precisely, we consider the following risk measure enabling the re-weighting of the subgroups' risks\footnote{Note that \Cref{eq:sup-risk} is a distributionally robust optimization problem \citep{scarf1957min,delage2010distributionally}.}:
\begin{align}
\label{eq:sup-risk}
\risk(h) \defeq  \sup_{\rho \in E} \, \EE_{\a\sim \rho} \, \La(h), \quad \text{with}\quad  E \subseteq \M(\Acal),
\end{align}
where $\M(\Acal)$ is the set of probability measures on~$\Acal$. 
Here, $\rho$ is a probability distribution over the subgroups, controlling the weight of each subgroup loss $\La(h)$, and~$E$ denotes a set of admissible distributions.

In this paper, we go beyond previous PAC-Bayesian generalization bounds by considering a new class of risk measures, which we call \emph{constrained $f$-entropic risk measures}, that go beyond the traditional vanilla true/empirical risks.
The key idea is to constrain the set $E$ in \Cref{eq:sup-risk} to better control the subgroup imbalances while taking into account the distribution shifts thanks to a $f$-divergence.
Our definition extends the Conditional Value at Risk \citep[CVaR, see][]{rockafellar2000optimization} while keeping the flexibility of $f$-entropic risk measures \citep{ahmadi2012entropic}.
Then, we propose disintegrated (and classical) PAC-Bayesian generalization bounds for constrained $f$-entropic risk measures in two regimes:
\textit{(i)} when the set of subgroups can be smaller than the learning set, and \textit{(ii)} when there is only one example per subgroup.
Then, we design a self-bounding algorithm that minimizes our disintegrated PAC-Bayesian bound associated with each regime. 
Finally, we illustrate the effectiveness of our bounds and self-bounding algorithm in both regimes.

\textbf{Organization of the paper.}
\Cref{sec:preliminaries} introduces notations, recalls on PAC-Bayes, $f$-entropic risk measures, and related works. 
\Cref{sec:risk} defines our constrained $f$-entropic risk measures, and \Cref{sec:bounds} derives our new PAC-Bayesian bounds. 
\Cref{sec:self-bounding} presents the associated self-bounding algorithm, evaluated in \Cref{sec:expe}.

\section{PRELIMINARIES}
\label{sec:preliminaries}

\subsection{Additional Notations\protect\footnote{A summary table of notations is given in Appendix~\ref{sec:notations}.}}
 \looseness=-1
We consider learning tasks modeled by an unknown distribution $\D$ on $\XY$.
A learning algorithm is provided with a learning set $\S \!=\!\{(\x_i,y_i)\}_{i=1}^m$ consisting of $m$ examples $(\x_i,y_i)$ drawn \iid from~$\D$; we denote by~$\D^m$ the distribution of such a \mbox{$m$-sample}.
We assume $\cardA$ subgroups, defining a partition $\A\!=\!\{\a_1,\dots,\a_n\}$ of the data in~$D$.
To simplify the reading, $\a$ denotes the index of the subgroup in $\A$.
Then, we assume that the learning set 
can be partitioned into subgroups $\S \!=\! \{\S_\a\}_{\a\in\A}$, such that $\forall \a\in\A$, we have $\S_\a \!=\! \{(\x_j,y_j)\}_{j=1}^{\ma}$, and  the size of $\S_\a$ is $\ma\!\in\!\N^{*}$.
Therefore, the learner's objective is to minimize the true risks $\La(h)$ of each subgroup aggregated with the risk $\risk(h)$ as defined in \Cref{eq:sup-risk}.
The set $E$ will be further specialized in \Cref{sec:risk}.

\subsection{PAC-Bayes in a Nutshell} 

\looseness=-1
We specifically work in the setting of the PAC-Bayesian theory.
We assume a \textit{prior} distribution~$\P$ over the hypothesis space $\Hcal$, which encodes an \textit{a priori} belief about the hypotheses in $\Hcal$ before observing any data.  
Then, given $\P$ and a learning set $\S \!\sim\! \D^m$, the learner constructs a \textit{posterior} distribution $\Q_\S \!\in\! \M(\Hcal)$.  
We assume that $\Q_\S\!\ll\!\P$, \ie, the posterior $\Q_{\S}$ is absolutely continuous \wrt the prior $\P$.
In practice, this condition ensures that the corresponding densities have the same support.
Depending on the interpretation, $\Q_\S$ can be used in the two following ways.

In \textbf{classical PAC-Bayes}, $\Q_\S$ defines a randomized predictor\footnote{The randomized predictor is called the Gibbs classifier.}, which samples $h \!\sim\! \Q_\S$ for each input~$\xbf$, and then outputs $h(\xbf)$.
The generalization gap is then the deviation between the expected true risk $\EE_{h \sim \Q_\S} \L(h)$ and the expected empirical risk $\EE_{h \sim \Q_\S} \Lemp(h)$.

In \textbf{disintegrated (or derandomized) PAC-Bayes}, 
$\Q_\S\!=\!\Phi(\S,\P)$ is learned by a deterministic algorithm\footnote{More formally, $\Q_{\S}$ can be seen as a Markov kernel.}  $\Phi: (\XY)^m \!\times\! \M(\Hcal)\! \rightarrow \! \M(\Hcal)$.
Then, a single deterministic hypothesis $h$ drawn from $\Q_\S$ is considered.
This implies that the generalization gap measures the deviation between $\L(h)$ and $\Lemp(h)$ for this hypothesis $h$.

Historically, PAC-Bayes has focused on the randomized risk~\citep{shawetaylor1997pac,mcallester1998some}.
A seminal result is the bound of \citet{mcallester2003pac}, improved by \citet{maurer2004note}, stating that with probability at least $1\!-\!\delta$ over $\S\!\sim\!\D^m$, we have
\begin{align} 
\nonumber&\forall \Q \in\M(\Hcal),\\[-2.5mm]
&\EE_{h \sim \Q}L(h)  -  \EE_{h\sim \Q}\Lemp(h) \le \sqrt{\frac{ \KL (\Q\|\P) + \ln \frac{2\sqrt{m}}{\delta}}{2m}},
\label{eq:mcallester-bound}
\end{align}
where $\KL(\Q\|\P){\defeq}\!\EE_{h\sim\Q}\ln\!\big(\frac{d\Q}{d\P}(h)\big)$, and $\frac{d\Q}{d\P}$ the Radon-Nikodym derivative.
If $\Q\!\ll\!\P$, then $\KL(\Q\|\P)$ is the KL-divergence; otherwise $\KL(\Q\|\P)\!=\!+\infty$.
While the randomized risk may be meaningful 
 (\eg, when studying randomized predictors~\citep{dziugaite2017computing} or majority votes~\citep{germain2009pac}), in practice, we often deploy a single deterministic model.
To tackle this, disintegrated PAC-Bayes~\citep{blanchard2007occam,catoni2007pac,viallard2024general,viallard2024leveraging} has been proposed, where generalization  
bounds apply directly to a single hypothesis $h\! \sim\! \Q_\S$, after learning $\Q_\S$.
For instance, \citet{rivasplata2020pac} derived bounds of the following form. 
With probability at least $1\!-\!\delta$ over $\S\!\sim\!\D^m$ \textbf{and} $h\!\sim\!\Q_\S$, we have
\begin{align} 
 L(h) - \Lemp(h) \le \sqrt{\frac{1}{2m} \left[\lnplus\!\!\left(\frac{d\Q_\S}{d\P}(h)\!\right) + \ln \frac{2\sqrt{m}}{\delta}\right]},
\label{eq:rivasplata-bound}
\end{align}
where $\Q_\S\!=\!\Phi(\S,\P)$, and $\lnplus\!(\cdot) \!=\! \ln(\max(0, \cdot))$, and $\lnplus\!\big(\frac{d\Q_\S}{d\P}(h)\big)$ is the ``disintegrated'' KL-divergence.
Such results are crucial when we seek guarantees for a single deployed model $h$.

In our work, we are not interested in upper-bounding the classical gap between  $\L(h)$ and $\Lemp(h)$.
We want to study the gap between the risk measures:
\begin{align}
\nonumber &\risk(h) =   \sup_{\rho \in E} \EE_{\a\sim \rho}  \La(h)\  \mbox{ and }\ \riskemp(h) =  \sup_{\rho \in E} \EE_{\a\sim \rho}  \Laemp(h),\\
&\mbox{with }  E\subseteq  E_{\alpha} = \left\{ \rho \,  \middle| \,  \rho  \ll  \distA, \text{ and }  \frac{d\rho}{d\distA}  \le  \frac{1}{\alpha} \right\},\label{eq:single-risk}
\end{align}
with $\alpha\!\in\!(0,1]$, and $\distA$ a reference\footnote{To avoid any confusion with PAC-Bayes posterior/prior distributions, we call ``reference distribution'' the distribution $\pi$ of the (constrained) $f$-entropic risk measures.} distribution on the subgroups $\a\!\in\!\A$.
Intuitively, $\alpha$ constrains how much~$\rho$ can deviate from~$\pi$.
We derive in \Cref{sec:bounds}, classical and disintegrated PAC-Bayesian bounds; thus, we are interested in the true randomized risk measures
\begin{align}
\EE_{h\sim\Q} \risk(h) &\defeq \EE_{h\sim\Q} \ \sup_{\rho \in E}\  \EE_{\a\sim \rho}\  \La(h), \label{eq:true-risk}\\
\mbox{or}\qquad \risk(\Q) &\defeq \sup_{\rho \in E}\ \ \EE_{\a\sim \rho} \ \EE_{h\sim\Q} \La(h).
\end{align}
By Jensen's inequality, we have $\risk(\Q)\leq \EE_{h\sim\Q} \risk(h)$.
Furthermore, the associated empirical counterparts are 
\begin{align}
\EE_{h\sim\Q} \riskemp(h) &\defeq \EE_{h\sim\Q} \ \sup_{\rho \in E}\  \EE_{\a\sim \rho}\  \Laemp(h),\\
 \text{or}\qquad \riskemp(\Q) &\defeq \sup_{\rho \in E}\  \EE_{\a\sim \rho}\  \EE_{h\sim\Q} \Laemp(h).\label{eq:emp-risk}
\end{align}

\subsection{\texorpdfstring{$f$}{f}-Entropic Risk Measures in a Nutshell}

\looseness=-1
In Equations~\eqref{eq:single-risk} to~\eqref{eq:emp-risk}, we have to define the right set $E$. 
For example, we can use \mbox{$f$-divergences} \citep{csiszar1963informationstheoretische,csiszar1967information,morimoto1963markov,ali1966general} as follows.
\begin{assumption}
Let $f$ be a convex function with $f(1)\!=\!0$ and $f(0)\!=\!\lim_{t\to 0^+}f(t)$ such that $D_{f}(\rho\|\pi)\!\defeq\!\EE_{\a\sim\pi}\!\big[f\big(\frac{d\rho}{d\pi}(\a)\big)\big]$ is a $f$-divergence. 
Let $\beta\!\geq\!0$. 
We~have
\begin{align*}
 E \defeq E_{f,\beta} \defeq \left\{ \, \rho\  \middle|\  \rho \ll \pi,\ \text{and} \EE_{\a\sim\pi} f\left(\frac{d\rho}{d\pi}(\a)\!\right) \le \beta \right\},
\end{align*}
where $\ll$ denotes absolute continuity and $\pi$ is a reference distribution over $\A$.
\label{ass:classic}
\end{assumption}
\begin{restatable}{definition}{defferm}{\citep{ahmadi2012entropic}}
\label{def:f-entropic}
We say that $\risk$ of \Cref{eq:sup-risk} is a \emph{\mbox{$f$-entropic} risk measure} if $E$ satisfies \Cref{ass:classic}.
\end{restatable}
\looseness=-1
The Conditional Value at Risk (CVaR,\,\citet{rockafellar2000optimization}) is an example of a $f$-entropic risk measure.
Let $\alpha \!\in\!(0{,}1]$ and $g_{\alpha}(x) {\defeq} \iota\!\left[x \!\in\! [0{,} \frac{1}{\alpha}]\right]$ with $\iota[a]{=}0$ if $a$ is true and $+\infty$ otherwise, CVaR is defined~for
\begin{align}
\nonumber  E = E_{g_{\alpha},0} \defeq& \left\{\, \rho  \ \middle| \  \rho \ll \pi, \text{ and } \EE_{\a\sim\pi}g_{\alpha}\!\left(\frac{d\rho}{d\pi}(\a)\!\right) \le 0\right\} \\
\nonumber =& \left\{\, \rho  \ \Big| \  \rho \ll \pi, \text{ and } D_{g_{\alpha}}(\rho\|\pi) \le 0 \right\}\\
=& \left\{ \rho  \ \middle| \  \rho \ll \pi, \text{ and } \frac{d\rho}{d\pi} \le \frac{1}{\alpha}\right\} = E_{\alpha}.\label{eq:set-cvar}
\end{align}
Note that CVaR also belongs to another family of measures known as \emph{Optimized Certainty Equivalents} \citep[OCE,][]{bental1986expected,bental2007old}.\footnote{The link between $f$-entropic risk measures and OCEs is detailed in Appendix~\ref{sec:link-f-entropic-risk-oce}.}

\subsection{Related Work}

\looseness=-1
There exist some generalization bounds related to ours.
Unlike our setting, which allows partitioning $\S$ into $n$ subgroups $\A$, these existing bounds hold for $|\A|\!=\!n\!=\!m\!=\!|\S|$, \ie, there is only one example per subgroup.

\looseness=-1
Apart from PAC-Bayes bounds, generalization bounds that focus on the worst-case generalization gap, $$\sup_{h\in\Hcal}\big|\risk(h)-\riskemp(h)\big|,$$ have been introduced.
For example, \citet{curi2020adaptive} derived an upper bound on the CVaR, relying on \citet{brown2007large}'s concentration inequality.
Their bound holds either for finite hypothesis sets or for infinite hypothesis sets, but with a bound depending on covering numbers or \citet{pollard1984convergence}'s pseudo-dimension.
Another example is \citet{lee2020learning}'s generalization bound for OCEs, which relies on the Rademacher complexity associated with $\Hcal$ \citep[see, \eg,][]{bartlett2002rademacher}.
In these examples, the bounds are not easy to manipulate in practice.

\looseness=-1
The bound that is most closely related to our bounds in \Cref{sec:bounds} is the classical PAC-Bayes bound of \citet{mhammedi2020pac} on the CVaR (recalled in \Cref{thm:mhammedi}).
Their bound holds only with one example per subgroup with a uniform \textit{reference} distribution $\pi$.
\begin{restatable}[\small PAC-Bayesian Bound on CVaR \citep{mhammedi2020pac}]{theorem}{thmhammedi}
\label{thm:mhammedi}
   For any distribution $\D$ over $\XY$,
    for any prior $\P \! \in \! \M(\Hcal)$, 
    for any loss $\ell :  \Y \!\times \Y \!\rightarrow [0,1]$,
    for any $\alpha \!\in \!(0,1]$,
    for any $\delta\!\in\!(0,1]$,
   with probability at least $1 {-} \delta$ over $\S {\sim} \D^m$, we have for all $\Q \! \in \! \M(\Hcal)$, 
\begin{align*}
         \EE_{h \sim \Q}\, \risk(h) \le \riskemp(\Q) \ + \qquad \qquad \qquad \qquad \qquad \qquad \quad &\\
          2\, \riskemp(\Q)\! \left[ \! \sqrt{\frac{1}{2 \alpha m }\ln\frac{2\lceil \log_2[\frac{m}{\alpha}] \rceil}{\delta}}  \!+\!  \frac{1}{3m\alpha}\ln\frac{2\lceil \log_2[\frac{m}{\alpha}] \rceil}{\delta}\! \right]&\\
      +\ \sqrt{\frac{27}{{5 \alpha m}}\riskemp(\Q)\!\left[\KL (\Q \| P) {+} \ln\tfrac{2\lceil \log_2[\frac{m}{\alpha}] \rceil}{\delta}\right]  
        }&\\
         +\  \frac{27}{5 \alpha m} 
      \left[\KL (\Q \| P) {+} \ln\frac{2\lceil \log_2[\frac{m}{\alpha}] \rceil}{\delta}\right], &
\end{align*}
 \begin{align*}     
&\text{where}\ \ \EE_{h\sim\Q} \risk(h) \defeq \EE_{h\sim\Q}\ \ \sup_{\rho \in E}\  \EE_{(\x,y)\sim \rho}\  \ell(y, h(\x))\\
&\text{with}\ \  E\!=\!\left\{ \ \rho  \ \big| \  \rho \ll \D, \text{ and } \frac{d\rho}{d\D} \le \frac{1}{\alpha}\right\},\\
&\text{and}\ \ \riskemp(\Q) \defeq \sup_{\rho \in \widehat{E}}\ \ \EE_{\a\sim \rho}\  \EE_{h\sim\Q} \ell(y_\a, h(\x_\a)), \\
&\text{with}\ \ \widehat{E}\!=\!\left\{ \ \rho  \ \big| \  \rho \ll \pi, \text{ and } \frac{d\rho}{d\pi} \le \frac{1}{\alpha}\right\},\\
&\text{where}\ \ \pi(\a)=\frac{1}{m}.
\end{align*}
\end{restatable}
\Cref{thm:mhammedi} upper-bounds the expected true CVaR by its empirical counterpart and terms that depend on the KL-divergence between \textit{posterior} and \textit{prior} over $\Hcal$.
Note that contrary to our bounds, \Cref{thm:mhammedi} does not hold for other measures.
This is due to the proof that involves concentration inequalities tailored for CVaR, making extensions to other measures hard to obtain.

Another related framework is \textit{Group Distributionally Robust Optimization} \citep[Group DRO,][]{sagawa2019distributionally}, which considers a risk measure defined as the maximum of the expected loss on each subgroup. 
\citet{sagawa2019distributionally} proposes a learning procedure based on the principle of structural risk minimization \citep{vapnik1991principles}, where the regularization term estimates the generalization gap and mainly depends on a tunable hyperparameter.
Our work differs in two ways: \textit{(i)} we aggregate subgroup expected loss using the worst-case distribution in a ball defined by a $f$-entropic divergence and constrained by $\alpha$, \textit{(ii)} our learning procedure directly minimizes the generalization gap by optimizing PAC-Bayes bounds, yielding models with built-in generalization guarantees.

\section{CONSTRAINED \texorpdfstring{$f$}{f}-ENTROPIC RISK MEASURES}
\label{sec:risk}
In this paper, we extend the definition of the CVaR to obtain more general PAC-Bayesian generalization bounds (in \Cref{sec:bounds}) for a larger class of risk measures, which we call \textit{constrained $f$-entropic risk measures}.
We construct our new class as a restricted subclass of \mbox{$f$-entropic} risk measures by preserving their flexibility (\Cref{ass:classic}) while considering an additional constraint that controls how much the distribution $\rho$ can deviate from a given reference $\pi$ (\Cref{eq:set-cvar}).
To do so, we assume the following restricted set~$E$.
\begin{assumption} 
\label{ass:set}  
Let $f$ be defined such that $D_{f}(\rho\|\pi)$ is a $f$-divergence. Let $\beta\ge 0$ and $\alpha>0$.
We have
 \begin{align*}
&E = \left\{\, \rho\ \middle| \ \rho\ll\pi\ \text{and}\ \EE_{\a\sim\pi} f\left(\frac{d\rho}{d\pi}(\a)\!\right)\le \beta \right. \\ 
& \phantom{abcdefghijlkmnopq}\text{and}\ \left. \forall \a\in\A,\  \frac{d\rho}{d\pi}(\a)  \le \frac{1}{\alpha} \right\},
\end{align*}
with $\pi$ a reference distribution over $\A$.
\end{assumption}
\looseness=-1
Put into words, $E$ contains all distributions $\rho$ that: 
\textit{(i)}~are absolutely continuous \wrt~$\pi$; 
\textit{(ii)}~have a $f$-divergence with $\pi$ bounded by~$\beta$;
\textit{(iii)}~satisfy a uniform upper bound on the density ratio $\tfrac{d\rho}{d\pi}(\a)\! \le\! \frac1\alpha$.  
We now define the constrained $f$-entropic risk measures.
\begin{restatable}{definition}{defcfermspecial}
\label{def:constrained-f-entropic}
We say that $\risk$ or $\riskemp$ is a \emph{constrained \mbox{$f$-entropic} risk measure} if $E$ satisfies \Cref{ass:set}.
\end{restatable}
A key observation is that a constrained $f$-entropic risk measure corresponds to a standard $f$-entropic risk measure with an augmented function \mbox{$f\!+\!g_\alpha$} (with $g_\alpha$ as defined for \Cref{eq:set-cvar}). 
Indeed, $E$ can be rewritten as 
\begin{align*}
&E =\left\{\, \rho\, \middle| \, \rho\!\ll\!\pi\ \text{and} \EE_{\a\sim\pi} f\left(\frac{d\rho}{d\pi}(\a)\!\right)\le \beta \right. \\ 
& \phantom{abcdefghijklmnopqrstuv}\left.\text{and} \ \EE_{\a\sim\pi}g_{\alpha}\!\left(\frac{d\rho}{d\pi}(\a)\!\right)  \le 0 \right\}
\\
&= \left\{\, \rho\, \middle| \, \rho\!\ll\!\pi\ \text{and} \EE_{\a\sim\pi}\!\left[f\!\left(\frac{d\rho}{d\pi}(\a)\!\right) \!+\! g_{\alpha}\!\left(\frac{d\rho}{d\pi}(\a)\!\right) \right] \!\le\! \beta \right\}\\
&= \Big\{\, \rho\, \Big| \, \rho\!\ll\!\pi\ \text{and}\ D_{f{+}g_{\alpha}}(\rho\|\pi) \le \beta \Big\} = E_{f{+}g_{\alpha},\beta} \!\subseteq\! E_{\alpha},
\end{align*}
where \mbox{$f\!+\!g_\alpha$} generates the divergence $D_{f\!+\!g_{\alpha}}(\rho\|\pi)$, since it is convex, and we have $f(1)\!+\!g_{\alpha}(1)\!=\!0$, and $\lim_{t\to 0^+}f(t)\!+\!g_{\alpha}(t) \!=\! f(0)\!+\!g_{\alpha}(0)$.
Thanks to \Cref{def:constrained-f-entropic}, when $\beta\!\to\! +\infty$, the measure $\rho$ becomes less constrained by $D_f(\rho\|\pi)$, implying that $\risk(h)$ becomes the true CVaR.
Moreover, when $\alpha\!\to\! 0$, the condition \mbox{$\frac{d\rho}{d\pi}(\a)  \le \frac{1}{\alpha}$} does not restrict the set $E$.
In this case, $\risk$ of \Cref{def:constrained-f-entropic} becomes a $f$-entropic risk measure.

\section{PAC-BAYESIAN BOUNDS ON CONSTRAINED \texorpdfstring{$f$}{f}-ENTROPIC RISK MEASURES}
\label{sec:bounds}

\looseness=-1
We present our main contribution, \ie, classical and disintegrated PAC-Bayesian bounds for constrained $f$-entropic risk measures, by distinguishing two regimes.
In \Cref{sec:large}, we focus on the case where the number of subgroups is smaller than the learning set size, \ie,  $|\A|\!\le\! m$. 
For completeness, since the bound of \Cref{sec:large} becomes vacuous when $|\A|\!=\!m$, we consider, in \Cref{sec:small}, the case where each subgroup contains only one example (more specifically, one loss), \ie,  $|\A|\!=\!m$.

\subsection{When \texorpdfstring{$|\A| \le m$}{|A| ≤ m}}
\label{sec:large}

In \Cref{thm:general-ma}, we present both classical and disintegrated \textit{general} PAC-Bayesian bounds.
As commonly done in PAC-Bayes~\citep[\eg,][]{germain2009pac}, these general results are flexible since they depend on a convex deviation function $\varphi$ between true and empirical risks.
Different choices of $\varphi$ result in different instantiations of the bound, allowing us to capture the deviation in different ways. 
Our theorem below upper-bounds the deviations $\varphi\big( \riskemp(\Q), \EE_{h \sim \Q}  \risk(h) \big)$ and $\varphi\big( \riskemp(h), \risk(h) \big)$  for the classical and disintegrated settings, respectively. 
\begin{restatable}{theorem}{thmgeneralboundmaalpha}
\label{thm:general-ma}
\looseness=-1
For any distribution $\D$ on $\XY$,
for any positive, jointly convex function $\varphi(a,b)$ that is non-increasing in $a$ for any fixed $b$,
for any finite set $\A$ of $\cardA$ subgroups, 
for any $\lambdaa >0$ for each $\a\! \in\! \A$,
for any distribution $\pi$ on $\A$,
for any distribution $\P \! \in \! \M(\Hcal)$, 
for any loss $\ell\!:\! \Y  \!\times \!\Y \!\to\! [0,1]$,
for any constrained $f$-entropic risk measure $\risk$ satisfying \Cref{def:constrained-f-entropic}, 
for any $\delta \!\in\! (0,1]$,
for any $\alpha\!\in\! (0,1]$, we have the following bounds.\\[1mm]
\textbf{Classical PAC-Bayes.}
With probability at least $1 {-} \delta$ over $\S {\sim} \D^m$,
for all distributions $\Q \!\in\! \M(\Hcal)$,
we have\ifnotappendix%
\begin{align}
 \nonumber &\displaystyle   \varphi\Big( \riskemp(\Q), \EE_{h \sim \Q}  \risk(h) \Big) \le \EE_{\a \sim \pi}  \frac{1}{\alpha\,\lambdaa} \bigg[\KL(\Q\|\P) \\
&\phantom{abc} + \ln \left(\frac{\cardA}{\delta}   \EE_{\S'  \!\sim  \D^{m}} \EE_{h' \!\sim \P}e^{\lambdaa\varphi\left( \Laempprime(h'), \Laprime(h')\right)}\right)\bigg].
\label{eq:general-ma-classical} 
\end{align}
\else%
\begin{align}
\varphi\Big( \riskemp(\Q), \EE_{h \sim \Q}  \risk(h) \Big) \le \EE_{\a \sim \pi}  \frac{1}{\alpha\,\lambdaa} \bigg[\KL(\Q\|\P) + \ln \left(\frac{\cardA}{\delta}   \EE_{\S'  \!\sim  \D^{m}} \EE_{h' \!\sim \P}e^{\lambdaa\varphi\left( \Laempprime(h'), \Laprime(h')\right)}\right)\bigg].
\end{align}
\fi%
\textbf{Disintegrated PAC-Bayes.}
 For any algorithm \mbox{$\Phi: (\XY)^m \!\times\! \M(\Hcal) \rightarrow \M(\Hcal)$}, with probability at least $1-\delta$ over $\S\sim\D^m$ and $h\sim\Q_{\S}$, we have\ifnotappendix%
\begin{align}
 \nonumber &\displaystyle   \varphi\Big( \riskemp(h), \risk(h) \Big) \le \EE_{\a \sim \pi}  \frac{1}{\alpha\,\lambdaa} \bigg[\lnplus\!\left(\frac{d\Q_{\S}}{d\P}(h)\right) \\
&\phantom{abc} + \ln \left(\frac{\cardA}{\delta}   \EE_{\S'  \!\sim  \D^{m}} \EE_{h' \!\sim \P}e^{\lambdaa\varphi\left( \Laempprime(h'), \Laprime(h')\right)}\right)\bigg],%
\label{eq:general-ma-dis} 
\end{align}
\else%
\begin{align}
&\displaystyle   \varphi\Big( \riskemp(h), \risk(h) \Big) \le \EE_{\a \sim \pi}  \frac{1}{\alpha\,\lambdaa} \bigg[\lnplus\!\left(\frac{d\Q_{\S}}{d\P}(h)\right) + \ln \left(\frac{\cardA}{\delta}   \EE_{\S'  \!\sim  \D^{m}} \EE_{h' \!\sim \P}e^{\lambdaa\varphi\left( \Laempprime(h'), \Laprime(h')\right)}\right)\bigg],%
\end{align}
\fi%
where $\Q_{\S}$ is the posterior learned with $\Phi(\S,\P)$.
\end{restatable}

\begin{proof}%
(Complete proofs are deferred in Appendix~\ref{sec:general-ma-proof}.)
To prove \Cref{eq:general-ma-classical}, we start by using the proof technique of \citet{germain2009pac} to obtain a general PAC-Bayes bound on subgroup risks that holds for any $\a \in \A$: With probability at least $1-\delta$ over the random choice of $\S \sim \D^m$, we have $\forall \Q \in \mathcal{M}(\mathcal{H})$, 
\begin{align*}
    &\varphi \left( \EE_{h \sim \Q} \Laemp(h),  \EE_{h \sim \Q} \La(h)\right) \le \frac{1}{\lambda_\a} \bigg[ \KL(\Q\|\P) \\
    &\phantom{aaaaaaa}+ \ln \left( \frac{n}{\delta}  \EE_{\S' \sim \D^m} \EE_{h' \sim \P}\ e^{\lambda_{\a} \varphi\left( \Laempprime (h'), \La(h')\right)}\right) \bigg]. 
\end{align*}
Then, we apply a union bound to combine these subgroup bounds with an expectation over $\A$ using the reference $\pi$ (see \Cref{lem:general-ma}).
Finally, we use our class of measure of risks' constraint stating that $\forall \a,  \frac{d\rho}{d\pi}(\a) \le \frac{1}{\alpha}$, and we perform a change of measure to prove that
\begin{align*}
    \varphi\!\left[\riskemp (\Q), \! \EE_{h \sim \Q} \!\risk(h)\right] \!\le\! \frac{1}{\alpha}\! \EE_{\a \sim \pi}  \!\varphi\! \left[ \EE_{h \sim \Q}\! \Laemp\!(h), \!\EE_{h \sim \Q} \!\La\!(h)\right]\!,
\end{align*} 
leading to the desired result.
The proof of \Cref{eq:general-ma-dis} follows the same steps after using \citet{rivasplata2020pac}'s proof technique to get the first bound.
\end{proof}

As in \Cref{eq:mcallester-bound,eq:rivasplata-bound}, the bounds in \Cref{eq:general-ma-classical,eq:general-ma-dis} depend respectively on the KL-divergence and its disintegrated version between $\Q$ and~$\P$.
Our bounds additionally involve the parameter~$\lambdaa$, which varies \wrt the subgroup $\a\! \in\! \A$.
Interestingly, since the Radon-Nikodym derivative is uniformly bounded by~$\tfrac{1}{\alpha}$, our bounds depend only on the parameter $\alpha$ of the constrained $f$-entropic risk measure.

To make the result more concrete, we instantiate our disintegrated bound in \Cref{cor:mca-dis} with two choices of deviation $\varphi$.
For completeness, we report in Appendix (\Cref{cor:mca-pb}) the corresponding classical bounds.
First, we use $\varphi(a,b)\!=\!\kl^+\!(a \| b)$ defined, for any \mbox{$a,b \!\in\! [0,1]$}, as
\begin{align*}
\kl^+\!\left(a \middle\|b\right) \!\triangleq \! 
\left\{
\begin{array}{l}
\kl(a\|b) = a \ln \frac{a}{b}{+}(1{-}a) \ln \frac{1-a}{1-b}\  \text{if}\  a \le b, \\
 0\quad \text{otherwise.}
\end{array}\right.
\end{align*}
This quantity corresponds to the KL-divergence between two Bernoulli distributions with parameters $a$ and $b$ (truncated to $a\le b$).
Second, thanks to Pinsker's inequality, we have $2(a\!-\!b)^2 \!\le\! \kl^+(a \|b)$ for $a\!\le\! b$, which yields another (direct) bound with $\varphi(a,b)\!=\!2(a\!-\!b)^2$.
Hence, we obtain the following corollary. 
\begin{restatable}{corollary}{corboundseegermcadis}
\label{cor:mca-dis}
For any $\D$ on $\XY$,
for any $\A$ of $\cardA$ subgroups, 
for any $\pi$ over $\A$,
for any $\P \! \in \! \M(\Hcal)$, 
for any loss $\ell: \Y  \!\times \!\Y \!\to\! [0,1]$,
for any $\risk$ satisfying \Cref{def:constrained-f-entropic}, 
for any $\delta \!\in\! (0,1]$,
for any $\alpha\!\in\! (0,1]$, 
for any algorithm $\Phi: (\XY)^m\!\times\! \M(\Hcal)\! \rightarrow\! \M(\Hcal)$,
with probability at least $1 {-} \delta$ \mbox{over $\S {\sim} \D^m$} and $h\!\sim\!\Q_\S$,
we have 
\allowdisplaybreaks[4]
\begin{align}
\label{eq:seeger-ma-dis} &\kl^+\! \Big( \riskemp(h) \Big\| \risk(h) \Big)\!  \le \! \EE_{\a \sim \pi} \! \! \frac{\lnplus\!\!\left[\frac{d\Q_\S}{d\P}(h)\right] \! \! + \!  \ln\frac{2 \cardA  \sqrt{\ma}}{\delta}}{\alpha \, \ma},\\
\label{eq:mcallester-ma-dis}&\text{and}\ \risk(h)  \le  \riskemp(h) \! + \! \sqrt{ \EE_{\a \sim \pi} \!  \frac{\lnplus\!\!\left[\frac{d\Q_\S}{d\P}(h)\right]  \!+\!  \ln\frac{2 \cardA  \sqrt{\ma}}{\delta}}{2 \, \alpha \, \ma}},
\end{align}
where $\Q_{\S}$ is the posterior learned with $\Phi(\S,\P)$.
\end{restatable}
\begin{proof}
Deferred to \Cref{sec:mca-dis-proof}.
\end{proof}
Put into words, the larger the subgroup size $m_\a$, the tighter the bound. 
Conversely, smaller values of $\alpha$ make the bound looser, due both to the multiplicative factor $\frac{1}{\alpha}$ and to the fact that smaller $\alpha$ values make the constrained $f$-entropic risk measures more pessimistic.

\subsection{When \texorpdfstring{$|\A| = m$}{|A| = m}}
\label{sec:small}
\looseness=-1
When each subgroup corresponds to a single example of $\S$, the bounds of \Cref{thm:general-ma} become vacuous (since $\forall \a\!\in\!\A,$ $m_\a\!=\!1$).
To obtain a non-vacuous bound in this context, we derive bounds that take a different form.
Formally, for a learning set \mbox{$\S\!=\!\{(\xbf_\a,y_\a)\}_{\a=1}^m \!\sim\! \D^m$}, we set the reference distribution $\pi$ to be the uniform distribution over $\S$, we have 
\begin{align}
\label{eq:ass-small}
\Laemp(h) = \ell(y_\a, h(\x_\a)), \quad\text{and}\quad \pi(\a)=\frac{1}{m},
\end{align}
and we constrain the distribution $\rho$ with $\alpha$, \ie, for each $(\xbf_\a,y_\a)$.
We obtain the following PAC-Bayes bounds.

\begin{restatable}{theorem}{thmboundoneexample}
For any $\D$ on $\XY$,
for any $\lambda \!>\!0$,
for any $\P\!  \in \! \M(\Hcal)$, 
for any loss $\ell\!:\!\Y\!\times\!\Y\!\to\! [0,1]$,
for any constrained $f$-entropic risk measure $\riskemp$ satisfying \Cref{def:constrained-f-entropic} \emph{and} \Cref{eq:ass-small},
for any algorithm $\Phi: (\XY)^m \!\times \!\M(\Hcal) \!\rightarrow \!\M(\Hcal)$, 
for any $\delta\! \in\! (0,1]$, we have the following bounds.\\[1mm]
\textbf{Classical PAC-Bayes.}
With probability at least $1\!-\!\delta$ over $\S\!\sim\!\D^m$, we have
\ifnotappendix%
\begin{align}
&\left|\, \EE_{h\sim\Q_{\S}} \riskemp(h)  - \EE_{h\sim\Q_{\S}} \EE_{\S' \sim \D^m} \riskempprime(h)\, \right|\  \leq \nonumber\\
 &\frac{1}{\alpha} 
 \sqrt{\frac{1}{2m}\!\left(\left[1 {+} \frac{1}{\lambda}\right]\! \KL(\Q_\S \| \P ) \!+\! \ln \!\left[\!\frac{2(\lambda{+}1)}{\delta}\!  \right] \!\!+\! 3.5\!\right)},\label{eq:mcallester-one}
\end{align}
\else%
\begin{align}
&\left|\, \EE_{h\sim\Q_{\S}} \riskemp(h)  - \EE_{h\sim\Q_{\S}} \EE_{\S' \sim \D^m} \riskempprime(h)\, \right|\  \leq \frac{1}{\alpha} 
 \sqrt{\frac{1}{2m}\!\left(\left[1 {+} \frac{1}{\lambda}\right]\! \KL(\Q_\S \| \P ) \!+\! \ln \!\left[\!\frac{2(\lambda{+}1)}{\delta}\!  \right] \!\!+\! 3.5\!\right)},
\end{align}
\fi%
where $\Q_{\S}$ is the posterior learned with $\Phi(\S,\P)$.\\[3mm]
\textbf{Disintegrated PAC-Bayes.}
With probability at least $1\!-\!\delta$ over $\S\!\sim\!\D^m$ and $h\!\sim\! \Q_\S$, 
we have%
\ifnotappendix%
\begin{align}
& \left|\, \riskemp(h)  - \EE_{\S' \sim \D^m} \riskempprime(h)\, \right|\  \leq \nonumber\\ 
 &\frac{1}{\alpha}\sqrt{\frac{1}{2m}\! \left( \left[1 {+} \frac{1}{\lambda} \right] \lnplus\!\!\left[\frac{d\Q_\S}{d\P}(h)\right]\! +\! \ln \!\left[\!\frac{2(\lambda{+}1)}{\delta}  \!\right] \right)},\label{eq:mcallester-one-dis}
\end{align}
\else%
\begin{align}
& \left|\, \riskemp(h)  - \EE_{\S' \sim \D^m} \riskempprime(h)\, \right|\  \leq \frac{1}{\alpha}\sqrt{\frac{1}{2m}\! \left( \left[1 {+} \frac{1}{\lambda} \right] \lnplus\!\!\left[\frac{d\Q_\S}{d\P}(h)\right]\! +\! \ln \!\left[\!\frac{2(\lambda{+}1)}{\delta}  \!\right] \right)},
\end{align}
\fi%
where $\Q_{\S}$ is the posterior learned with $\Phi(\S,\P)$.
\label{thm:bound-for-one-example}
\end{restatable}
\begin{proof}
\looseness=-1
    (Complete proofs are deferred in Appendix~\ref{sec:bound-for-one-example-proof}.)
    With McDiarmid's inequality, we prove the next concentration inequality for constrained $f$-entropic measures (see \Cref{lem:brown}), holding for a fixed hypothesis $h\in \mathcal{H}$, with probability of at most $\delta$ on $\S \sim \D^m$, we have
    \begin{align*}
    |\riskemp(h)  - \EE_{\S' \sim \D^m}  \riskemp(h)|\ge \frac{1}{\alpha} \sqrt{\frac{\ln (2/\delta)}{2m}}.
    \end{align*}
    Then, we use Occam's hammer \citep{blanchard2007occam} to make the disintegrated KL and the sampling $h\!\sim\!\Q_{\S}$ appear and obtain the disintegrated bound of \Cref{eq:mcallester-one-dis}.
    The bound of \Cref{eq:mcallester-one} then follows by plugging \Cref{eq:mcallester-one-dis} into the argument of \citet{blanchard2007occam} that recovers a classical PAC-Bayesian bound from a disintegrated one.
\end{proof}
The proof of \Cref{thm:bound-for-one-example} follows \citet{blanchard2007occam}'s technique for the classical generalization gap.
Unlike \Cref{thm:general-ma}, \Cref{thm:bound-for-one-example} is not a general PAC-Bayesian theorem (\ie, it does not involve a deviation~$\varphi$), but it is a \textit{parametrized} PAC-Bayes bound with parameter $\lambda$ which controls the trade-off between the concentration terms and the KL-divergence, and which is independent of the risk measure and the subgroups. 
Moreover, the classical PAC-Bayes bound of \Cref{eq:mcallester-one} derives from the disintegrated one, so it holds only for the posterior $\Q_\S$ learned from~$\S$.
Finally, we recall that \Cref{cor:mca-dis} suffers from subgroup sizes $m_\a$ when some $m_\a$ are small, due to the $\tfrac{1}{m_\a}$ term.
In contrast, the bounds of \Cref{thm:bound-for-one-example} only depend on the global sample size $m$ with a $\tfrac{1}{m}$ term, as in standard PAC-Bayesian bounds.

\textbf{Comparison with \Cref{thm:mhammedi}.} 
We compare the two classical PAC-Bayes bounds, \Cref{eq:mcallester-one} and the one of \citet{mhammedi2020pac} (see \Cref{thm:mhammedi}).
Even though the generalization gaps of the two bounds do not involve the same quantities, we can compare the rates.
Interestingly, when $\riskemp(\Q)\!>\!0$, which is a reasonable assumption in practice, our bound is asymptotically tighter, with a rate of $\mathcal{O}(\sqrt{1/m})$ compared to their $\mathcal{O}(\sqrt{(\ln \ln m)/m})$.
Importantly, our work establishes the first disintegrated PAC-Bayesian bounds that are not the vanilla true/empirical risk $\L(h)$ and $\Lemp(h)$. 
This yields a key practical advantage: The empirical CVaR becomes computable. 
In contrast, \Cref{thm:mhammedi} relies on the computation of $\riskemp(\Q)$, which can only be estimated and for which no standard concentration inequality (\eg,  Hoeffding's inequality) provides a non-vacuous bound.
Additionally, although our bound can suffer from the $\frac{1}{\alpha^2}$ factor (larger than the $\frac{1}{\alpha}$ factor in \Cref{thm:mhammedi}), we observe in practice that our disintegrated bound remains at least comparable.

\section{SELF-BOUNDING ALGORITHMS}
\label{sec:self-bounding}

\looseness=-1
Our bounds are general, as they do not impose any algorithm for learning the posterior.
In the following, we have two objectives: \textit{(i)} in this section, designing a self-bounding algorithm~\citep{freund1998self} to learn a model by directly minimizing our bounds, and 
\textit{(ii)} in \Cref{sec:expe}, showing the usefulness of our bounds on two types of subgroups (one class per group, one example per group). 
A self-bounding algorithm outputs a model together with its own \mbox{non-vacuous} generalization bound: the one optimized. 
For practical purposes, we focus on an algorithm for disintegrated bounds, since they apply to a deterministic model.
Indeed, we recall that \textit{(i)}~classical PAC-Bayes bounds hold for a randomized model over the entire hypothesis space, which incurs additional computational cost, and \textit{(ii)}~the measure $\riskemp(\Q)$ involved in the classical bounds (\eg, \citet{mhammedi2020pac}) is not directly computable, unlike $\riskemp(h)$ in our disintegrated bounds (we detail the objective functions associated with our bounds in Appendix~\ref{bounds:expe}).

\Cref{alg:self-bounding} below summarizes the bound's minimization procedure\footnote{\Cref{alg:self-bounding} follows a quite standard procedure to minimize a bound, but is specialized to our setting.}. 
We parametrize the posterior distribution denoted by $\Q_\theta$ and we update the parameters $\theta$ by (a variant of) stochastic gradient descent as follows.
For each epoch and mini-batch $U\!\subset \!\S$ (Lines \ref{line:epoch}-\ref{line:minibatch}), we draw a model $\hthetatilde$ from the current posterior distribution $\Q_\theta$ (\Cref{line:model}). 
Then, we compute the empirical risk $\riskempB(\hthetatilde)$ of $\hthetatilde$ on $U$ (\Cref{line:compute-risk}), which is used to compute the bound, denoted $\empBound$ (\Cref{line:compute-bound}), and we update the parameters $\theta$ of the posterior distribution using the gradient $\nabla_{\theta}\empBound(\riskempB(\hthetatilde), \Q_{\theta}, \hthetatilde)$ (\Cref{line:compute-gradient}). 
Finally, we return a model drawn from the learned $\Q_{\theta}$ (\Cref{line:final-model}).

\begin{algorithm}[ht]
\caption{Self-bounding algorithm for constrained $f$-entropic risk measures}
\label{alg:self-bounding}
\begin{algorithmic}[1]
  \Require Set $\S\!=\!\{(\xbf_i,y_i)\}_{i=1}^m$, 
  number of epochs $T$, 
  variance $\sigma^2$, 
  prior $\P \!=\! \Ncal(\thetaP, \sigma^2 I_d)$ where $\thetaP \in \R^d$, 
  bound $\empBound$, 
  reference $\pi$,  
  parameters $\alpha$, $\beta$
  \State Initialize $\theta \gets \thetaP$ \label{line:init} 
  \For{$t = 1$ \textbf{to} $T$\label{line:epoch}}
    \ForAll{mini-batches $\batch \subset \S$ drawn \wrt $\pi$\label{line:minibatch}}
        \State Draw a model $\hthetatilde$ from $\Q_{\theta} \!=\! \Ncal(\theta, \sigma^2 I_d)$ \label{line:model}
      \State Compute the risk $\riskempB(\hthetatilde)$ on the mini-batch\!\!\!\!\label{line:compute-risk}
      \State Compute the bound $\empBound(\riskempB(\hthetatilde), \Q_{\theta}, \thetatilde)$
      \label{line:compute-bound} 
      \State \mbox{Update $\theta$ with gradient $\nabla_{\!\theta}\empBound(\riskempB(\hthetatilde), \Q_{\theta}, \thetatilde)$\!\!\!\!\!}\label{line:compute-gradient}
    \EndFor
  \EndFor
  \State Draw a model $h_{\hat{\theta}}$ from $\Q_{\theta}$ \label{line:final-model}
  \State \textbf{return} $h_{\hat{\theta}}$
\end{algorithmic}
\end{algorithm}

\looseness=-1
\textbf{On the prior distribution $\P$.}
A key ingredient of PAC-Bayesian methods is the choice of $\P$ (which can be set to uniform by default).
Here, we adopt a different, but classical, approach \citep[\eg,][]{ambroladze2006tighter,germain2009pac,parradohernandez2012pac,perezortiz2021tighter,dziugaite2021role,viallard2024general}: The prior $\P$ is learned from an auxiliary set $\Sp$, disjoint from the learning set $\S$ (often obtained by a 50/50 split).
Here, we learn the parameters $\thetaP$ of the prior distribution with a variant of \Cref{alg:self-bounding}: 
We remove the bound computation (\Cref{line:compute-bound}), replace the gradient in \Cref{line:compute-gradient} by $\nabla_{\thetaP} \riskempB(\hthetatildeP)$, and keep the rest unchanged.
Concretely, for each mini-batch $\batch \!\subset\! \Sp$ (Lines \ref{line:epoch}-\ref{line:minibatch}), we sample $\hthetatildeP$ from  $\P_\theta \!=\!\Ncal_\P(\thetaP, \sigma^2 I_d)$, evaluate $\riskempB(\hthetatildeP)$ (\Cref{line:compute-risk}), and update $\thetaP$ with the gradient $\nabla_{\thetaP} \riskempB(\hthetatildeP)$. 
Instead of returning a model sampled from the final $\P_\theta$ (\Cref{line:final-model}), we output the prior~$\P$ parametrized by the best-performing $\thetaP$ over the epochs and across the hyperparameter grid search.

\section{EXPERIMENTS\protect\footnote{The code is available at \url{https://gitlab.com/hatbir/aistats26-pb-cferm}.}}
\label{sec:expe}

\begin{figure*}[t]
\centering
\includegraphics[width=\textwidth]{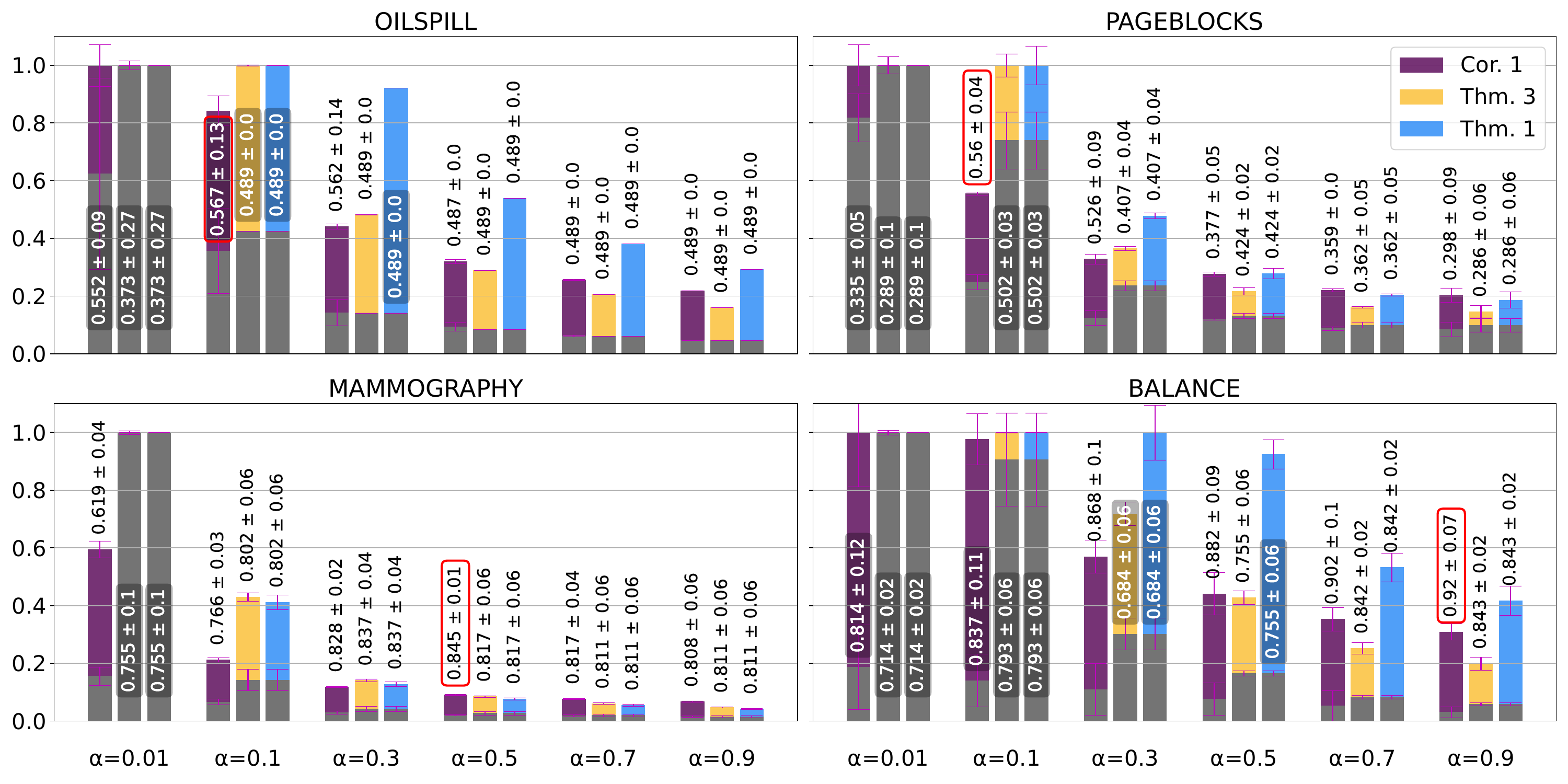}
    \caption{Bound values (in color), test risk $\risk_{\T}$ (in grey), and F-score value on $\T$ (with their standard deviations) for \Cref{thm:bound-for-one-example}, \Cref{cor:mca-dis}, and \Cref{thm:mhammedi}, as a function of $\alpha$ (on the $x$-axis).
    The $y$-axis corresponds to the value of the bounds and test risks.
    The highest F-score for each dataset is emphasized with a red frame.}
    \label{fig:bounds}
\end{figure*}

\begin{figure}[ht]
\centering
\includegraphics[trim=0 20 0 20,width=\columnwidth]{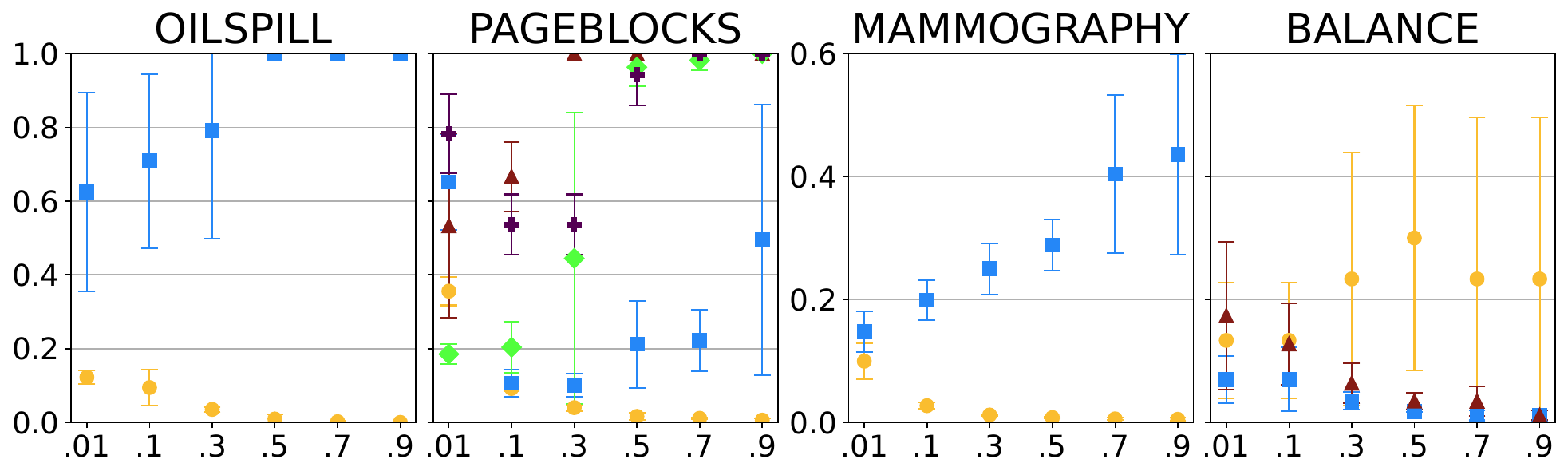}
  \caption{Evolution of the class-wise error rates and standard deviation on the set $\T$ ($y$-axis) as a function of the parameter $\alpha$ ($x$-axis) with \Cref{cor:mca-dis}. Each class is represented by different markers and colors.}
    \label{fig:risks}
\end{figure}

We now illustrate the potential of our PAC-Bayes bounds for constrained $f$-entropic risk measures with the CVaR, focusing on imbalances in the classical class-imbalance setting.
To do so, we study the behavior of our self-bounding algorithm with our bounds in \Cref{eq:mcallester-ma-dis} (\Cref{cor:mca-dis}, with one group corresponding to a class, \ie, $|\A|\!=\!|\Y|\!\leq\! m$), and \Cref{eq:mcallester-one} (\Cref{thm:bound-for-one-example}, with one example per group, \ie, $|\A|\!=\!m$), with \citet{mhammedi2020pac}'s bound (\Cref{thm:mhammedi}), and discuss their potential.
Before analyzing our results, we present our general experimental setting (more details are given in Appendix~\ref{sec:expe-details}).
 
\textbf{Datasets.} 
We report results for the 4 most imbalanced datasets we considered \citep[taken from OpenML,][]{vanschoren2013openml}: \textit{Oilspill} {\small (class ratio .96/.04)} \citep{kubat1998machine}, \textit{Mammography} {\small(.98/.02)}, \textit{Balance} {\small(.08/.46/.46)} \citep{siegler1976balance}, and \textit{Pageblocks} {\small (.90/.06/.01/.02/.02)} \citep[][]{malerba1994page}.
Each dataset is split into a training set ($\S'$) and a test set ($\T$) with a $80\%/20\%$ ratio.
Following our PAC-Bayesian \Cref{alg:self-bounding}, we  split $\S'$ into two disjoint sets $\S$ and $\Sp$ with a $50\%/50\%$ ratio; $\S$ is used to learn the posterior $\Q_\theta$ and $\Sp$ to learn the prior $\P$.
All the splits preserve the original class ratio.
Note that each experiment is repeated $3$ times with random splits. 

\looseness=-1
\textbf{Models \& distributions.}
We consider neural networks with 2 hidden layers of size $128$ (a 2-hidden-layer multilayer perceptron), with leaky ReLUs activations.
To learn the prior $\P\!=\!\Ncal(\thetaP, \sigma^2I_d)$, \ie, $\thetaP$, we initialize the parameters with a Xavier uniform distribution \citep{glorot2010understanding}, then, to learn the posterior distribution $\Q_\theta\!=\!{\cal N}(\theta,\sigma^2I_d)$, the parameters are initialized with $\thetaP$ (\Cref{line:init} of \Cref{alg:self-bounding}), and $\sigma^2\!=\!10^{-6}$.

\textbf{Risk.} We recall that we compare two regimes with the CVaR as the risk measure: \textit{(i)} for \Cref{cor:mca-dis} when $\A\!\leq\!m$ with $\A$ defined by classes, \ie, for all $y\in\Y$, we have a subgroup $\S_\a\!=\!\{(\xbf_j,y)\}_{j=1}^{\ma}$, with the reference $\pi$ set to the class ratio, and \textit{(ii)} for \Cref{thm:bound-for-one-example} and \Cref{thm:mhammedi} when $\A\!=\!m$ where each subgroup is a single example, \ie, $\A\!=\!\S\!=\!\{(\xbf_\a,y_\a)\}_{\a=1}^m$ with $\pi$ set to the uniform distribution.
The CVaR is computed with bounded cross-entropy of \citet{dziugaite2018data} as the loss, with parameter $\ell_{\text{max}} \!=\! 4$ (the loss is rescaled to $[0,1]$ to align with the theoretical assumptions).
To solve the maximization problem associated with \Cref{eq:emp-risk}, we use the python library \textit{cvxpylayers} \citep{agrawal2019differentiable} that creates differentiable convex optimization layers.
This layer is built on top of \textit{CVXPY}~\citep{diamond2016cvxpy}; We use the optimizer SCS~\citep{odonoghue2023scs} under the hood, with $\varepsilon\!=\!10^{-5}$ and a maximum of $100000$ iterations.
In additional experiments, in Appendix~\ref{sec:additional-expe}, we provide results with $\pi$ as the uniform distribution, and 
for another constrained $f$-entropic risk measure (a constrained version of the EVaR \citet{ahmadi2012entropic}).

\looseness=-1
\textbf{Bound.} We compare our disintegrated bounds of \Cref{cor:mca-dis} and \Cref{thm:bound-for-one-example} with an estimate of \citeauthor{mhammedi2020pac}'s bound (\Cref{thm:mhammedi}), obtained by sampling a single model from the posterior $\Q_\theta$.
We think this estimation is reasonable, since our bounds also rely on a single model sampled from $\Q_\theta$, and since 
 \Cref{thm:mhammedi}'s bound is harder to estimate, as it requires sampling and evaluating a large number of models to estimate the expectation over $\Q_\theta$. 
 For all bounds, we fix $\delta=0.05$ and for \Cref{thm:bound-for-one-example} we fix $\lambda=1$.
The details of the bounds' evaluation are deferred in Appendix~\ref{bounds:expe}.

\looseness=-1
\textbf{Optimization.} We use the Adam optimizer~\citep{kingma2014adam}. 
We set the parameters $\beta_1$ and $\beta_2$ to their default values in PyTorch.
For each experiment, we learn 3 prior distributions with $\Sp$ using learning rates in $\{0.1, 0.01, 0.001\}$, with $20$ epochs.
We select the best-performing prior (according to the same loss used for optimization) on $\S$ to compute the bound.
To learn the posterior on $\S$ we set the learning rate to $10^{-8}$, and the number of epochs to $10$.
We fix the batch size to $256$.

\looseness=-1
\textbf{Analysis.}
\Cref{fig:bounds} exhibits the bounds values computed on $\S$, along with the CVaR computed on the test set $\T$, highlighting the tightness of the bounds as a function of $\alpha\!\in\!\{0.01,0.1,0.3,0.5,0.7,0.9\}$. 
To give additional information on the performance of the models, and since the CVaR is not necessarily easy to interpret on its own, we report the F-score on $\T$.

\looseness=-1
First of all, as expected, \Cref{fig:bounds} shows that $\alpha$ strongly influences the tightness of the bounds: the higher $\alpha$, the tighter the bounds. 
This is not only due to the factor $\frac{1}{\alpha}$ or $\frac{1}{\alpha^2}$ in the bounds, but also because a larger $\alpha$ makes the CVaR tighter. 
Indeed, when $\alpha\!\to\! 0$, the CVaR acts as a supremum over the subgroups and puts all the weights on the subgroup that has the highest risk, while when $\alpha\!\to\!1$ the weights are equal to the reference.
This highlights that $\alpha$ plays an important role in the behavior of the algorithm and needs to be chosen to balance the predictive performance and the theoretical guarantee.
To confirm this, \Cref{fig:risks} reports class-wise error rates on the test set $\T$ as a function of $\alpha$ when optimizing \Cref{cor:mca-dis}'s bound (since it provides the best F-score).
We observe that depending on the dataset and on the value of $\alpha$, the class-wise error rates move closer or farther apart.
The main issue with the choice of $\alpha$ is that we cannot select the one associated with the lowest bound, since \textit{(i)} the tightest bound is obtained when $\alpha\!=\!1$, and \textit{(ii)} we observe on \Cref{fig:bounds} that the value of $\alpha$ leading to the tightest bounds does not yield the best F-score.

If we compare \Cref{thm:mhammedi,thm:bound-for-one-example} (which uses the same subgroups defined by one example), as expected, our bound is generally tighter (or very close for \textit{mammography}), for all values of $\alpha$.
Remarkably, when $\alpha\!\in\!\{0.01,0.1,0.3\}$, \Cref{cor:mca-dis} gives the smallest bound, and it continues to give non-vacuous and competitive bounds as long as $\alpha$ remains relatively high despite the $\frac{1}{\alpha\ma}$ term in the bound.
Moreover, as mentioned previously, \Cref{cor:mca-dis} gives the best F-score, confirming the interest of capturing the subgroups in $\S$ with our constrained $f$-entropic risk measures to tackle the imbalance better.

\begin{figure}[t]
\centering
\includegraphics[width=\columnwidth]{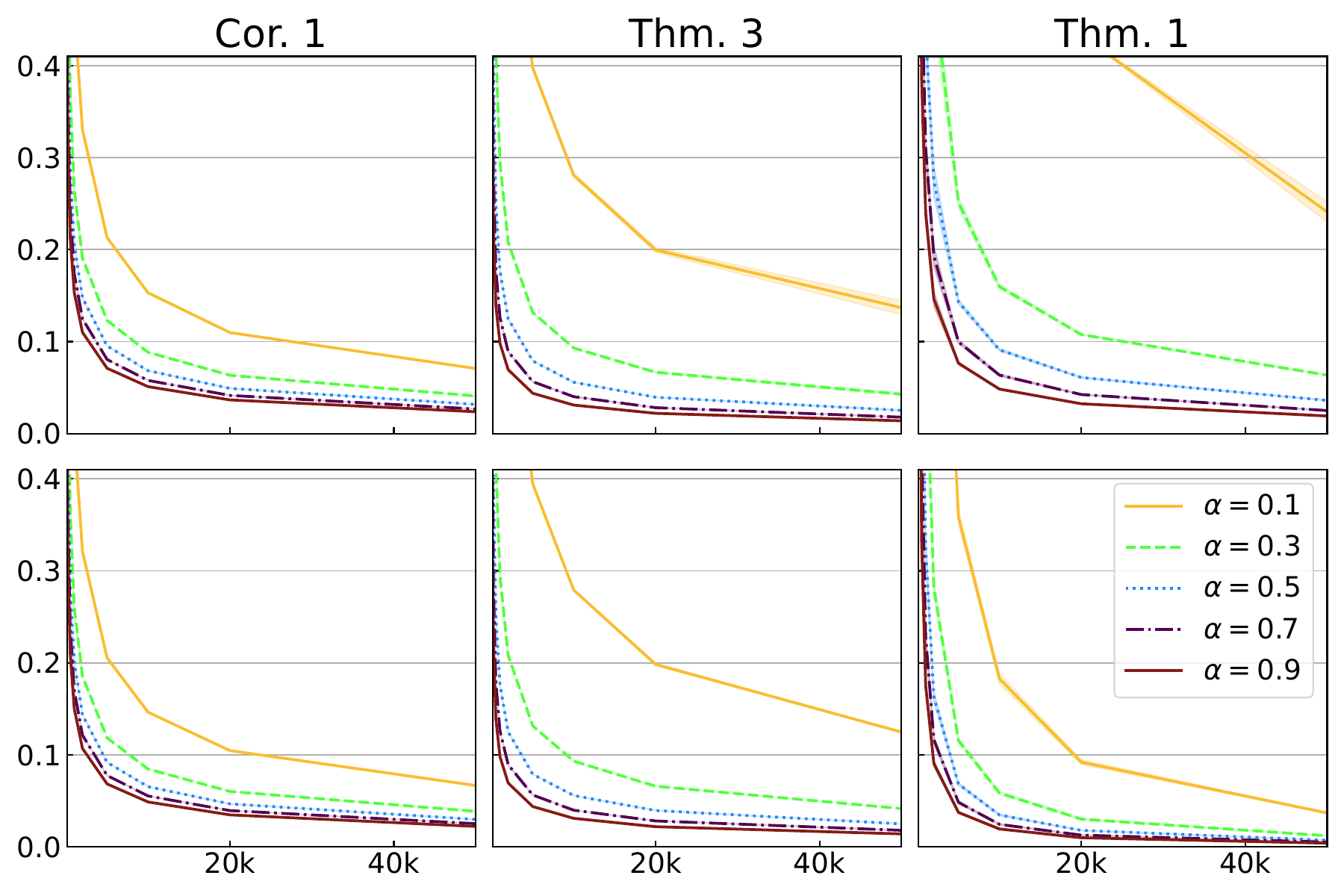}
  \caption{Evolution of the difference between the bound values and the empirical risks as a function of $m$ for \Cref{thm:bound-for-one-example,thm:mhammedi}.
  On the top figures the classes are balanced when $m$ varies.
 On the bottom figures one class size is fixed to $50$ when the other varies.}
    \label{fig:synthetic}
\end{figure}

\textbf{Synthetic experiments.}
 We report in \Cref{fig:synthetic} additional experiments on synthetic data, to show the convergence of the bounds of \Cref{thm:bound-for-one-example,thm:mhammedi}, and \Cref{cor:mca-dis} to the associated empirical risk $\riskemp$ as a function of the number $m$ of examples in $\S$ from $100$ to $100,000$ (with a test set size of $m$).
We consider two settings (with $3$ random seeds) \textit{(on the top figure)} when the classes are perfectly balanced \textit{(on the bottom figures)} when one class has a fixed size of $50$ and the other one varies. 
We observe that for any value of $\alpha$ and even in imbalanced settings, the bound values tend to the empirical risk when $m$ increases.

\section{CONCLUSION}
\looseness=-1
In this paper, we introduce classical and disintegrated PAC-Bayesian generalization bounds for a broad new family of risks, namely the constrained $f$-entropic risk measures. 
We show that the computable terms of the disintegrated bounds can be minimized with a self-bounding algorithm, leading to models equipped with tight PAC-Bayesian generalization guarantees.

As direct future work, we plan to extend our algorithm to different frameworks with subgroup structure (\eg, groups defined by populations in fairness settings or by tasks in multitask learning).
Another direction is to explore the integration of alternative risk measures into self-bounding algorithms.
One could replace the $f$-divergence with Integral Probability Metrics, such as Wasserstein distance, or consider related risk measures such as group DRO or OCEs.
Finally, we believe that our work opens the door to studying the generalization properties of other measures. 
For example, we could design an extension where $\alpha$ varies across subgroups $\a\! \in\! \A$, which can be relevant, \eg, in cost-sensitive learning, or adapt (and potentially learn) $\alpha$ dynamically to better handle harder-to-learn subgroups.

\subsubsection*{Acknowledgments}
\looseness=-1
We would like to sincerely thank the reviewers for their valuable feedback.
We are also grateful to Rémi Eyraud for his support during this work and to Rémi Emonet for valuable feedback on a draft of the manuscript.
This work was supported in part by the French Project FAMOUS ANR-23-CE23-0019.
Paul Viallard is partially funded through Inria with the associate team PACTOL and the exploratory action HYPE.

\subsubsection*{References}
\renewcommand{\bibsection}{}
\bibliographystyle{apalike}
\bibliography{biblio}

@inproceedings{agrawal2019differentiable,
  author       = {Akshay Agrawal and
                  Brandon Amos and
                  Shane T. Barratt and
                  Stephen P. Boyd and
                  Steven Diamond and
                  J. Zico Kolter},
  title        = {{Differentiable Convex Optimization Layers}},
  booktitle    = {{Advances in Neural Information Processing Systems}},
  year         = {2019},
}

@article{ahmadi2012entropic,
  author       = {Amir Ahmadi-Javid},
  title        = {{Entropic value-at-risk: A new coherent risk measure}},
  journal      = {{Journal of Optimization Theory and Applications}},
  year         = {2012},
}

@article{ali1966general,
  author       = {S. M. Ali and S. D. Silvey},
  title        = {{A General Class of Coefficients of Divergence of One Distribution from Another}},
  journal      = {{Journal of the Royal Statistical Society. Series B (Methodological)}},
  year         = {1966},
}

@inproceedings{ambroladze2006tighter,
  author       = {Amiran Ambroladze and
                  Emilio Parrado{-}Hern{\'{a}}ndez and
                  John Shawe{-}Taylor},
  title        = {{Tighter PAC-Bayes Bounds}},
  booktitle    = {{Advances in Neural Information Processing Systems}},
  year         = {2006},
}

@article{bartlett2002rademacher,
  author    = {Peter Bartlett and
               Shahar Mendelson},
  title     = {{Rademacher and Gaussian Complexities: Risk Bounds and Structural Results}},
  journal   = {{Journal of Machine Learning Research}},
  year      = {2002},
}

@article{bental1986expected,
  author       = {Aharon Ben-Tal and Marc Teboulle},
  title        = {{Expected Utility, Penalty Functions, and Duality in Stochastic Nonlinear Programming}},
  journal      = {Management Science},
  year         = {1986},
}

@article{bental2007old,
  author       = {Aharon Ben-Tal and Marc Teboulle},
  title        = {{An Old‐New Concept Of Convex Risk Measures: The Optimized Certainty Equivalent}},
  journal      = {{Mathematical Finance}},
  year         = {2007},
}

@inproceedings{blanchard2007occam,
  author       = {Gilles Blanchard and
                  Fran{\c{c}}ois Fleuret},
  title        = {{Occam's Hammer}},
  booktitle    = {{Conference on Learning Theory}},
  year         = {2007},
}

@article{brown2007large,
  author       = {David B. Brown},
  title        = {{Large deviations bounds for estimating conditional value-at-risk}},
  journal      = {{Operations Research Letters}},
  year         = {2007},
}

@article{catoni2007pac,
  author       = {Olivier Catoni},
  title        = {{PAC-Bayesian supervised classification: the thermodynamics of statistical learning}},
  journal      = {{arXiv}},
  volume       = {abs/0712.0248},
  year         = {2007}
}

@article{csiszar1963informationstheoretische,
  author       = {Imre Csisz{\'a}r},
  title        = {{Eine informationstheoretische Ungleichung und ihre Anwendung auf den Beweis der Ergodizit{\"a}t von Markoffschen Ketten}},
  journal      = {{A Magyar Tudom{\'a}nyos Akad{\'e}mia Matematikai Kutat{\'o} Int{\'e}zet{\'e}nek K{\"o}zlem{\'e}nyei}},
  year         = {1963}
}

@article{csiszar1967information,
  author       = {Imre Csisz{\'a}r},
  title        = {{Information-Type Measures of Difference of Probability Distributions and Indirect Observations}},
  journal      = {{Studia Scientiarum Mathematicarum Hungarica}},
  year         = {1967}
}

@inproceedings{curi2020adaptive,
  author       = {Sebastian Curi and
                  Kfir Y. Levy and
                  Stefanie Jegelka and
                  Andreas Krause},
  title        = {{Adaptive Sampling for Stochastic Risk-Averse Learning}},
  booktitle    = {{Advances in Neural Information Processing Systems}},
  year         = {2020},
}

@article{delage2010distributionally,
  author       = {Erick Delage and Yinyu Ye},
  title        = {{Distributionally Robust Optimization under Moment
Uncertainty with Application to Data-Driven Problems}},
  journal      = {{Operations research}},
  year         = {2010},
}

@article{diamond2016cvxpy,
  author       = {Steven Diamond and
                  Stephen P. Boyd},
  title        = {{CVXPY: A Python-Embedded Modeling Language for Convex Optimization}},
  journal      = {{Journal of Machine Learning Research}},
  year         = {2016},
}

@inproceedings{dziugaite2021role,
  author       = {Gintare Karolina Dziugaite and
                  Kyle Hsu and
                  Waseem Gharbieh and
                  Gabriel Arpino and
                  Daniel M. Roy},
  title        = {{On the role of data in PAC-Bayes}},
  booktitle    = {{International Conference on Artificial Intelligence and Statistics}},
  year         = {2021},
}

@inproceedings{dziugaite2017computing,
  author       = {Gintare Karolina Dziugaite and
                  Daniel M. Roy},
  title        = {{Computing Nonvacuous Generalization Bounds for Deep (Stochastic) Neural Networks with Many More Parameters than Training Data}},
  booktitle    = {{Conference on Uncertainty in Artificial Intelligence}},
  year         = {2017},
}

@inproceedings{dziugaite2018data,
  author       = {Gintare Karolina Dziugaite and
                  Daniel M. Roy},
  title        = {{Data-dependent PAC-Bayes priors via differential privacy}},
  booktitle    = {{Advances in Neural Information Processing Systems}},
  year         = {2018},
}

@inproceedings{freund1998self,
  author       = {Yoav Freund},
  title        = {{Self Bounding Learning Algorithms}},
  booktitle    = {{Conference on Computational Learning Theory}},
  year         = {1998},
}

@inproceedings{germain2009pac,
  author       = {Pascal Germain and
                  Alexandre Lacasse and
                  Fran{\c{c}}ois Laviolette and
                  Mario Marchand},
  title        = {{PAC-Bayesian learning of linear classifiers}},
  booktitle    = {{International Conference on Machine Learning}},
  year         = {2009},
}

@inproceedings{glorot2010understanding,
  author       = {Xavier Glorot and
                  Yoshua Bengio},
  title        = {{Understanding the difficulty of training deep feedforward neural networks}},
  booktitle    = {{International Conference on Artificial Intelligence and Statistics}},
  year         = {2010},
}

@inproceedings{kingma2014adam,
  author       = {Diederik P. Kingma and
                  Jimmy Ba},
  title        = {{Adam: A Method for Stochastic Optimization}},
  booktitle    = {International Conference on Learning Representations},
  year         = {2015},
}

@article{kubat1998machine,
  author       = {Miroslav Kubat and
                  Robert C. Holte and
                  Stan Matwin},
  title        = {{Machine Learning for the Detection of Oil Spills in Satellite Radar
                  Images}},
  journal      = {{Machine Learning}},
  year         = {1998},
}

@inproceedings{lee2020learning,
  author       = {Jaeho Lee and
                  Sejun Park and
                  Jinwoo Shin},
  title        = {{Learning Bounds for Risk-sensitive Learning}},
  booktitle    = {{Advances in Neural Information Processing Systems}},
  year         = {2020},
}

@misc{malerba1994page,
  author       = {Donato Malerba},
  title        = {{Page Blocks Classification}},
  year         = {1994},
  howpublished = {UCI Machine Learning Repository},
}

@article{maurer2004note,
  author       = {Andreas Maurer},
  title        = {{A Note on the PAC Bayesian Theorem}},
  journal      = {{arXiv}},
  volume       = {cs.LG/0411099},
  year         = {2004},
}

@article{mcallester2003pac,
  author       = {David A. McAllester},
  title        = {{PAC-Bayesian Stochastic Model Selection}},
  journal      = {{Machine Learning}},
  year         = {2003},
}

@inproceedings{mcallester1998some,
  author       = {David A. McAllester},
  title        = {{Some PAC-Bayesian Theorems}},
  booktitle    = {{Conference on Computational Learning Theory}},
  year         = {1998},
}

@inproceedings{mhammedi2020pac,
  author       = {Zakaria Mhammedi and
                  Benjamin Guedj and
                  Robert C. Williamson},
  title        = {{PAC-Bayesian Bound for the Conditional Value at Risk}},
  booktitle    = {{Advances in Neural Information Processing Systems}},
  year         = {2020},
}

@article{morimoto1963markov,
  author       = {Tetsuzo Morimoto},
  title        = {{Markov processes and the H-theorem}},
  journal      = {{Journal of the Physical Society of Japan}},
  year         = {1963},
}

@misc{odonoghue2023scs,
    author       = {Brendan O'Donoghue and 
                    Eric Chu and
                    Neal Parikh and
                    Stephen Boyd},
    title        = {{SCS: Splitting Conic Solver}},
    year         = {2023}
}

@article{parradohernandez2012pac,
  author       = {Emilio Parrado{-}Hern{\'{a}}ndez and
                  Amiran Ambroladze and
                  John Shawe{-}Taylor and
                  Shiliang Sun},
  title        = {{PAC-bayes bounds with data dependent priors}},
  journal      = {{Journal of Machine Learning Research}},
  year         = {2012},
}

@article{perezortiz2021tighter,
  author       = {Mar{\'{\i}}a P{\'{e}}rez{-}Ortiz and
                  Omar Rivasplata and
                  John Shawe{-}Taylor and
                  Csaba Szepesv{\'{a}}ri},
  title        = {{Tighter Risk Certificates for Neural Networks}},
  journal      = {{Journal of Machine Learning Research}},
  year         = {2021},
}

@book{pollard1984convergence,
  author       = {David Pollard},
  title        = {{Convergence of Stochastic Processes}},
  publisher    = {Springer New York},
  year         = {1984},
}

@phdthesis{rivasplata2022pac,
  author       = {Omar Rivasplata},
  title        = {{PAC-Bayesian Computation}},
  year         = {2022},
  school       = {University College London, United Kingdom}
}

@inproceedings{rivasplata2020pac,
  author       = {Omar Rivasplata and
                  Ilja Kuzborskij and
                  Csaba Szepesv{\'{a}}ri and
                  John Shawe{-}Taylor},
  title        = {{PAC-Bayes Analysis Beyond the Usual Bounds}},
  booktitle    = {{Advances in Neural Information Processing Systems}},
  year         = {2020},
}

@article{rockafellar2000optimization,
  author       = {R. Tyrrell Rockafellar and
                  Stanislav Uryasev},
  title        = {{Optimization of Conditional Value-at-Risk}},
  journal      = {Journal of risk},
  year         = {2000}
}

@inproceedings{sagawa2019distributionally,
  author       = {Shiori Sagawa and
                  Pang Wei Koh and
                  Tatsunori B. Hashimoto and
                  Percy Liang},
  title        = {{Distributionally Robust Neural Networks for Group Shifts: On the Importance of Regularization for Worst-Case Generalization}},
  booktitle    = {{International Conference on Learning Representations}},
  year         = {2020}
}

@techreport{scarf1957min,
  author       = {Herbert E. Scarf},
  title        = {{A min-max solution of an inventory problem}},
  year         = {1957},
  institution  = {Rand Corporation}
}

@inproceedings{shawetaylor1997pac,
  author       = {John Shawe{-}Taylor and
                  Robert C. Williamson},
  title        = {{A PAC Analysis of a Bayesian Estimator}},
  booktitle    = {{Conference on Computational Learning Theory}},
  year         = {1997},
}

@misc{siegler1976balance,
  author       = {R. Siegler},
  title        = {{Balance Scale}},
  year         = {1976},
  howpublished = {UCI Machine Learning Repository},
}

@article{valiant1984theory,
  author    = {Leslie Valiant},
  title     = {{A Theory of the Learnable}},
  journal   = {{Communications of the ACM}},
  year      = {1984},
}

@article{vanschoren2013openml,
  author       = {Joaquin Vanschoren and
                  Jan N. van Rijn and
                  Bernd Bischl and
                  Lu{\'{\i}}s Torgo},
  title        = {{OpenML: networked science in machine learning}},
  journal      = {SIGKDD Explorations},
  year         = {2013},
}

@inproceedings{vapnik1991principles,
  author       = {Vladimir Vapnik},
  title        = {{Principles of Risk Minimization for Learning Theory}},
  booktitle    = {{Advances in Neural Information Processing Systems}},
  year         = {1991},
}

@article{vapnik1971uniform,
  author    = {Vladimir Vapnik and
               Alexey Chervonenkis},
  title     = {{On the Uniform Convergence of Relative Frequencies of Events to Their Probabilities}},
  journal   = {{Theory of Probability and its Applications}},
  year      = {1971},
}

@phdthesis{viallard2023pac,
  author    = {Paul Viallard},
  title     = {{PAC-Bayesian Bounds and Beyond: Self-Bounding Algorithms and New Perspectives on Generalization in Machine Learning}},
  school    = {{University Jean Monnet Saint-Etienne, France}},
  year      = {2023},
}

@inproceedings{viallard2024leveraging,
  author       = {Paul Viallard and
                  R{\'{e}}mi Emonet and
                  Amaury Habrard and
                  Emilie Morvant and
                  Valentina Zantedeschi},
  title        = {{Leveraging PAC-Bayes Theory and Gibbs Distributions for Generalization Bounds with Complexity Measures}},
  booktitle    = {{International Conference on Artificial Intelligence and Statistics}},
  year         = {2024},
}

@article{viallard2024general,
  author       = {Paul Viallard and
                  Pascal Germain and
                  Amaury Habrard and
                  Emilie Morvant},
  title        = {{A general framework for the practical disintegration of PAC-Bayesian bounds}},
  journal      = {{Machine Learning}},
  year         = {2024},
}


\section*{Checklist}

\begin{enumerate}

  \item For all models and algorithms presented, check if you include:
  \begin{enumerate}
    \item A clear description of the mathematical setting, assumptions, algorithm, and/or model.
   \textbf{Yes}
    \item An analysis of the properties and complexity (time, space, sample size) of any algorithm.
    \textbf{No}
    \item (Optional) Anonymized source code, with specification of all dependencies, including external libraries.
    \textbf{Yes}, \textit{The code is available } \href{https://gitlab.com/hatbir/aistats26-pb-cferm}{here}.
  \end{enumerate}

  \item For any theoretical claim, check if you include:
  \begin{enumerate}
    \item Statements of the full set of assumptions of all theoretical results.
   \textbf{Yes}, \textit{the assumptions are recalled in the statements of the theorems and corollaries}.
    \item Complete proofs of all theoretical results.
    \textbf{Yes}, \textit{for the sake of readability, the proofs are deferred in the Appendix.}
    \item Clear explanations of any assumptions.
    \textbf{Yes}, \textit{all the assumptions are related to notations and mathematical settings that have been introduced.}
  \end{enumerate}

  \item For all figures and tables that present empirical results, check if you include:
  \begin{enumerate}
    \item The code, data, and instructions needed to reproduce the main experimental results (either in the supplemental material or as a URL).
    \textbf{Yes}, \textit{as supplementary material.}
    \item All the training details (e.g., data splits, hyperparameters, how they were chosen).
    \textbf{Yes}, \textit{the main details are in \Cref{sec:expe} (additional details are given in Appendix).}
    \textbf{Yes}, \textit{for the sake of readability, note that the main details are given in \Cref{sec:expe} (and the complete details are given in the Appendix).}

    \item A clear definition of the specific measure or statistics and error bars (e.g., with respect to the random seed after running experiments multiple times).
    \textbf{Yes}
    \item A description of the computing infrastructure used. (e.g., type of GPUs, internal cluster, or cloud provider). 
    \textbf{No}
  \end{enumerate}

  \item If you are using existing assets (e.g., code, data, models) or curating/releasing new assets, check if you include:
  \begin{enumerate}
    \item Citations of the creator If your work uses existing assets.
    \textbf{Yes}, \textit{in \Cref{sec:expe}}.
    \item The license information of the assets, if applicable.
    \textbf{Not Applicable}
    \item New assets either in the supplemental material or as a URL, if applicable.
    \textbf{Yes}, \textit{our code is in the supplementary material.}
    \item Information about consent from data providers/curators.
     \textbf{Not Applicable}, \textit{the data used are publicly available from OpenML \citep{vanschoren2013openml}.}
    \item Discussion of sensible content if applicable, e.g., personally identifiable information or offensive content.
     \textbf{Not Applicable}
  \end{enumerate}

  \item If you used crowdsourcing or conducted research with human subjects, check if you include:
  \begin{enumerate}
    \item The full text of instructions given to participants and screenshots.
     \textbf{Not Applicable}
    \item Descriptions of potential participant risks, with links to Institutional Review Board (IRB) approvals if applicable.
    \textbf{Not Applicable}
    \item The estimated hourly wage paid to participants and the total amount spent on participant compensation.
     \textbf{Not Applicable}
  \end{enumerate}

\end{enumerate}

\clearpage
\appendix
\thispagestyle{empty}

\notappendixfalse
\onecolumn

\aistatstitle{Supplementary Materials of\\[1mm]
PAC-Bayesian Bounds on Constrained $f$-Entropic Risk Measures
}
\newenvironment{itemize*}
    {\begin{itemize}%
      \setlength{\itemsep}{1.5pt}%
      \setlength{\parskip}{0pt}}%
    {\end{itemize}}
The supplementary materials are organized as follows.\\[-8mm]
\begin{itemize*}
    \item \Cref{sec:notations} recalls the list of the main notations of the paper;
    \item \Cref{sec:link-f-entropic-risk-oce} discusses the relationship between (constrained) $f$-entropic risk measures and OCE measures;
    \item \Cref{sec:general-ma-proof,sec:klplus-proof,sec:cor-general-ma-proof,sec:bound-for-one-example-proof} contains all the proofs of our statements;
    \item \Cref{sec:expe-details} gives more details about our method and experimental setting;
    \item \Cref{sec:additional-expe} reports the associated additional empirical results;
\end{itemize*}

\section{TABLES OF NOTATION}
\label{sec:notations}
\def\arraystretch{1.35}
\paragraph{Probability theory}~\\
\begin{tabular}{cp{14cm}}
$\EE_{z\sim\mathcal{Z}}$ & Expectation \wrt the random variable $z\sim\mathcal{Z}$\\
$\PP_{z\sim\mathcal{Z}}$ &  Probability \wrt the random variable $z\sim\mathcal{Z}$\\
$\rho \ll \pi$ & $\rho$  is absolutely continuous \wrt $\pi$\\
$\displaystyle \frac{d\rho}{d\pi}$ & Radon-Nikodym derivative\\[2mm]
$\KL(\cdot\|\cdot)$ & Kullback-Leibler (KL) divergence\\
$\kl^+(a\|b)$ &  KL divergence between 2 Bernoulli distributions with param. $a$ and $b$ (truncated to $a\leq b$)\\
$\M(\Hcal)$& Set of probability measures / distributions\\
$\Ncal(\theta, \sigma^2)$ & Normal distribution with mean $\theta$ and variance $\sigma^2$
\end{tabular}

\paragraph{Main notations}~\\
\begin{tabular}{cp{13cm}}
    $\X$ & Input space\\
    $\Y$ & Output/label space\\
    $\D$ &  Data distribution over $\XY$\\
    $\D^m$ & Distribution of a $m$-sample\\
    $\displaystyle \S=\{(\xbf_i,y_i)\}_{i=1}^m\sim \D^m$ & Learning set of $m$ examples drawn \iid from $\D$\\
$\A=\{\a_1,\dots,\a_n\}$ & Partition of the data from $\D$ into $n$ subgroups\\
  $\S=\{\S_\a\}_{\a\in\A}$ & Partition of $\S$ into $n$ subgroups\\
    $\forall \a\in\A,\ \S_\a=\{(\xbf_j,y_j)\}_{j=1}^{\ma}$ & A subgroup $\S_\a$ consists of $\ma$ examples\\
    $\Dac$ & Conditional distribution on $\a\in\A$\\
    $\pi$ & Reference distribution over $\A$\\
    $\rho$ & Distribution over $\A$\\
    $\Hcal$ & Hypothesis space of predictors $h:\X\to\Y$\\
    $\P$ & (PAC-Bayesian) prior distribution over $\Hcal$\\
   $\Q$ or $\Q_\S$ & (PAC-Bayesian) posterior distribution over $\Hcal$\\
   $\Phi:(\XY)^m\!\times\!\M(\Hcal) \!\to\! \M(\Hcal)$ & Deterministic algorithm to learn $\Q_\S=\Phi(\S, P)$
    \end{tabular}
    
\paragraph{Risk measures}~\\
    \begin{tabular}{cp{6cm}}
    $\ell(\cdot,\cdot)$  &  Loss function $\Y\!\times\! \Y\!\to\! [0,1]$\\[2mm]
    $\displaystyle\L(h) = \EE_{(\xbf,y)\sim \D} \ell(y,h(\xbf))$ & Classical true risk of $h$\\
    $\displaystyle
    \Lemp_{\S}(h) = \frac{1}{m} \sum_{i=1}^m \ell(y_i,h(\xbf_i))$ & Classical empirical risk of $h$\\
    $\displaystyle\La(h) = \EE_{(\xbf,y)\sim \Dac} \ell(y,h(\xbf))$ & Classical true risk of $h$ on subgroup $\a$\\
    $\displaystyle \Laemp(h)  =  \frac{1}{\ma} \sum_{j=1}^{\ma} \ell(y_j,h(\xbf_j))$ & Classical empirical risk of $h$ on subgroup $\S_\a$ of size $\ma$\\[1mm]
\midrule
  $\displaystyle\risk(h) =  \sup_{\rho \in E} \, \EE_{\a\sim \rho} \, \La(h)$
  & True risk measure \\[2mm]
     $\displaystyle \riskemp(h) =  \sup_{\rho \in E} \, \EE_{\a\sim \rho} \, \Laemp(h)$
     & Empirical risk measure \\[2mm]
     \small with  $\displaystyle E = E_{f,\beta} \defeq \bigg\{  \rho\  \bigg|\  \rho \ll \pi\ \text{and} \EE_{\a\sim\pi} f\bigg(\frac{d\rho}{d\pi}(\a)\!\bigg) \!\le\! \beta \bigg\} $ &  $f$-entropic risk measure\\[2mm]
     \small with $\displaystyle E \!=\! E_{\alpha} \!=\! \bigg\{ \rho  \ \bigg| \  \rho \ll \pi\ \text{and}\ \frac{d\rho}{d\pi} \!\le\! \frac{1}{\alpha}\bigg\} $ & Conditional Value at Risk (CVaR)\\[2mm]
      \small with $\displaystyle E \!=\! \bigg\{ \rho\, \bigg| \, \rho\!\ll\!\pi\ \text{and} \EE_{\a\sim\pi} f\bigg(\frac{d\rho}{d\pi}(\a)\!\bigg)\!\le\! \beta \ \text{and}\  \forall \a\!\in\!\A,\,  \frac{d\rho}{d\pi}(\a)  \!\le\! \frac{1}{\alpha} \bigg\}$ & Constrained $f$-entropic risk measure\\[4mm]
      \midrule
      $\displaystyle \risk(\Q) \defeq \sup_{\rho \in E}\ \ \EE_{\a\sim \rho} \ \EE_{h\sim\Q} \La(h)$ & Randomized risk measure\\[1mm]
      $\displaystyle \EE_{h\sim\Q} \risk(h) \defeq \EE_{h\sim\Q} \ \sup_{\rho \in E}\  \EE_{\a\sim \rho}\  \La(h)$ & we have $\risk(\Q) \leq \EE_{h\sim\Q} \risk(h)$
\end{tabular}

\paragraph{Specific notations of  \Cref{sec:self-bounding}, \ie, for the self-bounding algorithm}~\\
\begin{tabular}{cp{13cm}}
$\S$ & Learning set for the posterior\\
$\S_P$ & Learning set for the prior (independent from $\S$)\\
$\T$ & Test set \\
$U\subset \S$ & A mini-batch\\
$\P=\Ncal(\theta_\P, \sigma^2 I_d)$ & Prior parametrized by $\thetaP$\\
$\Q_\theta=\Ncal(\theta, \sigma^2 I_d)$ & Posterior parametrized by $\theta$ \\
$\theta$ & Parameters of $\Q$\\
$\hthetatilde$ & Model drawn from the current $\Q_{\theta}$ at each iteration\\
$\riskempB(\hthetatilde)$ & Risk measure evaluated on the mini-batch $U$\\
$\empBound(\cdot)$ &  Objective function associated with the bound\\
$h_{\hat{\theta}}$ & The final model drawn from the final $\Q_{\theta}$
\end{tabular}

\newpage 

\section{ABOUT THE LINK BETWEEN (CONSTRAINED) \texorpdfstring{$f$}{f}-ENTROPIC RISK MEASURES AND OCES}
\label{sec:link-f-entropic-risk-oce}

In order to compare more precisely the (constrained) \texorpdfstring{$f$}{f}-entropic risk measure and the Optimized Certainty Equivalents (OCE), we first present another formulation of the \texorpdfstring{$f$}{f}-entropic risk measure, and the definition of an OCE. 

\paragraph{(Constrained) $f$-entropic risk measure.} Let $\beta \geq 0$, recall from \Cref{ass:classic} and \Cref{def:f-entropic} that true and empirical $f$-entropic risk measures are defined by
\begin{align*}
\nonumber &\risk(h) =   \sup_{\rho \in E} \EE_{\a\sim \rho}  \La(h)\  \mbox{ and }\ \riskemp(h) =  \sup_{\rho \in E} \EE_{\a\sim \rho}  \Laemp(h),\\
&\mbox{with }  E = E_{f,\beta} \defeq \left\{ \, \rho\  \middle|\  \rho \ll \pi,\ \text{and} \EE_{\a\sim\pi} f\left(\frac{d\rho}{d\pi}(\a)\!\right) \le \beta \right\},
\end{align*}
where $f$ is defined such that $D_{f}(\rho\|\pi)\!\defeq\!\EE_{\a\sim\pi}\!\big[f\big(\frac{d\rho}{d\pi}(\a)\big)\big]$ is a $f$-divergence.\\
From \citet[Theorem~5.1]{ahmadi2012entropic}, we have the following equalities:
\begin{align}
\risk(h) \!=\!\! \inf_{t>0,\mu\in\mathbb{R}}\left\{ t\left[ \mu +\! \EE_{\a\sim\pi}f^*\!\left(\frac{\La(h)}{t}\!-\!\mu\!+\!\beta\right) \right] \right\}, \ \text{and}\  \riskemp(h) \!=\!\! \inf_{t>0,\mu\in\mathbb{R}}\left\{ t\left[ \mu \!+\! \EE_{\a\sim\pi}f^*\!\left(\frac{\Laemp(h)}{t}\!-\!\mu\!+\!\beta\right) \right] \right\},\label{eq:true-emp-f-entropic}
\end{align}
where $f^*$ is the convex conjugate of $f$.
Note that these results hold also for the constrained $f$-entropic risk measure since it is a $f$-entropic risk measure as we use the divergence $f+g_\alpha$ instead of $f$; see \Cref{sec:risk}.

\paragraph{OCE Risk Measure.} According to \citet{bental1986expected,bental2007old}, an OCE is defined by 
\begin{align}
\risk^{\texttt{oce}}(h) \defeq \inf_{\mu\in\mathbb{R}}\left\{ \mu + \EE_{\a\sim\pi}f^*\left(\La(h) - \mu\right) \right\} \ \text{and}\ \riskemp^{\texttt{oce}}(h) := \inf_{\mu\in\mathbb{R}}\left\{ \mu + \EE_{\a\sim\pi}f^*\left(\Laemp(h)-\mu\right) \right\}.\label{eq:oce}
\end{align}

\paragraph{Comparison.} 
By comparing \Cref{eq:true-emp-f-entropic} and \Cref{eq:oce}, we can remark that in \Cref{eq:oce}, we have $t=1$ and $\beta=0$. Following the proof of Theorem 5.1 in \citet{ahmadi2012entropic} (with $t=1$ and $\beta=0$), we deduce that  
\begin{align*}
\risk^{\texttt{oce}}(h) &:= \inf_{\mu\in\mathbb{R}}\left\{ \mu + \EE_{\a\sim\pi}f^*\left(\La(h) - \mu\right) \right\} = \sup_{\rho \ll \pi}\left\{ \EE_{\a\sim\rho}\La(h) - D_{f}(\rho\|\pi) \right\},\\
\text{and}\quad \riskemp^{\texttt{oce}}(h) &:= \inf_{\mu\in\mathbb{R}}\left\{ \mu + \EE_{\a\sim\pi}f^*\left(\Laemp(h) - \mu\right) \right\} = \sup_{\rho \ll \pi}\left\{ \EE_{\a\sim\rho}\Laemp(h) - D_{f}(\rho\|\pi) \right\}.
\end{align*}

Hence, as we can remark, the OCE corresponds to a different optimization problem from the (constrained) $f$-entropic risk measures. 
Indeed, the OCE finds the distribution $\rho$ that maximizes $\EE_{\a\sim\rho}\La(h) - D_{f}(\rho\|\pi)$ or $\EE_{\a\sim\rho}\Laemp(h) - D_{f}(\rho\|\pi)$. 
The (constrained) \texorpdfstring{$f$}{f}-entropic risk maximizes the risk $\EE_{\a\sim \rho}  \La(h)$ or $\EE_{\a\sim \rho}  \Laemp(h)$ while keeping $D_{f}(\rho\|\pi) \le \beta$.

\section{PROOF OF THEOREM~\ref{thm:general-ma}}
\label{sec:general-ma-proof}

In this section, we give the proof of the following theorem.

\thmgeneralboundmaalpha*

We prove \Cref{eq:general-ma-classical} in \Cref{sec:general-ma-classical-proof}, and   \Cref{eq:general-ma-dis} in \Cref{sec:general-ma-dis-proof}.

\subsection{Proof of \Cref{eq:general-ma-classical}}
\label{sec:general-ma-classical-proof}

To prove \Cref{eq:general-ma-classical}, we first prove \Cref{lem:general-ma}, which follows the steps of the general proof of the PAC-Bayesian theorem by \citet{germain2009pac} and a union bound.

\begin{lemma}\label{lem:general-ma}
    For any distribution $\D$ on $\XY$,
    for any positive, jointly convex function $\varphi(a,b)$,
    for any finite set $\A$ of $\cardA$ subgroups, 
    for any $\lambdaa >0$ for each $\a\! \in\! \A$,
    for any distribution $\pi$ over $\A$,
    for any distribution $\P  \in  \M(\Hcal)$, 
    for any loss function $\ell\!:\!\Y\!\times\!\Y\!\to\! \R$,
    for any $\delta \in (0,1]$,
    for any $\alpha  \in  (0,1)$,
    with probability at least $1 {-} \delta$ over $\S {\sim} \D^m$,
    for all distributions $\Q \in \M(\Hcal)$, we have
    \begin{align*}
          \EE_{\a \sim \pi }\varphi\left(\EE_{h \sim \Q} \Lemp_{\a}(h), \EE_{h \sim \Q} \L_{\a}(h)\right)  \le \EE_{\a \sim \pi} \frac{1}{\lambda_\a}\left[\KL(\Q\|\P) + \ln\left( \frac{n}{\delta} \EE_{\S' \sim \D^m} \EE_{h' \sim \P }\ e^{\lambda_\a \varphi\left( \Laempprime(h'), \Laprime(h')\right)} \right) \right].
    \end{align*}
\end{lemma}
\begin{proof}
First of all, our goal is to upper-bound $\lambda_\a\varphi\left(\EE_{h \sim \Q} \Laemp(h), \EE_{h \sim \Q} \La(h)\right)$ for each $\a\in\A$. To do so, we follow the steps of \citet{germain2009pac}.
From the Donsker-Varadhan representation of the KL divergence, we have
\begin{align}
\lambda_\a \varphi\left(\EE_{h \sim \Q} \Laemp(h), \EE_{h \sim \Q} \La(h)\right) &\le \KL(\Q \| \P) +\ln \left(  \EE_{h \sim \P }\ e^{\lambda_\a \varphi\left( \Laemp(h), \La(h)\right)} \right).\label{eq:general-ma-proof-3}
\end{align}
Now, we apply Markov's inequality to $\EE_{h \sim \P}e^{\lambda_\a \varphi(\Laemp(h), \La(h))}$, which is positive.
We have
\begingroup
\allowdisplaybreaks
\begin{align}
    &\PP_{\S \sim \D^m} \left[\EE_{h \sim \P }\ e^{\lambda_\a \varphi\left( \Laemp(h), \La(h)\right)} \le \frac{n}{\delta}\EE_{\S' \sim \D^m} \EE_{h' \sim \P }\ e^{\lambda_\a \varphi\left( \Laempprime(h'), \La(h')\right)} \right] \ge 1 - \frac{\delta}{n}\nonumber\\
    \iff \quad &\PP_{\S \sim \D^m} \left[\ln\left(\EE_{h \sim \P }\ e^{\lambda_\a \varphi\left( \Laemp(h), \La(h)\right)} \right) \le \ln\left( \frac{n}{\delta} \EE_{\S' \sim \D^m} \EE_{h' \sim \P }\ e^{\lambda_\a \varphi\left( \Laempprime(h'), \La(h')\right)} \right)  \right] \ge 1 - \frac{\delta}{n}. \label{eq:general-ma-proof-4}
\end{align}
\endgroup
Hence, by combining \Cref{eq:general-ma-proof-3} and \Cref{eq:general-ma-proof-4}, we have for any $\a \in \A$, 
\begin{align*}
    &\PP_{\S \sim \D^m} \left[ 
    \begin{array}{l}
         \displaystyle \forall \Q \in \M(\Hcal),  \\
         \displaystyle \varphi\left(\EE_{h \sim \Q} \Laemp(h), \EE_{h \sim \Q} \La(h)\right) \le \frac{1}{\lambda_\a}\left[ \KL(\Q\|\P) + \ln\left( \frac{n}{\delta} \EE_{\S' \sim \D^m} \EE_{h' \sim \P }\ e^{\lambda_\a \varphi\left( \Laempprime(h'), \La(h')\right)} \right) \right]
    \end{array}
     \right] \ge 1 - \frac{\delta}{n}.
\end{align*}
As $\A$ is finite with $|\A|=n$, we apply a union bound argument to obtain
\begingroup
\allowdisplaybreaks
\begin{align}
     \iff &\PP_{\S \sim \D^m} \left[ 
    \begin{array}{l}
         \displaystyle \forall \a \in \A,\  \forall \Q \in \M(\Hcal),  \\
         \displaystyle \varphi\left(\EE_{h \sim \Q} \Lemp_{\a}(h), \EE_{h \sim \Q} \L_{\a}(h)\right)   \\
         \displaystyle \phantom{aaaa} \le \frac{1}{\lambda_\a}\left[\KL(\Q\|\P) + \ln\left( \frac{n}{\delta} \EE_{\S' \sim \D^m} \EE_{h' \sim \P }\ e^{\lambda_\a \varphi\left( \Laempprime(h'), \Laprime(h')\right)} \right) \right]
    \end{array}
     \right] \ge 1-  \delta \\
     \iff &\PP_{\S \sim \D^m} \left[ 
    \begin{array}{l}
         \displaystyle \forall \a \in \A,\  \forall \Q \in \M(\Hcal),  \\
         \displaystyle \pi(\a) \varphi\left(\EE_{h \sim \Q} \Lemp_{\a}(h), \EE_{h \sim \Q} \L_{\a}(h)\right)   \\
         \displaystyle \phantom{aaaa} \le \pi(\a) \frac{1}{\lambda_\a}\left[\KL(\Q\|\P) + \ln\left( \frac{n}{\delta} \EE_{\S' \sim \D^m} \EE_{h' \sim \P }\ e^{\lambda_\a \varphi\left( \Laempprime(h'), \Laprime(h')\right)} \right) \right]
    \end{array}
     \right] \ge 1-  \delta \\
     \implies &\PP_{\S \sim \D^m} \left[ 
    \begin{array}{l}
         \displaystyle  \forall \Q \in \M(\Hcal),  \\
         \displaystyle \sum_{\a \in \A} \pi(\a) \varphi\left(\EE_{h \sim \Q} \Lemp_{\a}(h), \EE_{h \sim \Q} \L_{\a}(h)\right)   \\
         \displaystyle \phantom{aaaa} \le \sum_{\a \in \A} \pi(\a) \frac{1}{\lambda_\a}\left[\KL(\Q\|\P) + \ln\left( \frac{n}{\delta} \EE_{\S' \sim \D^m} \EE_{h' \sim \P }\ e^{\lambda_\a \varphi\left( \Laempprime(h'), \Laprime(h')\right)} \right) \right]
    \end{array}
     \right] \ge 1-  \delta \\
     \iff &\PP_{\S \sim \D^m} \left[ 
    \begin{array}{l}
         \displaystyle \forall \Q \in \M(\Hcal),  \\
         \displaystyle \EE_{\a \sim \pi }\varphi\left(\EE_{h \sim \Q} \Lemp_{\a}(h), \EE_{h \sim \Q} \L_{\a}(h)\right)   \\
         \displaystyle \phantom{aaaa} \le \EE_{\a \sim \pi} \frac{1}{\lambda_\a}\left[\KL(\Q\|\P) + \ln\left( \frac{n}{\delta} \EE_{\S' \sim \D^m} \EE_{h' \sim \P }\ e^{\lambda_\a \varphi\left( \Laempprime(h'), \Laprime(h')\right)} \right) \right] 
    \end{array}
     \right] \ge 1-  \delta,
     \label{eq:general-ma-proof-5}
    \end{align}
\endgroup
which is the desired result.
\end{proof}

Thanks to \Cref{lem:general-ma}, we are now ready to prove \Cref{eq:general-ma-classical} of \Cref{thm:general-ma}.

\begin{proof}
For any $\rho^* \in E$, we can define $\varepsilon_{\rho^*} \ge 0$ such that we have 
\begin{align*}
 \risk(h) = \sup_{\rho \in \env} \EE_{\a \sim \rho} \La(h) = \EE_{\a \sim \rho^*} \La(h) + \varepsilon_{\rho^*}.
\end{align*}
Therefore, we have for all $\rho^* \in E$
\begin{align}
\varphi\left(\riskemp(\Q) , \EE_{h \sim \Q}  \risk(h) - \varepsilon_{\rho^*}\right) & = \varphi\left(\sup_{\rho \in \env} \EE_{\a \sim \rho} \EE_{h \sim \Q} \Laemp(h) \, ,\, \EE_{h \sim \Q} \EE_{\a\sim \rho^{*}} \La(h) \right)\nonumber\\
&\le \varphi\left(\EE_{\a\sim \rho^{*}}\EE_{h \sim \Q} \Laemp(h), \EE_{\a\sim \rho^{*}}\EE_{h \sim \Q} \La(h)\right)\nonumber\\
&\le \EE_{\a\sim \rho^{*}} \varphi\left(\EE_{h \sim \Q} \Laemp(h), \EE_{h \sim \Q} \La(h)\right),
\label{eq:general-ma-proof-6}
\end{align}
where the first inequality comes from the fact that $\rho^* \in E$ and $\varphi$ is non-increasing with respect to its first argument, and we used, for the second inequality, Jensen's inequality (since $\varphi$ is jointly convex).
Moreover, as $\varphi$ is positive and since $\frac{d \rho^*}{d \pi}(\a) \le \frac{1}{\alpha}$ for all $\a\in\A$, we have
\begin{align}
    \EE_{\a\sim \rho^{*}} \varphi\left(\EE_{h \sim \Q} \Laemp(h), \EE_{h \sim \Q} \La(h)\right) &= \EE_{\a \sim \pi} \frac{d \rho^{*}}{d \pi}(\a)\   \varphi\left(\EE_{h \sim \Q} \Laemp(h), \EE_{h \sim \Q} \La(h)\right) \nonumber\\
    &\le \EE_{\a \sim \pi} \frac{1}{\alpha} \   \varphi\left(\EE_{h \sim \Q} \Laemp(h), \EE_{h \sim \Q} \La(h)\right) \nonumber\\
    &= \frac{1}{\alpha} \EE_{\a \sim \pi}  \varphi\left(\EE_{h \sim \Q} \Laemp(h), \EE_{h \sim \Q} \La(h)\right).
    \label{eq:general-ma-proof-7}
\end{align}
By combining \Cref{eq:general-ma-proof-6,eq:general-ma-proof-7} and \Cref{lem:general-ma} we get
\begin{align}
    \PP_{\S \sim \D^m} \left[ 
    \begin{array}{l}
         \displaystyle \forall \Q \in \M(\Hcal), \forall \rho^*\in E,\\
         \displaystyle \varphi\left(\riskemp(\Q) , \EE_{h \sim \Q}  \risk(h) - \varepsilon_{\rho^*}\!\right) \le \EE_{\a \sim \pi} \frac{1}{\alpha \, \lambda_\a}\left[ \KL(\Q\|\P) + \ln\left( \frac{n}{\delta} \EE_{\S' \sim \D^m} \EE_{h' \sim \P }\ e^{\lambda_\a \varphi\left( \Laempprime(h'), \Laprime(h')\right)} \right) \right]
    \end{array}
     \right] \ge 1-  \delta  \, .
\end{align}
Finally, since the bound holds for all $\rho^*\in E$, we can let $\varepsilon_{\rho^*} \rightarrow 0$ to get the desired result.
\end{proof}

\subsection{Proof of \Cref{eq:general-ma-dis}}
\label{sec:general-ma-dis-proof}

To prove \Cref{eq:general-ma-dis}, we first prove \Cref{lem:general-ma-dis}; the proof essentially follows the steps of \citet{rivasplata2020pac} before applying a union bound.

\begin{lemma}\label{lem:general-ma-dis}
    For any distribution $\D$ on $\XY$,
    for any positive, jointly convex function $\varphi(a,b)$,
    for any finite set $\A$ of $\cardA$ subgroups, 
    for any $\lambdaa >0$ for each $\a\! \in\! \A$,
    for any distribution $\pi$ over $\A$,
    for any distribution $\P  \in  \M(\Hcal)$, 
    for any loss function $\ell\!:\!\Y\!\times\!\Y\!\to\! [0,1]$,
    for any $\delta \in (0,1]$,
    for any $\alpha  \in  (0,1)$,
    for any algorithm $\Phi: (\XY)^m \!\times\! \M(\Hcal) \rightarrow \M(\Hcal)$,
    with probability at least $1-\delta$ over $\S\sim\D^m$ and $h\sim\Q_{\S}$, we have
    \begin{align*}
          \EE_{\a \sim \pi}\varphi\left(\Lemp_{\a}(h), \L_{\a}(h)\right)  \le \EE_{\a \sim \pi} \frac{1}{\lambda_\a}\left[ \lnplus\!\left(\frac{d\Q_{\S}}{d\P}(h)\right) + \ln\left( \frac{n}{\delta} \EE_{\S' \sim \D^m} \EE_{h' \sim \P }\ e^{\lambda_\a \varphi\left( \Laempprime(h'), \Laprime(h')\right)} \right) \right].
    \end{align*}
\end{lemma}
\begin{proof}
We apply Markov's inequality on $e^{\lambda_\a \varphi\left( \Laemp(h), \La(h)\right)-\ln\frac{d\Q_\S}{d\P}(h)}$.
Indeed, we have, with probability at least $1-\delta/n$ over $\S\sim\D^m$ and $h\sim\Q_\S$
\begingroup
\allowdisplaybreaks
\begin{align*}
&e^{\lambda_\a \varphi\left( \Laemp(h), \La(h)\right)-\ln\frac{d\Q_\S}{d\P}(h)} \le \frac{n}{\delta}\EE_{\S'\sim\D^m}\EE_{h\sim\Q_{\S'}}e^{\lambda_\a \varphi\left( \Laemp(h), \La(h)\right)-\ln\frac{d\Q_{\S'}}{d\P}(h)}\\
\iff\quad &\ln\left(e^{\lambda_\a \varphi\left( \Laempprime(h), \La(h)\right)-\ln\frac{d\Q_\S}{d\P}(h)}\right) \le \ln\frac{n}{\delta} + \ln\left(\EE_{\S'\sim\D^m}\EE_{h\sim\Q_{\S'}}e^{\lambda_\a \varphi\left( \Laempprime(h), \La(h)\right)-\ln\frac{d\Q_{\S'}}{d\P}(h)}\right)\\
\iff\quad &\lambda_\a \varphi\left( \Laemp(h), \La(h)\right)-\ln\frac{d\Q_\S}{d\P}(h) \le \ln\frac{n}{\delta} + \ln\left(\EE_{\S'\sim\D^m}\EE_{h\sim\Q_{\S'}}e^{\lambda_\a \varphi\left( \Laempprime(h), \La(h)\right)-\ln\frac{d\Q_{\S'}}{d\P}(h)}\right)\\
\iff &\lambda_\a \varphi\left( \Laemp(h), \La(h)\right)-\ln\frac{d\Q_\S}{d\P}(h) \le \ln\frac{n}{\delta} + \ln\left(\EE_{\S'\sim\D^m}\EE_{h\sim\P} e^{\lambda_\a \varphi\left( \Laempprime(h), \La(h)\right)}\right)\\
\iff &\varphi\left( \Laemp(h), \La(h)\right) \le \frac{1}{\lambda_\a}\left[\ln\frac{d\Q_\S}{d\P}(h) + \ln\frac{n}{\delta} + \ln\left(\EE_{\S'\sim\D^m}\EE_{h\sim\P} e^{\lambda_\a \varphi\left( \Laempprime(h), \La(h)\right)}\right)\right].
\end{align*}
\endgroup
Furthermore, since $\ln(\cdot) \le \lnplus\!(\cdot)$, with probability at least $1-\delta/n$ over $\S\sim\D^m$ and $h\sim\Q_\S$, we have
\begin{align*}
\varphi\left( \Laemp(h), \La(h)\right) \le \frac{1}{\lambda_\a}\left[\lnplus\!\left(\frac{d\Q_\S}{d\P}(h)\right)+ \ln\frac{n}{\delta} + \ln\left(\EE_{\S'\sim\D^m}\EE_{h\sim\P} e^{\lambda_\a \varphi\left( \Laempprime(h), \La(h)\right)}\right)\right].
\end{align*}

As $\A$ is finite with $|\A|=n$, we apply the union bound argument to obtain
\begingroup
\allowdisplaybreaks
\begin{align*}
     \iff &\PP_{\S \sim \D^m, h\sim\Q_\S} \left[ 
    \begin{array}{l}
         \displaystyle \forall \a \in \A,  \\
         \displaystyle \varphi\left(\Lemp_{\a}(h), \L_{\a}(h)\right)   \\
         \displaystyle \phantom{aaaa} \le \frac{1}{\lambda_\a}\left[\lnplus\!\left(\frac{d\Q_\S}{d\P}(h)\right) + \ln\left( \frac{n}{\delta} \EE_{\S' \sim \D^m} \EE_{h' \sim \P }\ e^{\lambda_\a \varphi\left( \Laempprime(h'), \Laprime(h')\right)} \right) \right]
    \end{array}
     \right] \ge 1-  \delta \\
     \iff &\PP_{\S \sim \D^m, h\sim\Q_\S} \left[ 
    \begin{array}{l}
         \displaystyle \forall \a \in \A,  \\
         \displaystyle \pi(\a) \varphi\left(\Lemp_{\a}(h), \L_{\a}(h)\right)   \\
         \displaystyle \phantom{aaaa} \le \pi(\a) \frac{1}{\lambda_\a}\left[\lnplus\!\left(\frac{d\Q_\S}{d\P}(h)\right) + \ln\left( \frac{n}{\delta} \EE_{\S' \sim \D^m} \EE_{h' \sim \P }\ e^{\lambda_\a \varphi\left( \Laempprime(h'), \Laprime(h')\right)} \right) \right]
    \end{array}
     \right] \ge 1-  \delta \\
     \implies &\PP_{\S \sim \D^m, h\sim\Q_\S} \left[ 
    \begin{array}{l}
         \displaystyle \sum_{\a \in \A} \pi(\a) \varphi\left(\Lemp_{\a}(h), \L_{\a}(h)\right)   \\
         \displaystyle \phantom{aaaa} \le \sum_{\a \in \A} \pi(\a) \frac{1}{\lambda_\a}\left[\lnplus\!\left(\frac{d\Q_\S}{d\P}(h)\right) + \ln\left( \frac{n}{\delta} \EE_{\S' \sim \D^m} \EE_{h' \sim \P }\ e^{\lambda_\a \varphi\left( \Laempprime(h'), \Laprime(h')\right)} \right) \right]
    \end{array}
     \right] \ge 1-  \delta \\
     \iff &\PP_{\S \sim \D^m, h\sim\Q_\S} \left[ 
    \begin{array}{l}
         \displaystyle \EE_{\a \sim \pi }\varphi\left(\Lemp_{\a}(h), \L_{\a}(h)\right)   \\
         \displaystyle \phantom{aaaa} \le \EE_{\a \sim \pi} \frac{1}{\lambda_\a}\left[\lnplus\!\left(\frac{d\Q_\S}{d\P}(h)\right) + \ln\left( \frac{n}{\delta} \EE_{\S' \sim \D^m} \EE_{h' \sim \P }\ e^{\lambda_\a \varphi\left( \Laempprime(h'), \Laprime(h')\right)} \right) \right] 
    \end{array}
     \right] \ge 1-  \delta,
    \end{align*}
\endgroup
which is the desired result.
\end{proof}

We are now ready to prove \Cref{eq:general-ma-dis} of \Cref{thm:general-ma}.

\begin{proof}
For any $\rho^* \in E$, we can define $\varepsilon_{\rho^*} \ge 0$ such that we have 
\begin{align*}
 \risk(h) = \sup_{\rho \in \env} \EE_{\a \sim \rho} \La(h) = \EE_{\a \sim \rho^*} \La(h) + \varepsilon_{\rho^*}.
\end{align*}
Therefore, we have for all $\rho^* \in E$
\begin{align}
\varphi\left(\riskemp(h) , \risk(h) - \varepsilon_{\rho^*}\right) & = \varphi\left(\sup_{\rho \in \env} \EE_{\a \sim \rho} \Laemp(h) \, ,\, \EE_{\a\sim \rho^{*}} \La(h) \right)\nonumber\\
&\le \varphi\left(\EE_{\a\sim \rho^{*}}\Laemp(h), \EE_{\a\sim \rho^{*}}\La(h)\right)\nonumber\\
&\le \EE_{\a\sim \rho^{*}} \varphi\left(\Laemp(h), \La(h)\right),
\label{eq:general-ma-dis-proof-1}
\end{align}
where the first inequality comes from the fact that $\rho^{*} \in E$ and $\varphi$ is non-increasing with respect to its first argument, and we used, for the second inequality, Jensen's inequality (since $\varphi$ is jointly convex). \\
Moreover, as $\varphi$ is positive and since $\frac{d \rho^{*}}{d \pi}(\a) \le \frac{1}{\alpha}$ for all $\a\in\A$, we have
\begin{align}
    \EE_{\a\sim \rho^{*}} \varphi\left(\Laemp(h), \La(h)\right) &= \EE_{\a \sim \pi} \frac{d \rho^{*}}{d \pi}(\a)\   \varphi\left(\Laemp(h), \La(h)\right) \nonumber\\
    &\le \EE_{\a \sim \pi} \frac{1}{\alpha} \   \varphi\left(\Laemp(h), \La(h)\right) \nonumber\\
    &= \frac{1}{\alpha} \EE_{\a \sim \pi}  \varphi\left(\Laemp(h), \La(h)\right).
    \label{eq:general-ma-dis-proof-2}
\end{align}
By combining \Cref{eq:general-ma-dis-proof-1,eq:general-ma-dis-proof-2} and \Cref{lem:general-ma-dis} we get
\begin{align}
    \PP_{\S \sim \D^m} \! \left[ 
    \begin{array}{l}
         \displaystyle \forall \rho^*\in E, \\
         \displaystyle \varphi\left(\riskemp(h) , \risk(h) - \varepsilon_{\rho^*}\right) \le \EE_{\a \sim \pi} \frac{1}{\alpha \, \lambda_\a}\left[ \lnplus\!\left(\frac{d\Q_\S}{d\P}(h)\right) + \ln\left( \frac{n}{\delta} \EE_{\S' \sim \D^m} \EE_{h' \sim \P }\ e^{\lambda_\a \varphi\left( \Laempprime(h'), \Laprime(h')\right)} \right) \right]
    \end{array}
   \!  \right] \ge 1\!-\! \delta  \,.
\end{align}
Finally, since the bound holds for all $\rho^*\in E$, we can have $\varepsilon_{\rho^*} \rightarrow 0$ to get the desired result.
\end{proof}

\section{ABOUT THE \texorpdfstring{$\kl^+$}{kl+}}
\label{sec:klplus-proof}

In this section, we prove two properties of $\kl^+$ that are useful in \Cref{sec:cor-general-ma-proof}.

\begin{lemma}[Useful properties on $\kl^+$] 
    For any $a, b \in [0,1]$ we have 
    \begin{align*}
        \kl(a\|b) \triangleq a \ln \frac{a}{b}{+}(1{-}a) \ln \frac{1-a}{1-b} \quad \text{ and } \quad
        \kl^+ \left(a \middle\|b\right) \triangleq  \begin{cases}
                    \kl(a\|b) &\text{if } a \le b, \\
                    0 &\text{otherwise.}
                    \end{cases}
    \end{align*}
    \begin{enumerate}
        \item $\kl^+\left(a \middle\|b\right)$ is non-increasing in $a$ for any fixed $b$. 
        \item $\kl^+(a\|b) \le \kl(a \|b)$.
    \end{enumerate}
    \label{lem:kl}
\end{lemma}
\begin{proof}[Proof of 1.]
If $a > b$, by definition, $\kl^+(a\|b) = 0$, which is constant.
Otherwise, if $a \leq b$, we compute the derivative of $\kl(a \| b)$ with respect to $a$. We have
\begin{align*}
\frac{d}{da} \kl(a\|b) &= \frac{d}{da} \left[ a \ln \frac{a}{b} + (1-a) \ln \frac{1-a}{1-b} \right]\\
    &= \ln \frac{a}{b} - \ln \frac{1-a}{1-b}. \\
    &= \ln \left( \frac{a(1-b)}{b(1-a)} \right).
\end{align*}
For $a \leq b$, we have $\frac{a(1-b)}{b(1-a)} \leq 1$, so its logarithm is non-positive, meaning $\frac{d}{da} \kl(a \| b) \leq 0.$
Thus, $\kl(a \| b)$ is non-increasing in $a$ when $a \leq b$.
\end{proof}

\begin{proof}[Proof of 2.]
If $a\le b$, $\kl^+(a\|b) = \kl(a\|b)$. Otherwise, $a > b$, $\kl^+(a\|b) = 0 \le \kl(a\|b)$ as $\kl(a\|b) \ge 0$.
\end{proof}

\begin{lemma}[Pinsker's inequality for $\kl^+$]
    \label{lem:pinsker}
For any $a, b \in [0,1]$,
    \begin{align*}
        b - a \le \sqrt{\frac{1}{2} \ \kl^+\left(a \middle\| b\right)}
    \end{align*}  
    \end{lemma}
\begin{proof}
    If $a\le b$, $\kl^+ = \kl$, we apply Pinsker's inequality. Otherwise, $a>b$, meaning $b-a < 0$, and $\sqrt{\frac{1}{2} \ \kl^+\left(a \middle\| b\right)} = 0$, so the inequality holds.
\end{proof}

\section{COROLLARIES OF THEOREM~\ref{thm:general-ma}}
\label{sec:cor-general-ma-proof}

\subsection{\Cref{cor:mca-dis}}
\label{sec:mca-dis-proof}

\corboundseegermcadis*
\begin{proof}[Proof of \Cref{eq:seeger-ma-dis}]
    As $\kl^+(a,b)$ is positive and non-increasing in $a$ (\Cref{lem:kl}) we can apply \Cref{thm:general-ma} with $\lambdaa = \ma$ for any $\a\in A$ and the function $\kl^+$.
    We have with probability at least $1 {-} \delta$ over $\S \sim \D^m$, $\forall \Q \in \M(\Hcal)$,
    \begin{align}
    \kl^+ \left(\riskemp\left(\Q\right) \middle\| \risk\left(h\right)\right) \le \EE_{\a \sim \pi} \frac{1}{\alpha \,\ma}\left[ \KL(\Q\|\P) + \ln\left(\frac{\cardA }{\delta}\EE_{\S' \sim \D^m} \EE_{h'\sim\P}e^{\ma\kl^+ \left( \Laempprime(h') \middle\| \La(h')\right)}\right) \right] \,.
    \label{eq:proof-ma-1}
    \end{align}
    
    \noindent Since $P$ does not depend on $\S'$, we have for any $\a \in \A$,
    \begin{align*}
        \ln  \Biggl(  \frac{\cardA }{\delta}  \EE_{\S' \sim  \D^m} \EE_{h'\sim P}  e^{  \ma \, \kl^+\left( \Laempprime(h') \middle\| \La(h') \right) }  \Biggr) 
        = \ln  \Biggl( \frac{\cardA }{\delta} \EE_{h'\sim P} \EE_{\S'\sim \D^m} e^{ \ma \, \kl^+\left( \Laempprime(h') \middle\| \La(h') \right) } \Biggr).
    \end{align*}
    \noindent Thanks to \citet{maurer2004note}, for any $\a \in \A$ for any $h \in \Hcal$, we have
    \begin{align*}
        \EE_{\mathcal{S'} \sim \D^m} e^{ \ma \, \kl^+\left( \Laempprime(h) \; \middle\| \; \La(h) \right) } \le \EE_{\mathcal{S'} \sim \D^m} e^{ \ma \, \kl\left( \Laempprime(h) \; \middle\| \; \La(h) \right) }
        \le 2 \sqrt{\ma},
    \end{align*}
    \noindent Where the first inequality comes from the fact that $\kl^+ \le \kl$ (see \Cref{lem:kl}).
    
    \noindent Therefore, we have
    \begin{align}
        \ln  \Biggl( \frac{\cardA }{\delta} \EE_{h'\sim P} \EE_{\S'\sim \D^m} e^{ \ma \, \kl^+\left( \Laempprime(h') \middle\| \La(h') \right) } \Biggr)
        \le \ln  \Biggl( \frac{2\cardA \sqrt{\ma}}{\delta}\Biggr).
        \label{eq:proof-ma-2}
    \end{align}
    We get the desired result by combining \Cref{eq:proof-ma-1} and \Cref{eq:proof-ma-2}
\end{proof}

\begin{proof}[Proof of \Cref{eq:mcallester-ma-dis}]
    We apply \Cref{lem:pinsker} on \Cref{eq:seeger-ma-dis} and rearrange the terms.
\end{proof}

\subsection{\Cref{cor:mca-pb}}
\label{sec:mca-pb-proof}

\begin{restatable}{corollary}{corboundseegermca}
\label{cor:mca-pb}
For any  \emph{finite} set of $\cardA$ subgroups $\A$, 
for any distribution $\pi$ over $\A$,
for any distribution $\D$ over $\XY$,
for any distribution $\P  \in  \M(\Hcal)$, 
for any loss function $\ell: \Y  \times  \Y  \to  [0,1]$,
for any $\delta \in (0,1]$,
for any $\alpha   \in  (0,1)$
with probability at least $1 {-} \delta$ \mbox{over $\S {\sim} \D^m$},
for all distributions \mbox{$\Q  \in  \M(\Hcal)$},
we have 
\begin{align}
\label{eq:seeger-ma-dis-pb} \kl^+ \Big( \riskemp(\Q) \Big\| \EE_{h \sim \Q}   \risk(h) \Big)  \le  \EE_{\a \sim \pi} \!   \frac{\KL(\Q\|\P)   +   \ln\frac{2 \cardA  \sqrt{\ma}}{\delta}}{\alpha \, \ma},\\
\label{eq:mcallester-ma-dis-pb}\text{and}\ \EE_{h \sim \Q}   \risk(h)  \le  \riskemp(\Q)  +  \sqrt{ \EE_{\a \sim \pi} \!  \frac{\KL(\Q\|\P)  +  \ln\frac{2 \cardA  \sqrt{\ma}}{\delta}}{2 \, \alpha \, \ma}}.
\end{align}
\end{restatable}
\begin{proof}[Proof of \Cref{eq:seeger-ma-dis-pb}]
    As $\kl^+(a,b)$ is positive and non-increasing in $a$ (\Cref{lem:kl}) we can apply of \Cref{thm:general-ma} with $\lambdaa = \ma$ for any $\a\in A$ and the function $\kl^+$.
    We have with probability at least $1 {-} \delta$ over $\S \sim \D^m$, $\forall \Q \in \M(\Hcal)$,
    \begin{align}
    \kl^+ \left(\riskemp\left(\Q\right) \middle\| \risk\left(h\right)\right) \le \EE_{\a \sim \pi} \frac{1}{\alpha \,\ma}\left[ \KL(\Q\|\P) + \ln\left(\frac{\cardA }{\delta}\EE_{\S' \sim \D^m} \EE_{h'\sim\P}e^{\ma\kl^+ \left( \Laempprime(h') \middle\| \La(h')\right)}\right) \right] \,.
    \label{eq:seeger-ma-dis-pb-proof-1}
    \end{align}
    
    \noindent Since $P$ does not depend on $\S'$ we have for any $\a \in \A$,
    \begin{align*}
        \ln  \Biggl(  \frac{\cardA }{\delta}  \EE_{\S' \sim  \D^m} \EE_{h'\sim P}  e^{  \ma \, \kl^+\left( \Laempprime(h') \middle\| \La(h') \right) }  \Biggr) 
        = \ln  \Biggl( \frac{\cardA }{\delta} \EE_{h'\sim P} \EE_{\S'\sim \D^m} e^{ \ma \, \kl^+\left( \Laempprime(h') \middle\| \La(h') \right) } \Biggr).
    \end{align*}
    \noindent Thanks to \citet{maurer2004note}, for any $\a \in \A$ for any $h \in \Hcal$, we have
    \begin{align*}
        \EE_{\mathcal{S'} \sim \D^m} e^{ \ma \, \kl^+\left( \Laempprime(h) \; \middle\| \; \La(h) \right) } \le \EE_{\mathcal{S'} \sim \D^m} e^{ \ma \, \kl\left( \Laempprime(h) \; \middle\| \; \La(h) \right) } \le 2 \sqrt{\ma},
    \end{align*}
    where the first inequality comes from the fact that $\kl^+ \le \kl$ (see \Cref{lem:kl}).
    Therefore, we have
    \begin{align}
        \ln  \Biggl( \frac{\cardA }{\delta} \EE_{h'\sim P} \EE_{\S'\sim \D^m} e^{ \ma \, \kl^+\left( \Laempprime(h') \middle\| \La(h') \right) } \Biggr)
        \le \ln  \Biggl( \frac{2\cardA \sqrt{\ma}}{\delta}\Biggr).
        \label{eq:seeger-ma-dis-pb-proof-2}
    \end{align}
    We get the desired result by combining \Cref{eq:seeger-ma-dis-pb-proof-1} and \Cref{eq:seeger-ma-dis-pb-proof-2}
\end{proof}

\begin{proof}[Proof of \Cref{eq:mcallester-ma-dis-pb}]
    We apply \Cref{lem:pinsker} on \Cref{eq:seeger-ma-dis-pb} and rearrange the terms.
\end{proof}

\section{THEOREM~\ref{thm:bound-for-one-example}}
\label{sec:bound-for-one-example-proof}

\thmboundoneexample*

In the following, we first start by proving \Cref{eq:mcallester-one-dis} and then we prove \Cref{eq:mcallester-one}.

\subsection{Proof of \Cref{eq:mcallester-one-dis}}

To prove \Cref{thm:bound-for-one-example}, we first prove the following lemma.

\begin{lemma}
For any distribution $\D$ on $\XY$, 
for any loss function $\ell\!:\!\Y\!\times\!\Y\!\to\! [0,1]$,
for any constrained $f$-entropic risk measure $\riskemp$ satisfying \Cref{def:constrained-f-entropic} \emph{and} \Cref{eq:ass-small},
for any hypothesis $h\in\Hcal$,
for any $\delta \in (0,1]$, for any $\alpha\in (0,1]$,
we have
    \begin{align*}
        \PP_{\S \sim \D^m} \left[ \left| \riskemp(h)  - \EE_{\S' \sim \D^m} \riskempprime(h) \right | \ge \frac{1}{\alpha} \sqrt{\frac{\ln (2/\delta)}{2m}} \right] \le \delta.
    \end{align*}
    \label{lem:brown}
\end{lemma}
\begin{proof}
To prove the result, we aim to apply McDiarmid's inequality.
To do so, we need to find an upper-bound of $\sup_{(x_j', y_j') \in \X\times\Y}\sup_{\S \in (\X\times\Y)^m} |\riskemp(h) - \riskempprimej(h)|$, where $\S$ and $\S'_j$ differ from the $j$-th example.
For any $h\in\Hcal$, any $(x_j', y_j') \in \X\times\Y$ and $\S \in (\X\times\Y)^m$, we have
\allowdisplaybreaks[4]
\begin{align*}
\riskemp(h) - \riskempprimej(h) &= \sup_{\rho \in \Eemp}\left\{ \sum_{i=1}^m \rho(i)\cdot\ell(y_i, h(\xbf_i))\right\} - \sup_{\rho \in \Eemp}\left\{ \sum_{i=1}^m \rho(i)\cdot\ell(y_i', h(\xbf_i'))\right\}\\
&\le \sup_{\rho \in \Eemp}\left\{ \sum_{i=1}^m \rho(i)\cdot\ell(y_i, h(\xbf_i)) - \sum_{i=1}^m \rho(i)\cdot\ell(y_i', h(\xbf_i')) \right\}\\
&= \sup_{\rho \in \Eemp}\left\{ \sum_{i=1}^m \rho(i)\cdot\left(\ell(y_i, h(\xbf_i)) - \ell(y_i', h(\xbf_i')) \right) \right\}\\
&\le \sup_{\rho \in \Eemp}\left\{ \sum_{i=1}^m \rho(i)\cdot\left|\ell(y_i, h(\xbf_i)) - \ell(y_i', h(\xbf_i')) \right| \right\}\\
&= \sup_{\rho \in \Eemp}\Bigg\{ \rho(j)\cdot\left|\ell(y_j, h(\xbf_j)) - \ell(y_j', h(\xbf_j')) \right| \Bigg\}\\
&\le \sup_{\rho \in \Eemp}\left\{ \rho(j) \right\}\\
&\le \sup_{\rho \in \Eemp}\left\{ \frac{1}{\alpha}\pi(j)\right\} \ =\  \frac{1}{m\alpha}\,.
\end{align*}
Moreover, by doing the same steps, for $\riskempprimej(h) - \riskemp(h)$, we obtain: $\displaystyle \riskempprimej(h) - \riskemp(h) \le \frac{1}{m\alpha}$.

Finally, we get the desired result by applying McDiarmid's inequality.
\end{proof}

Now we recall Occam's hammer\footnote{\Cref{lem:occams-hammer} is a simpler version than Occam's hammer presented in \citet{blanchard2007occam}.} (Theorem 2.4 of \citet{blanchard2007occam}) that we use along with \Cref{lem:brown} to prove \Cref{eq:mcallester-one-dis}.

\begin{lemma}[Occam's hammer]
We assume that\\[-8mm]
\begin{enumerate}
    \item we have 
    \begin{align*}
    \forall h \in \Hcal,\ \forall \delta \in [0,1], \quad \PP_{\S \sim \D^m} \left[\S \in \mathcal{B}(h,\delta) \right] \le \delta,
    \end{align*}
    where $\mathcal{B}(h,\delta)$ is a set of bad events at level $\delta$ for $h$;
    \item the function $(\S, h, \delta) \in \Z^m \times \Hcal \times [0, 1] \rightarrow \mathbf{1}_{\{\S \in \mathcal{B}(h,\delta)\}}$ is jointly measurable in its three variables;
    \item for any $h \in H$, we have $\mathcal{B}(h,0) = \emptyset$;
    \item for any $h\in\Hcal$, $\mathcal{B}(h,\delta)$ is a nondecreasing sequence of sets: for $\delta \le \delta'$, we have $\mathcal{B}(h,\delta) \subseteq \mathcal{B}(h,\delta')$.
\end{enumerate} 
Then, we have
\begin{align*}
\PP_{\S \sim \D, \ h \sim \Q_\S}  \left[ \S \in \mathcal{B}\left(h, \Delta\left(h, \left[\frac{d\Q_\S}{d\P}(h)\right]^{-1}\right)\right) \right] \le \delta,
\end{align*}
where $\Delta(h, u) := \min(\delta \beta(u), 1)$, with $\Gamma$ be a probability distribution on $(0, +\infty)$ and $\beta(x) = \int_0^x u d\Gamma (u)$ for $x \in (0, +\infty)$.
  \label{lem:occams-hammer}
\end{lemma}

We are now ready to prove \Cref{eq:mcallester-one-dis} based on \Cref{lem:brown} and \Cref{lem:occams-hammer}.

\begin{proof}
Thanks to \Cref{lem:brown} we define for any $h\in\Hcal$, any $\delta \in [0,1]$, 
\begin{align} 
\label{eq:bad-event} \mathcal{B}(h,\delta) = \left\{ \S \in \Z^m \;\middle| \; \left| \riskemp(h)  - \EE_{\S' \sim \D^m} \riskempprime(h) \right | > \frac{1}{\alpha} \sqrt{\frac{\ln (2/\delta)}{2m}}\right\}.
\end{align}

Now we apply \Cref{lem:occams-hammer} to our set of \Cref{eq:bad-event}. As in the proof of Proposition 3.1 in \citet{blanchard2007occam}, we set $\Gamma$ as the probability distribution on $[0,1]$ having density $\Gamma(u)=\frac{1}{k} u^{-1 + \frac{1}{k}}$ for any $k>0$. Then we can compute $\beta(x)$.
For the sake of completeness, we compute $\beta$.
We consider two cases.

\textbullet~For $x \le 1$, we have
\begingroup
\allowdisplaybreaks
\begin{align*}
    \beta(x) &= \int_0^x u d \Gamma(u)\\
    &= \int_0^x u \frac{1}{k} u^{-1 + \frac{1}{k}}du\\
    &= \frac{1}{k} \int_0^x  u^{ \frac{1}{k}}du \\
    & = \frac{1}{k} \left[ \frac{1}{\frac{1}{k} + 1} u^{ \frac{1}{k} + 1}  \right]_0^x\\
    &= \frac{1}{k} \left[ \frac{k}{k + 1} u^{ \frac{1}{k} + 1}  \right]_0^x \\
    &= \frac{1}{k} \frac{k}{k + 1} x^{ \frac{1}{k} + 1} = \frac{1}{k + 1} x^{ \frac{1}{k} + 1}.
\end{align*}
\endgroup

\textbullet~For $x>1$, we have
\begingroup
\allowdisplaybreaks
\begin{align*}
    \beta(x) &= \int_0^x u d \Gamma(u)\\
    &= \int_0^1 u \frac{1}{k} u^{-1 + \frac{1}{k}} du +  \int_1^x 0 du  \\
    &= \frac{1}{k} \int_0^1 u^{ \frac{1}{k}} du + 0 \\
    & = \frac{1}{k} \left[ \frac{1}{\frac{1}{k} + 1} u^{ \frac{1}{k} + 1}  \right]_0^1 \\
    &= \frac{1}{k} \left[ \frac{k}{k + 1} u^{ \frac{1}{k} + 1}  \right]_0^1\\
    &= \frac{1}{k} \frac{k}{k + 1} 1^{ \frac{1}{k} + 1}\\
    &= \frac{1}{k + 1}.
\end{align*}
\endgroup
Therefore, we can deduce that $\beta(x) = \frac{1}{k+1} \min(x^{1 + \frac{1}{k}}, 1)$.
Then, by applying \Cref{lem:occams-hammer}, we have with probability at least $1-\delta$ over $\S \sim \D^m, \ h \sim \Q_\S$
\begingroup
\allowdisplaybreaks
\begin{align*}
    &\left| \riskemp(h)  - \EE_{\S' \sim \D^m} \riskempprime(h) \right |  \le \frac{1}{\alpha} \sqrt{\frac{1}{2m} \left[ \ln \left(\frac{2}{\Delta\left(h, \left[\frac{d\Q_\S}{d\P}(h)\right]^{-1}\right)}\right) \right]}\\
    \iff &\left| \riskemp(h)  - \EE_{\S' \sim \D^m} \riskempprime(h) \right |  \le \frac{1}{\alpha} \sqrt{\frac{1}{2m} \left[ \ln \left(\frac{2}{\min\left(\delta \beta \left( \left[\frac{d\Q_\S}{d\P}(h)\right]^{-1}\right), 1 \right)}\right) \right]}\\
    \iff &\left| \riskemp(h)  - \EE_{\S' \sim \D^m} \riskempprime(h) \right |  \le \frac{1}{\alpha} \sqrt{\frac{1}{2m} \left[ \ln \left(2 \max \left( \frac{1}{\delta \beta \left( \left[\frac{d\Q_\S}{d\P}(h)\right]^{-1}\right)}, 1\right) \right) \right]}\\
    \iff &\left| \riskemp(h)  - \EE_{\S' \sim \D^m} \riskempprime(h) \right |  \le \frac{1}{\alpha} \sqrt{\frac{1}{2m} \left[ \ln \left(2\right) + \ln \left( \max \left( \frac{1}{\delta \beta \left( \left[\frac{d\Q_\S}{d\P}(h)\right]^{-1}\right)}, 1\right) \right) \right]}\\
    \implies &\left| \riskemp(h)  - \EE_{\S' \sim \D^m} \riskempprime(h) \right |  \le \frac{1}{\alpha} \sqrt{\frac{1}{2m} \left[ \ln \left(2\right) + \lnplus\! \left( \frac{1}{\delta \beta \left( \left[\frac{d\Q_\S}{d\P}(h)\right]^{-1}\right)}\right) \right]}\\
    \iff &\left| \riskemp(h)  - \EE_{\S' \sim \D^m} \riskempprime(h) \right |  \le \frac{1}{\alpha} \sqrt{\frac{1}{2m} \left[ \ln \left(2\right) + \lnplus\! \left( \frac{1}{\delta \frac{1}{k+1} \min \left( \left(\left[\frac{d\Q_\S}{d\P}(h)\right]^{-1}\right)^{1 + \frac{1}{k}}, 1\right)}\right) \right]}\\
    \implies &\left| \riskemp(h)  - \EE_{\S' \sim \D^m} \riskempprime(h) \right |  \le \frac{1}{\alpha} \sqrt{\frac{1}{2m} \left[ \ln\left(2\right) {+} \ln\left(\frac{k{+}1}{\delta}\right) +\lnplus\left( \frac{1}{ \min \left( \left(\left[\frac{d\Q_\S}{d\P}(h)\right]^{-1}\right)^{1 {+} \frac{1}{k}}, 1\right)} \right) \right]}.
\end{align*}
\endgroup
The last implication is due to the fact that $\frac{k+1}{\delta} \ge 1$. 

Let $x,y \in \R_+$ such that $x \ge 1$, we have
\begin{align*}
\lnplus\!(xy) &= \max(\ln(xy), 0)\\
&= \max(\ln(x) + \ln(y), 0)\\
& \le \max(\ln(x), 0) + \max(\ln(y), 0)\\
&= \ln(x) + \max(\ln(y), 0)\\
&= \ln(x) + \lnplus\!(y)\,,
\end{align*}
where the inequality is due to the sub-additivity of $\max$.

Moreover, we have with probability at least $1-\delta$ over $\S \sim \D^m, \ h \sim \Q_\S$
\begingroup
\allowdisplaybreaks
\begin{align*}
    &\left| \riskemp(h)  - \EE_{\S' \sim \D^m} \riskempprime(h) \right |  \le \frac{1}{\alpha} \sqrt{\frac{1}{2m} \left[ \ln\left(\frac{2 \left(k{+}1\right)}{\delta} \right) +\lnplus\! \left(\! \frac{1}{ \!\min \!\left( \left(\left[\frac{d\Q_\S}{d\P}(h)\right]^{-1}\right)^{1 + \frac{1}{k}} \!, 1\right)}\right) \! \! \right]}\\
    \iff &\left| \riskemp(h)  - \EE_{\S' \sim \D^m} \riskempprime(h) \right |  \le \frac{1}{\alpha} \sqrt{\frac{1}{2m} \left[ \ln\left(\frac{2 \left(k{+}1\right)}{\delta} \right) +\lnplus\! \! \left(\! \max \! \left( \frac{1}{ \! \left(\left[\frac{d\Q_\S}{d\P}(h)\right]^{-1}\right)^{1 + \frac{1}{k}}} , 1 \!\right) \! \right) \! \right]}\\
    \iff &\left| \riskemp(h)  - \EE_{\S' \sim \D^m} \riskempprime(h) \right |  \le \frac{1}{\alpha} \sqrt{\frac{1}{2m} \left[ \ln\left(\frac{2 \left(k{+}1\right)}{\delta} \right) +\lnplus\! \! \left(  \frac{1}{ \! \left(\left[\frac{d\Q_\S}{d\P}(h)\right]^{-1}\right)^{1 + \frac{1}{k}}} \right) \! \right]}\\
    \iff &\left| \riskemp(h)  - \EE_{\S' \sim \D^m} \riskempprime(h) \right |  \le \frac{1}{\alpha} \sqrt{\frac{1}{2m} \left[ \ln\left(\frac{2 \left(k{+}1\right)}{\delta} \right) +\left( 1 + \frac{1}{k}\right) \lnplus\! \left( \frac{d\Q_\S}{ d\P}(h) \right) \right]},
\end{align*}
\endgroup
which is the desired result.
\end{proof}

\subsection{Proof of \Cref{eq:mcallester-one}}

This proof comes from \citet[Corollary 3.2]{blanchard2007occam}.

\begin{proof} 
From \Cref{eq:mcallester-one-dis}, we can deduce that 
\begin{align*}
&\PP_{\S\sim\D^m, h\sim\Q_\S} \left[ \left|\, \riskemp(h)  - \EE_{\S' \sim \D^m} \riskempprime(h)\, \right| > \frac{1}{\alpha}\sqrt{\frac{1}{2m}\! \left( \left[1 {+} \frac{1}{\lambda} \right] \lnplus\!\!\left[\frac{d\Q_\S}{d\P}(h)\right]\! +\! \ln \!\left[\!\frac{2(\lambda{+}1)}{\gamma\delta}  \!\right] \right)} \right] \le \delta \gamma.
\end{align*}
Moreover, from Markov's inequality, we can deduce that we have
\begin{align}
    &\PP_{\S \sim \D^m} \left[ \PP_{h \sim \Q_\S}  \! \Bigg[ \left|\, \riskemp(h)  - \EE_{\S' \sim \D^m} \riskempprime(h)\, \right| > \frac{1}{\alpha}\sqrt{\frac{1}{2m}\! \left( \left[1 {+} \frac{1}{\lambda} \right] \lnplus\!\!\left[\frac{d\Q_\S}{d\P}(h)\right]\! +\! \ln \!\left[\!\frac{2(\lambda{+}1)}{\gamma\delta}  \!\right] \right)} \Bigg] > \gamma \right]\nonumber\\
    \le &\PP_{\S \sim \D^m} \left[ \PP_{h \sim \Q_\S}  \! \Bigg[ \left|\, \riskemp(h)  - \EE_{\S' \sim \D^m} \riskempprime(h)\, \right| > \frac{1}{\alpha}\sqrt{\frac{1}{2m}\! \left( \left[1 {+} \frac{1}{\lambda} \right] \lnplus\!\!\left[\frac{d\Q_\S}{d\P}(h)\right]\! +\! \ln \!\left[\!\frac{2(\lambda{+}1)}{\gamma\delta}  \!\right] \right)} \Bigg] \ge \gamma \right]\nonumber\\
    \le &\frac{1}{\gamma}\EE_{\S \sim \D^m} \PP_{ h \sim \Q_\S}  \! \Bigg[ \!\left|\, \riskemp(h)  - \EE_{\S' \sim \D^m} \riskempprime(h)\, \right| > \frac{1}{\alpha}\sqrt{\frac{1}{2m}\! \left( \left[1 {+} \frac{1}{\lambda} \right] \lnplus\!\!\left[\frac{d\Q_\S}{d\P}(h)\right]\! +\! \ln \!\left[\!\frac{2(\lambda{+}1)}{\gamma\delta}  \!\right] \right)} \Bigg] \le \delta.\label{eq:mcallester-one-proof-1}
\end{align}
For any $i\in \mathbb{N}$, we consider $\delta_i=\delta 2^{-i}$ and $\gamma_i=2^{-i}$ in \Cref{eq:mcallester-one-proof-1} (instead of $\delta$ and $\gamma$).
Concerning $i=0$, we have a special case: We know that $\delta=0$ since we have $\gamma_0=2^{0}=1$.
Hence, we perform a union bound on $\delta_i$ where $i\in\mathbb{N}$; we have $\sum_{i\in\mathbb{N}} \delta_i = \delta_0 + \sum_{i\in\mathbb{N}, i>0} \delta_i = \sum_{i\in\mathbb{N}, i>0} \delta_i = \delta$ and
\begin{align*}
    &\!\!\PP_{\S \sim \D^m} \left[ \begin{array}{cc}
         &  \exists i \ge 0\,, \hfill \\
         &  \PP_{h \sim \Q_\S}  \! \Bigg[ \left|\, \riskemp(h)  - \EE_{\S' \sim \D^m} \riskempprime(h)\, \right| > \frac{1}{\alpha}\sqrt{\frac{1}{2m}\! \left( \left[1 {+} \frac{1}{\lambda} \right] \lnplus\!\!\left[\frac{d\Q_\S}{d\P}(h)\right]\! +\! \ln \!\left[\!\frac{2(\lambda{+}1)}{\delta 2^{-2i}}  \!\right] \right)} \Bigg] > 2^{-i}
    \end{array} \right] \le \delta.
\end{align*}
Moreover, let 
\begin{align}
    \phi(h, \S) = 2m\alpha^2 \left| \riskemp(h)  {-} \!\EE_{\S' \sim \D^m} \riskempprime(h) \right|^2 - \left(1 {+} \frac{1}{\lambda} \right) \lnplus\! \! \left(\frac{d\Q_\S}{d\P}(h) \!\right) - \ln \left(\frac{2(\lambda{+}1)}{\delta}  \right),
\end{align}
and we have
\begin{align}
    &\PP_{\S \sim \D^m} \left[ \exists i \ge 0\,,\ \PP_{h \sim \Q_\S}  \! \Bigg[ \phi(h, \S) > 2i\ln(2) \Bigg] > 2^{-i}\  \right] \le \delta\nonumber\\
    \iff\quad &\PP_{\S \sim \D^m} \left[ \forall i \ge 0\,,\ \PP_{h \sim \Q_\S}  \! \Bigg[ \phi(h, \S) > 2i\ln(2) \Bigg] \le 2^{-i}\  \right] \ge 1-\delta.\label{eq:mcallester-one-proof-2}
\end{align}
Moreover, note that we have
\begin{align*}
\EE_{h \sim \Q_\S}\left[ \phi(h, \S)\right] &\le \int_{t \ge 0} \PP_{h \sim \Q_{\S}} \left[\phi(h, \S) > t\right]dt\\
&= \sum_{i\in \mathbb{N}}\int_{2i\ln(2)}^{2(i+1)\ln(2)}\Bigg[\PP_{h \sim \Q_{\S}} \left[\phi(h, \S) > t\right]\Bigg]dt\\
&\le \sum_{i\in \mathbb{N}}\int_{2i\ln(2)}^{2(i+1)\ln(2)}\Bigg[\PP_{h \sim \Q_{\S}} \left[\phi(h, \S) > 2i\ln(2)\right]\Bigg]dt\\
&\le \sum_{i\in \mathbb{N}}\int_{2i\ln(2)}^{2(i+1)\ln(2)} 2^{-i} dt\\
&= 2\ln(2)\sum_{i\in \mathbb{N}} 2^{-i} dt\\
&= 4\ln(2) \ \le\  3.
\end{align*}
Put into words, having $\forall i \ge 0\,,\ \PP_{h \sim \Q_\S}[ \phi(h, \S) > 2i\ln(2)] \le 2^{-i}$ implies that $\EE_{h \sim \Q_\S}\left[ \phi(h, \S)\right] \le 3$.

Hence, thanks to this implication and \Cref{eq:mcallester-one-proof-2}, we can deduce that we have 
\begin{align}
        &\PP_{\S \sim \D^m} \left[ \EE_{h \sim \Q_\S} 2m\alpha^2\left| \riskemp(h)  {-} \!\EE_{\S' \sim \D^m} \riskempprime(h) \right |^2 - \left(1 {+} \frac{1}{\lambda} \right) \EE_{h \sim \Q_\S} \lnplus\! \! \left(\frac{d\Q_\S}{d\P}(h) \!\right) - \ln \left(\frac{2(\lambda{+}1)}{\delta}  \right)  \le 3
        \right] \ge 1 {-} \delta \nonumber\\
        \iff &\PP_{\S \sim \D^m} \left[ \sqrt{ \EE_{h \sim \Q_\S}  \left| \riskemp(h)  {-} \!\EE_{\S' \sim \D^m} \riskempprime(h) \right | ^2} \le \frac{1}{\alpha} \sqrt{\frac{\left(1 {+} \frac{1}{\lambda} \right) \EE_{h \sim \Q_\S} \lnplus\! \! \left(\frac{d\Q_\S}{d\P}(h) \!\right) + \ln \left(\frac{2(\lambda{+}1)}{\delta}  \right) + 3}{2m} }\right]\ge 1 {-} \delta \nonumber\\
        \implies &\PP_{\S \sim \D^m} \left[  \left| \EE_{h \sim \Q_\S}\riskemp(h)  {-} \!\EE_{h \sim \Q_\S}\EE_{\S' \sim \D^m} \riskempprime(h) \right | \le \frac{1}{\alpha} \sqrt{\frac{\left(1 {+} \frac{1}{\lambda} \right) \EE_{h \sim \Q_\S} \lnplus\! \! \left(\frac{d\Q_\S}{d\P}(h) \!\right) + \ln \left(\frac{2(\lambda{+}1)}{\delta}  \right) + 3}{2m} }\right]\ge 1 {-} \delta,\label{eq:mcallester-one-proof-3}
\end{align}
where the last implication comes from Jensen's inequality as $\sqrt{\cdot}$ is concave and  $|\cdot|$ is convex.
Finally, we have,
\begin{align}
    \EE_{h \sim \Q_\S} \lnplus\! \! \left(\frac{d\Q_\S}{d\P}(h) \!\right) &= \EE_{h \sim \P} \frac{d\Q_\S}{d\P}(h) \lnplus\! \! \left(\frac{d\Q_\S}{d\P}(h) \!\right) \nonumber\\
    &\le \EE_{h \sim \P} \frac{d\Q_\S}{d\P}(h) \lnplus\! \! \left(\frac{d\Q_\S}{d\P}(h) \!\right) - \min_{0\le x < 1} x \log x \nonumber\\
    &= \KL(\Q_\S \| \P) + e^{-1} 
    \label{eq:mcallester-one-proof-4}
\end{align}
Combining \Cref{eq:mcallester-one-proof-3} and \Cref{eq:mcallester-one-proof-4} and bounding $e^{-1}$ by $\frac{1}{2}$ gives the desired result.
\end{proof}

\section{DETAILS ABOUT THE EXPERIMENTS}
\label{sec:expe-details}

\subsection{Bounds in practice}
\label{bounds:expe}

\subsubsection{Batch sampling} 
We follow a mini-batch sampling strategy where batches are constructed \wrt the reference distribution $\pi$ on the classes in $\A$.
In this setting, examples belonging to subgroups that are less represented in the data might be present in different batches. 
However, for each batch, we ensure that the data is not redundant and that all subgroups are represented by at least one example.

\subsubsection{Prior learning algorithm} 
\Cref{alg:self-bounding} requires a prior distribution $\P$. 
In practice, we propose to learn this prior by running \Cref{alg:prior} below (as described in \Cref{sec:self-bounding}).  

\begin{algorithm}[h]
\caption{Learning a prior distribution for constrained 
$f$-entropic risk measures}
\label{alg:prior}
\begin{algorithmic}[1]
  \Require Prior learning set $\Sp$,
  posterior learning set $\S$,
  number of epochs $T$, 
  variance $\sigma^2$, 
  reference $\pi$,  
  set of hyperparameter configurations $\configs$ of size $K$,
  parameters $\alpha$, $\beta$
  \State Initialize the set of prior distributions: $\priorset \gets \emptyset$
  \ForAll{$\config \ \in \configs$}
      \State Initialize $\thetaP$
      \For{$t = 1$ \textbf{to} $T$}
        \ForAll{mini-batches $\batch \subset \Sp$ drawn \wrt $\pi$}
            \State Draw a model $\hthetatildeP$ from $\P_{\theta} \!=\! \Ncal(\thetaP, \sigma^2 I_d)$ \Comment{where $d$ is the size of $\thetaP$}
          \State Compute the risk $\riskempB(\hthetatildeP)$ on the mini-batch
          \State \mbox{Update $\thetaP$ with gradient $\nabla_{\!\thetaP}\riskempB(\hthetatildeP)$}
        \EndFor
        \State Add $\P_{\theta}$ to set of prior distributions:  $\priorset \gets \priorset \cup \{\P_{\theta}\}$
      \EndFor
  \EndFor
  \State \textbf{return} $\P = \argmin_{\P_{\theta} \in \priorset} \ \left\{\riskemp(\hthetatildeP), \  \mbox{with}\ \hthetatildeP \sim \P_{\theta}\right\}$
\end{algorithmic}
\end{algorithm}

Across the $T$ epochs and the hyperparameter configurations considered, we get $T\times K$ prior distributions on $\Sp$ stored in the set $\priorset$. 
In the end, the prior $P$ selected to learn the posterior distribution with \Cref{alg:self-bounding}, is the prior that minimizes the risk on the learning set $\S$. 

\subsubsection{Objective functions for learning the posterior with \Cref{alg:self-bounding}} 
Note that the bounds of \Cref{cor:mca-dis}, \Cref{thm:bound-for-one-example,thm:mhammedi} do not hold ``directly'' for the above choice of $\P$ as it depends on the posterior set $\S$. 
To tackle this issue in practice, we adapt and instantiate below the bounds to our practical setting.
We respectively obtain \Cref{cor:mca-pb-expe,cor:bound-for-one-example-dis-expe,cor:mhammedi-expe}, which hold for any prior $\P_t \in \priorset$ after drawing $\S \sim \D^m$, then, they hold for the prior that minimizes the empirical risk on $\S$. 
In consequence, the bounds hold for a prior learned by \Cref{alg:prior}, and we can deduce the objective functions to minimize.

\paragraph{Instantiation of \Cref{cor:mca-dis}, and the objective function.}

The objective function associated to the minimization of \Cref{cor:mca-dis} is

\fbox{$
\displaystyle \empBoundCorUn(\riskemp(\hthetatilde), \Q_{\theta}, \thetatilde)  =\! \sup_{\substack{\rho \in \R^n_+ \\ \frac{\rho_\a}{\pi_\a} \le \frac{1}{\alpha}}} \sum_{\a \in \A} \rho_\a \sum_{i=1}^{\ma} \frac{1}{m_\a} \ell(y_i, \hthetatilde(\xbf_i)) +  \sqrt{ \EE_{\a \sim \pi} \!  \frac{1}{2 \, \alpha \, \ma} \left[ \frac{\lVert \thetatilde \!-\! \thetaP \rVert^2_2 \!-\! \lVert\thetatilde \!-\! \theta \rVert^2_2}{2 \sigma ^2}  \!+\!  \ln\frac{2 \cardA T K \sqrt{\ma}}{\delta}\right]}
$}

with $\Q_\theta = \Ncal(\theta, \sigma^2 I_d)$, and $\hthetatilde \sim \Q_\theta$ with parameters $\thetatilde$, and $\P=\Ncal(\thetaP, \sigma^2 I_d)$, and $\sigma \in [0,1]$, and $\lambda > 0$, and $\alpha \in (0,1]$, and $\delta \in [0,1]$.

The definition of $\empBoundCorUn$ comes from the following corollary of \Cref{cor:mca-dis}.
\begin{restatable}{corollary}{corboundmcadis-expe}
\label{cor:mca-pb-expe}
For any \emph{finite} set of $\cardA$ subgroups $\A$, 
for any distribution $\pi$ over $\A$,
for any distribution $\D$ over $\XY$,
for any number of epochs $T$,
for any number of hyperparameter configurations $K$,
for any set of distributions $\mathcal{P} \in  \{\P_1, ..., \P_{T\times K}\}$, 
for any loss function $\ell: \Y  \times  \Y  \to  [0,1]$,
for any $\delta \in (0,1]$,
for any $\alpha   \in  (0,1)$,
for any algorithm $\Phi: (\XY)^m \times \M(\Hcal) \rightarrow \M(\Hcal)$,
for any $\sigma \in [0,1]$, 
with probability at least $1 {-} \delta$ \mbox{over $\S {\sim} \D^m$} and $h\sim\Q_\S$,
we have $\forall \P_t=\Ncal(\thetaP, \sigma^2 I_d) \in \mathcal{P}$,
\begin{align}
\risk(h)  \le  \riskemp(h)  +  \sqrt{ \EE_{\a \sim \pi} \!  \frac{1}{2 \, \alpha \, \ma} \left[ \frac{\lVert \thetatilde - \thetaP \rVert^2_2 - \lVert\thetatilde - \theta \rVert_2^2}{2 \sigma ^2}  +  \ln\frac{2 \cardA T K \sqrt{\ma}}{\delta}\right]},
\end{align}
with $\Q_\S=\Ncal(\theta, \sigma^2 I_d)$ the posterior distribution.

\end{restatable}
\proof{
As $\frac{\delta}{T K} \in [0,1]$, we have from \Cref{cor:mca-pb}, for any $\P_t \in \mathcal{P}$,
\begin{align*}
&\PP_{\S \sim \D^m, h \sim \Q_\S} \!\!\left[ \risk(h)  \ge  \riskemp(h)  +  \sqrt{ \EE_{\a \sim \pi} \!  \frac{1}{2 \, \alpha \, \ma} \left[ \frac{\lVert \thetatilde {-} \thetaP \rVert^2_2 - \lVert\thetatilde {-} \theta \rVert_2^2}{2 \sigma ^2}  +  \ln\frac{2 \cardA T K \sqrt{\ma}}{\delta}\right]} \right] \le \frac{\delta}{TK}, \\
\implies &\sum_{t=1}^{TK} \PP_{\S \sim \D^m, h \sim \Q_\S}\!\! \left[ \risk(h)  \ge  \riskemp(h)  + \!  \sqrt{ \EE_{\a \sim \pi} \!  \frac{1}{2 \, \alpha \, \ma} \!\left[ \frac{\lVert \thetatilde {-} \thetaP \rVert^2_2 - \lVert\thetatilde {-} \theta \rVert_2^2}{2 \sigma ^2}  +  \ln \! \frac{2 \cardA T K \sqrt{\ma}}{\delta}\right]} \right] \le \sum_{t=1}^{TK} \!\frac{\delta}{TK}, \\
\implies &\PP_{\S \sim \D^m, h \sim \Q_\S}\!\! \left[ \forall \thetaP, \quad \risk(h)  \ge  \riskemp(h)  + \!  \sqrt{ \EE_{\a \sim \pi} \!  \frac{1}{2 \, \alpha \, \ma} \!\left[ \frac{\lVert \thetatilde {-} \thetaP \rVert^2_2 - \lVert\thetatilde {-} \theta \rVert_2^2}{2 \sigma ^2}  +  \ln \! \frac{2 \cardA T K \sqrt{\ma}}{\delta}\right]} \right] \le \delta, 
\end{align*}
where the last inequality follows from the union bound.
}

\paragraph{Instantiation of \Cref{thm:bound-for-one-example}, and the objective function.}

The objective function associated to the minimization of \Cref{thm:bound-for-one-example} is

\fbox{$\displaystyle \empBoundCorDeux(\riskemp(\hthetatilde), \Q_{\theta}, \thetatilde) = \!\!\!\sup_{\substack{\rho \in \R^n_+\\ m \rho_\a \le \frac{1}{\alpha}}} \sum_{\a=1}^{m} \rho_\a  \ell(y_\a, \hthetatilde(\xbf_\a))  + \sqrt{\frac{1}{2\,\alpha^2\,m} \left[\left(1 \!+\! \frac{1}{\lambda} \right) \frac{\lVert \thetatilde {-} \thetaP \rVert^2_2 - \lVert\thetatilde {-} \theta \rVert^2_2}{2 \sigma ^2}  \!+\! \ln \left(\frac{2TK(\lambda\!+\!1)}{\delta}  \right)\right]} 
$}

with $\Q_\theta = \Ncal(\theta, \sigma^2 I_d)$, and $\hthetatilde \sim \Q_\theta$ with parameters $\thetatilde$, and $\P=\Ncal(\thetaP, \sigma^2 I_d)$, and $\sigma \in [0,1]$, and $\lambda > 0$, and $\alpha \in (0,1]$, and $\delta \in [0,1]$.

The definition of $\empBoundCorDeux$ comes from the following corollary of \Cref{thm:bound-for-one-example}.
\begin{restatable}{corollary}{corboundoneexample}
\label{cor:bound-for-one-example-dis-expe}
For any distribution $\D$ over $\XY$,
for any $\lambda > 0$,
for any number of epochs $T$,
for any number of hyperparameter configuration $K$,
for any set of distributions $\mathcal{P} \in  \{\P_1, ..., \P_{T\times K}\}$, 
for any loss function $\ell: \Y  \times  \Y  \to  [0,1]$,
for any $\delta \in (0,1]$,
for any $\alpha   \in  (0,1)$,
for any algorithm $\Phi: (\XY)^m \times \M(\Hcal) \rightarrow \M(\Hcal)$,
for any $\sigma \in [0,1]$, 
with probability at least $1 {-} \delta$ \mbox{over $\S {\sim} \D^m$} and $h\sim\Q_\S$,
we have $\forall \P_t =\Ncal(\thetaP, \sigma^2 I_d) \in \mathcal{P}$,
\begin{align*}
&\EE_{\S' \sim \D^m} \riskempprime(h) \le \riskemp(h) + \sqrt{\frac{1}{2\,\alpha^2\,m} \left[\left(1 \!+\! \frac{1}{\lambda} \right) \frac{\lVert \thetatilde {-} \thetaP \rVert^2_2 - \lVert\thetatilde {-} \theta \rVert_2^2}{2 \sigma ^2}  \!+\! \ln \left(\frac{2TK(\lambda\!+\!1)}{\delta}  \right)\right]}.
\end{align*}
\end{restatable}
\begin{proof}
The proof follows the same steps as the proof of \Cref{cor:mca-pb-expe}.
\end{proof}

\paragraph{Instantiation of \Cref{thm:mhammedi}, and the objective function.}
We recall that, in practice, we compute an estimation of the bound of \Cref{thm:mhammedi} obtained by
sampling a single model from the posterior $\Q_\theta$ (since we deal with disintegrated bounds).
The associated objective function is

\fbox{
\begin{minipage}[c]{.98\textwidth}
$\displaystyle 
\empBoundMham(\riskemp(\hthetatilde), \Q_{\theta}, \thetatilde)   = \!\! \sup_{\substack{\rho \in \R^n_+\\ m \rho_\a \le \frac{1}{\alpha}}} \sum_{\a=1}^{m} \rho_\a  \ell(y_\a, \hthetatilde(\xbf_\a))$\\
$\displaystyle \phantom{\empBoundMham(\riskemp(\hthetatilde), \Q_{\theta}, \thetatilde)   =} +\, 2\, \!\!\sup_{\substack{\rho \in \R^n_+\\ m \rho_\a \le \frac{1}{\alpha}}} \sum_{\a=1}^{m} \rho_\a \ell(y_\a, \hthetatilde(\xbf_\a))  \left[  \left(   \sqrt{\frac{\ln\frac{2 TK\lceil \log_2(\frac{m}{\alpha}) \rceil}{\delta}}{2\, m\,\alpha }}  {+}  \frac{\ln\frac{2 T K\lceil \log_2(\frac{m}{\alpha}) \rceil}{\delta}}{3\, m\,\alpha}  \right) \right]$\\[1mm]
$ \displaystyle \phantom{\empBoundMham(\riskemp(\hthetatilde), \Q_{\theta}, \thetatilde)   =}
+ \sqrt{\frac{27}{{5\, m\, \alpha}}  \, \sup_{\substack{\rho \in \R^n_+\\ m \rho_\a \le \frac{1}{\alpha}}} \sum_{\a=1}^{m} \rho_\a  \ell(y_\a, \hthetatilde(\xbf_\a)) \left[\frac{\lVert \theta - \thetaP \rVert^2_2}{2\sigma^2} {+} \ln\tfrac{2 T K\lceil \log_2(\frac{m}{\alpha}) \rceil}{\delta}\right]}  $\\[1mm]
$ \displaystyle \phantom{\empBoundMham(\riskemp(\hthetatilde), \Q_{\theta}, \thetatilde)   =}
        + \frac{27}{{5\, m\, \alpha}} 
      \left[\frac{\lVert \theta - \thetaP \rVert^2_2}{2\sigma^2}{+} \ln\frac{2T K\lceil \log_2(\frac{m}{\alpha}) \rceil}{\delta}\right]
$
\end{minipage}
}

with $\Q_\theta = \Ncal(\theta, \sigma^2 I_d)$, with $\hthetatilde \sim \Q_\theta$ with parameters $\thetatilde$, and $\P=\Ncal(\thetaP, \sigma^2 I_d)$, and $\sigma \in [0,1]$, and $\lambda > 0$, and $\alpha \in (0,1]$, and $\delta \in [0,1]$.

 The definition of $\empBoundMham$ comes from the following corollary of \Cref{thm:mhammedi}. 
\begin{restatable}{corollary}{cormhammedi}
\label{cor:mhammedi-expe}
For any distribution $\D$ over $\XY$,
for any prior $\P \! \in \! \M(\Hcal)$, 
for any loss $\ell :  \Y \!\times \Y \!\rightarrow [0,1]$,
for any $\alpha \!\in \!(0,1]$,
for any $\delta\!\in\!(0,1]$,
with probability at least $1 {-} \delta$ over $\S {\sim} \D^m$, we have $\forall \Q =\Ncal(\theta, \sigma^2 I_d)$ and $\forall \P_t =\Ncal(\thetaP, \sigma^2 I_d) \in \mathcal{P}$,
\begin{align*}
        &\nonumber \EE_{h \sim \Q}\, \risk(h) \le \riskemp(\Q) + 2\, \riskemp(\Q)\! \left[ \! \sqrt{\frac{1}{2 \alpha m }\ln\frac{2 T K \lceil \log_2[\frac{m}{\alpha}] \rceil}{\delta}}  \!+\!  \frac{1}{3m\alpha}\ln\frac{2 T K \lceil \log_2[\frac{m}{\alpha}] \rceil}{\delta}\! \right]\\
          &\phantom{ \EE_{h \sim \Q}\, \risk(h) \le \riskemp(\Q)}+ \sqrt{\frac{27}{{5 \alpha m}}\riskemp(\Q)\!\left[\frac{\lVert \theta - \thetaP \rVert^2_2}{2\sigma^2} {+} \ln\tfrac{2  T K \lceil \log_2(\frac{m}{\alpha}) \rceil}{\delta}\right]  
        } + \frac{27}{5 \alpha m} 
      \left[\frac{\lVert \theta - \thetaP \rVert^2_2}{2\sigma^2} {+} \ln\frac{2 T K \lceil \log_2(\frac{m}{\alpha}) \rceil}{\delta}\right], \\[3mm]
&\text{where}\ \ \EE_{h\sim\Q} \risk(h) \defeq \EE_{h\sim\Q}\ \ \sup_{\rho \in E}\  \EE_{(\x,y)\sim \rho}\  \ell(y, h(\x)), \quad \text{with}\ \  E\!=\!\left\{ \ \rho  \ \big| \  \rho \ll \D, \text{ and } \frac{d\rho}{d\D} \le \frac{1}{\alpha}\right\},\\
&\text{and}\ \ \riskemp(\Q) \defeq \sup_{\rho \in \widehat{E}}\ \ \sum_{i=1}^m \rho_\a \EE_{h\sim\Q} \ell(y_\a, h(\x_\a)),  \quad \text{with}\ \ \widehat{E}\!=\!\left\{ \ \rho  \ \big| \ \forall \a \in \A, \ \frac{d\rho_\a}{d\pi_\a} \le \frac{1}{\alpha}\right\} \quad \text{and } \pi_\a = \frac{1}{m}.
\end{align*}
\end{restatable}
\begin{proof}
The proof follows the same steps as the proof of \Cref{cor:mca-pb-expe}. 
\end{proof}

\subsubsection{Additional parameters studied during our experiments} 
\label{sec:param}
In Appendix~\ref{sec:additional-expe}, we present the complete results of our experiments with CVaR, and an additional constrained $f$-entropic risk measure, EVaR defined by \Cref{def:constrained-f-entropic} with the function $f(x) =x \ln x$ extended by continuity at $x=0$ with $f(0) = 0$, and $\beta = - \ln \alpha$.

The different settings we considered are (the rest of the setting follows \Cref{sec:expe}):
\vspace*{-3mm}
\begin{itemize*}
\item Two model architectures: a 2-hidden-layer multilayer perceptron and a perceptron.\\[-3mm]
\item When $|\A|\!\leq\!m$ (a subgroup corresponds to a class), for \Cref{cor:mca-dis}:
\begin{itemize*}
    \item Two reference distributions $\pi$: The class ratio, and the uniform distribution,
    \item Two risks: CVaR and EVaR.\\[-3mm]
\end{itemize*}
\item When $|\A|\!=\!m$ (a subgroup corresponds to a single example), for \Cref{thm:bound-for-one-example}:
\begin{itemize*}
    \item One reference distribution: The uniform distribution,
    \item Two risks: CVaR and EVaR,
    \item Two values of parameter $\lambda$: $\lambda=1$ and $\lambda=m$.\\[-3mm]
\end{itemize*}
\item When $|\A|\!=\!m$ (a subgroup corresponds to a single example), for \Cref{thm:mhammedi}:
\begin{itemize*}
 \item One reference distribution: The uniform distribution,
\item One risk: CVaR (since \Cref{thm:mhammedi} is only defined for CVaR).
\end{itemize*}
\end{itemize*}

\subsection{Datasets}
\label{subsec:datasets-details}
We perform our experiments on 19 datasets taken from OpenML \citep{vanschoren2013openml}. 
Their main characteristics are summarized in \Cref{tab:real-datasets}.

\def\arraystretch{1}
\begin{table}[h!]
    \centering
     \caption{Main characteristics of the datasets (* means that the classes are uniformly distributed).}
    \begin{footnotesize}
    \begin{tabular}{lrrrrr}
    \toprule
        dataset & n examples & n features & n classes & class ratio \\
        \midrule
        australian & 690 & 14 & 2 & 0.56/0.44 \\
        balance & 625 & 4 & 3 & 0.08/0.46/0.46 \\
        german & 1,000 & 20 & 2 & 0.3/0.7 \\
        heart & 270 & 13 & 2 & 0.56/0.44 \\
        iono & 351 & 34 & 2 & 0.36/0.64 \\
        letter & 20,000 & 16 & 26 & 0.04* \\
        mammography & 11,183 & 6 & 2 & 0.98/0.02 \\
        newthyroid & 215 & 5 & 3 & 0.7/0.16/0.14 \\
        oilspill & 937 & 49 & 2 & 0.96/0.04 \\
        pageblocks & 5473 & 10 & 5 & 0.9/0.06/0.01/0.02/0.02 \\
        pendigits & 10,992 & 16 & 10 & 0.1* \\
        phishing & 11,055 & 68 & 2 & 0.44/0.56 \\
        prima & 768 & 8 & 2 & 0.65/0.35 \\
        satimage & 6,430 & 36 & 6 & 0.24/0.11/0.21/0.1/0.11/0.23 \\
        segment & 2,310 & 19 & 7 & 0.14* \\
        spambase & 4,601 & 57 & 2 & 0.61/0.39 \\
        spectfheart & 267 & 44 & 2 & 0.21/0.79 \\
        splice & 3,190 & 287 & 3 & 0.24/0.24/0.52 \\
        wdbc & 569 & 30 & 2 & 0.63/0.37 \\
        \bottomrule
    \end{tabular}

    \end{footnotesize}
    \label{tab:real-datasets}
\end{table}

\section{RESULTS OF THE ADDITIONAL EXPERIMENTS}
\label{sec:additional-expe}

In the main paper, we reported the main behaviors we observed on a representative subset of our experiments (on the four most imbalanced datasets).
For completeness, the following pages provide all figures for every parameter setting and dataset, as described in Appendices~\ref{sec:param} and~\ref{subsec:datasets-details}.
Below, we summarize the main trends across all experiments.

\subsection*{Results in a nutshell.}

\textbf{On the role of $\alpha$.}
On the one hand, across all bar plots (Figures~\ref{fig:appendix-risk-cvar-neural}, \ref{fig:appendix-risk-cvar-perceptron}, \ref{fig:appendix-risk-evar-neural}, \ref{fig:appendix-risk-evar-linear}), we observe that  $\alpha$ strongly influences the tightness of all the bounds: higher values of $\alpha$ imply tighter bounds. 
As discussed in \Cref{sec:expe}, this is not only due to the factor $\frac{1}{\alpha}$ or $\frac{1}{\alpha^2}$ in the bounds but also because a larger $\alpha$ makes the CVaR tighter.
In consequence, the tightest bound values, which always correspond to the highest $\alpha=0.9$, do not lead to the best-performing models across the subgroups (in terms of F-score or in terms in class-wise error rates).

On the other hand, $\alpha$ also plays an important role on the performance across the subgroups.
Indeed, as we can see across all the bar plots, the best F-scores rarely coincide with the tightest bound values (68 times over 76), and
Figures~\ref{fig:appendix-risk-cvar-neural}, \ref{fig:appendix-risk-cvar-perceptron}, \ref{fig:appendix-risk-evar-neural}, and \ref{fig:appendix-risk-evar-linear} show that the class-wise error rates evolve with $\alpha$, showing that adjusting $\alpha$ can help to balance the performances across the subgroups.

\textbf{On the comparison with \citet{mhammedi2020pac} (one example per group setting).}
As expected, when comparing \Cref{thm:mhammedi,thm:bound-for-one-example} (which rely on the same subgroups), our bound of \cref{thm:bound-for-one-example} is generally tighter (or very close) for all values of $\alpha$. 
Note that we can observe that $\lambda=m$ leads always to bounds that are slightly higher than those of $\lambda=1$, but although this has a slight impact on the tightness of the bound, it does not change the overall behavior.

\textbf{On the role of $\pi$ for \Cref{cor:mca-dis}.}
The reference distribution $\pi$ also plays a role in the tightness of the bound. 
Except for the most balanced datasets (\textit{australian}, \textit{heart}, \textit{letter}, \textit{pendigits}, \textit{phishing}, \textit{segment}), where using a uniform $\pi$ or the class ratio $\pi$ yields similar results as expected, we observe that bounds computed with a uniform $\pi$ are generally (and sometimes significantly) looser than those computed with $\pi$ set to the class ratio. 
Remarkably, for $\alpha\!\in\!\{0.01,0.1,0.3\}$, \Cref{cor:mca-dis} with $\pi$ set to the class ratio continues to give non-vacuous and competitive bounds, even when $\alpha$ is relatively high, despite the $\frac{1}{\alpha\ma}$ term in the bound.
This suggests that choosing a reference $\pi$ that reflects the imbalance in the data can lead to better capturing the under-representation in the data while keeping guarantees

\textbf{On the performances.}
Interestingly, the bound of \Cref{cor:mca-dis} always leads to the best results in terms of F-score on the most imbalanced datasets (\textit{oilspill}, \textit{mammography}, \textit{balance}, \textit{pageblocks}, illustrating the usefulness of our subgroup-based approach.
For the 15 more balanced datasets, the bound of \Cref{cor:mca-dis} is always competitive, achieving the best performance in 9 cases (for each set of experiments), while the bound of \Cref{thm:bound-for-one-example} performs best 6 times.

\textbf{A note on the EVaR.}
The results obtained with EVaR are similar to the one observed with the CVaR.
This confirms that our bounds can be effectively applied to other constrained $f$-entropic risk measures beyond the CVaR.

\newpage 
\thispagestyle{empty}
\begin{landscape}
\begin{figure*}[p]
\includegraphics[width=\linewidth, keepaspectratio]{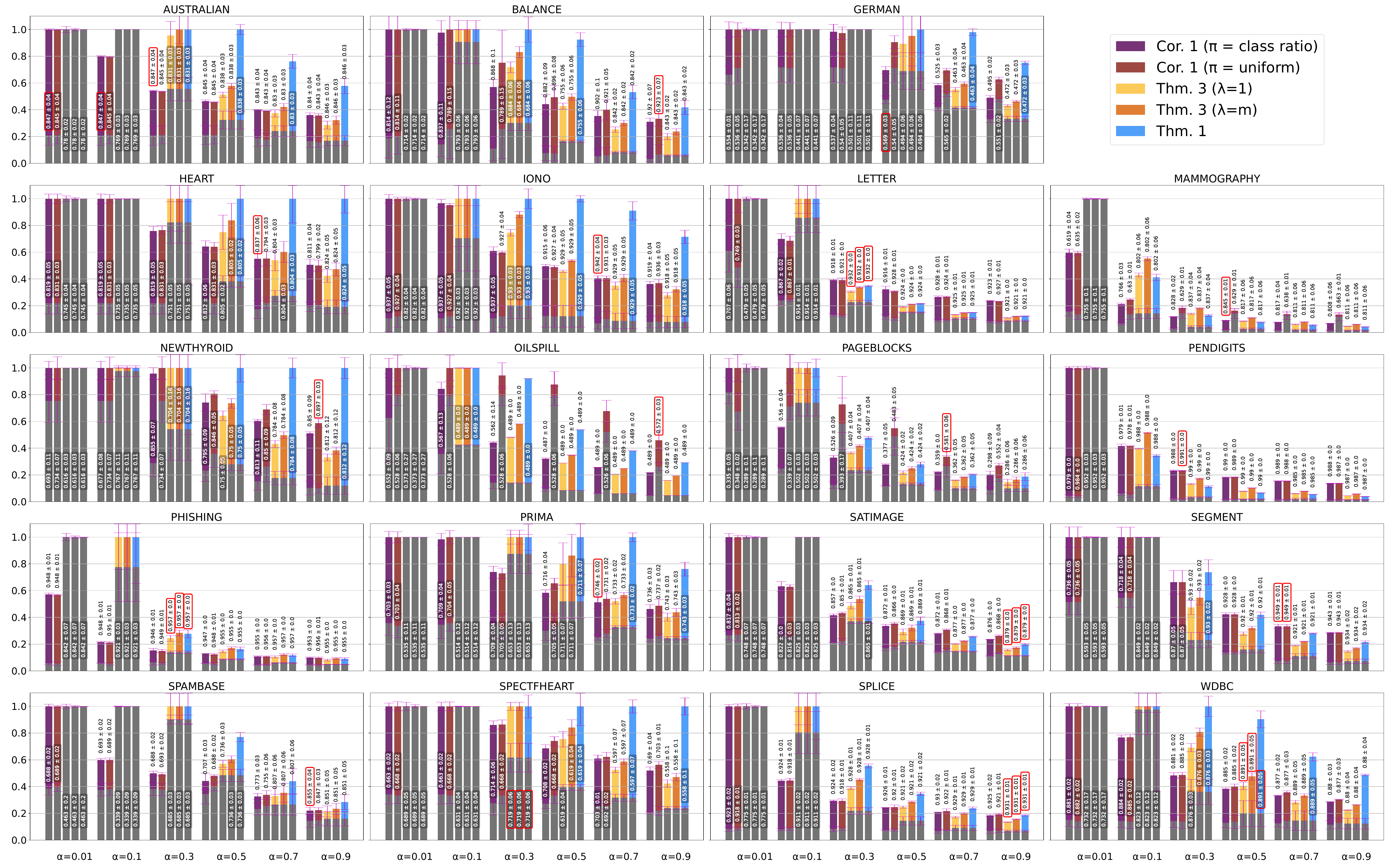}
 \caption{\textbf{2-hidden layer MLP with CVaR.} Bound values (in color), test risk $\risk_{\T}$ (in grey), and F-score value on $\T$ (with their standard deviations) for \Cref{thm:bound-for-one-example}, \Cref{cor:mca-dis}, and \Cref{thm:mhammedi}, as a function of $\alpha$ (on the $x$-axis).
    The $y$-axis corresponds to the value of the bounds and test risks.
    The highest F-score for each dataset is emphasized with a red frame.
    }
    \label{fig:appendix-cvar-neural}
\end{figure*}
\end{landscape}

\newpage 
\thispagestyle{empty}
\begin{landscape}
\begin{figure*}[p]
\includegraphics[width=\linewidth, keepaspectratio]{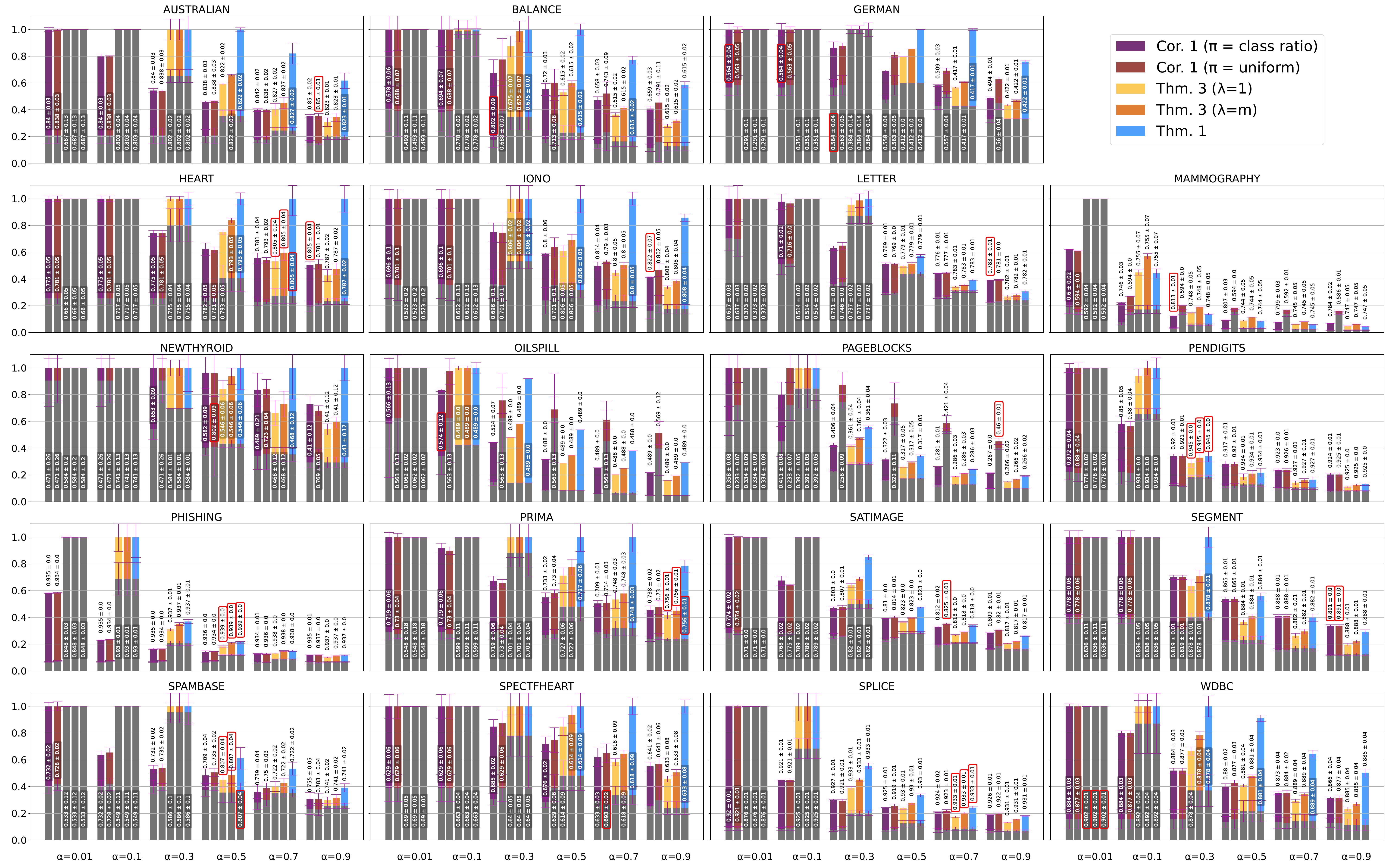}
\caption{\textbf{Perceptron with CVaR.} Bound values (in color), test risk $\risk_{\T}$ (in grey), and F-score value on $\T$ (with their standard deviations) for \Cref{thm:bound-for-one-example}, \Cref{cor:mca-dis}, and \Cref{thm:mhammedi}, as a function of $\alpha$ (on the $x$-axis).
    The $y$-axis corresponds to the value of the bounds and test risks.
    The highest F-score for each dataset is emphasized with a red frame.
    }
    \label{fig:appendix-cvar-linear}
\end{figure*}
\end{landscape}

\newpage 
\thispagestyle{empty}
\begin{landscape}
\begin{figure}[p]
  \centering
  \begin{subfigure}{0.48\linewidth} 
    \centering
    \includegraphics[width=\linewidth,height=0.9\textheight,keepaspectratio]{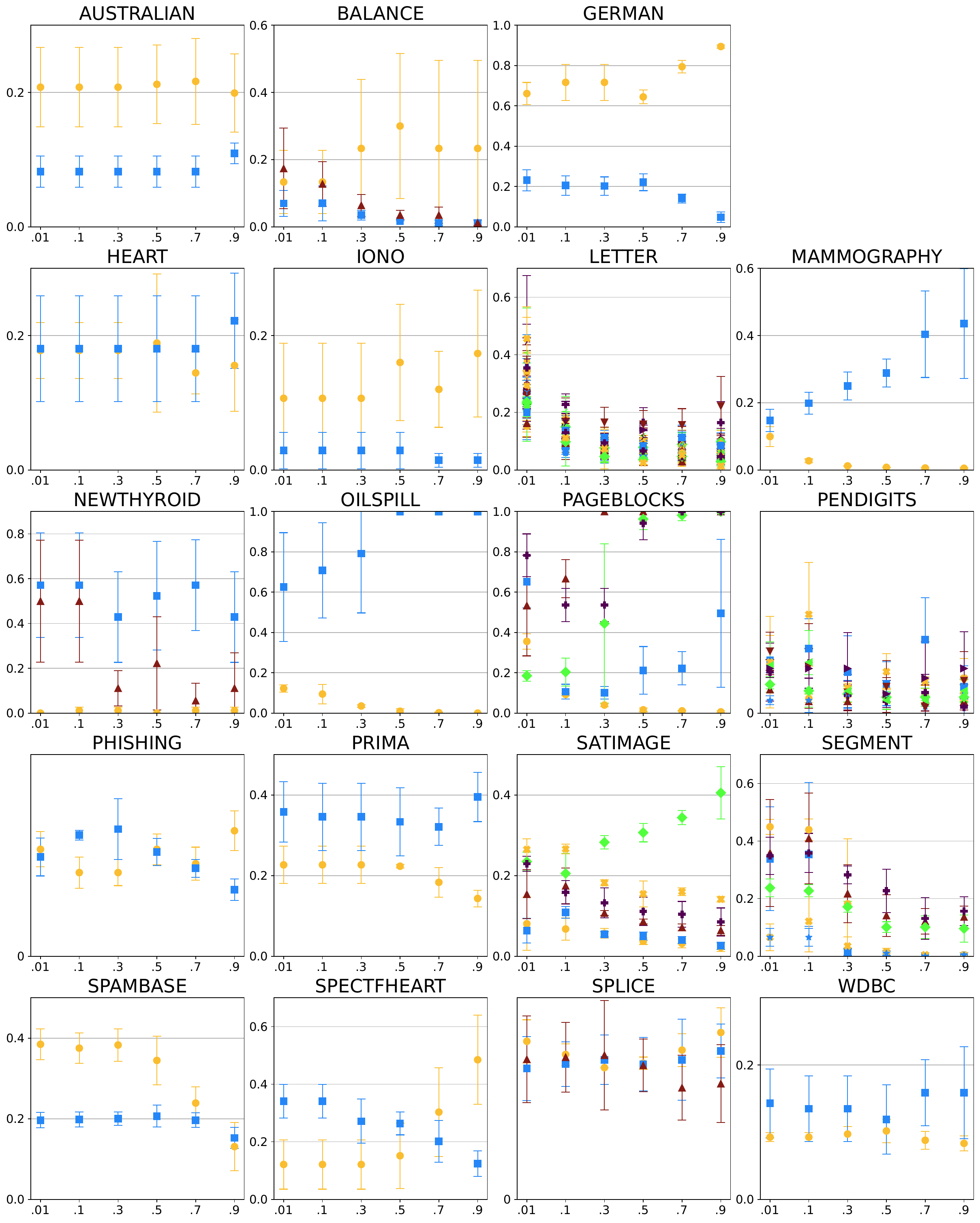}
    \caption{$\pi=$ Class ratio}
  \end{subfigure}\hfill
\begin{subfigure}{0.48\linewidth} 
    \centering
    \includegraphics[width=\linewidth,height=0.9\textheight,keepaspectratio]{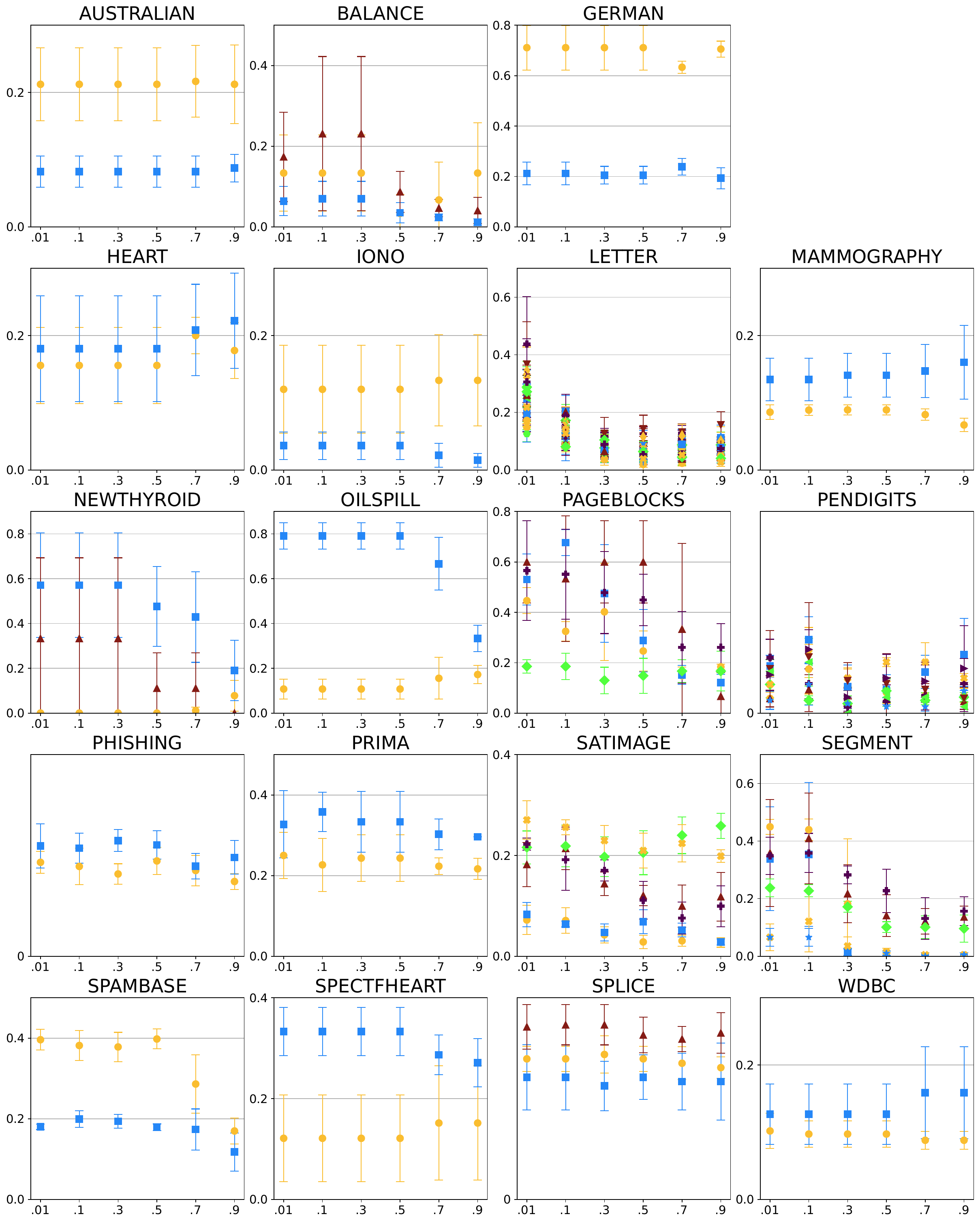}
    \caption{$\pi=$ Uniform}
  \end{subfigure}

  \caption{\textbf{2-hidden layer MLP with CVaR.} Evolution of the class-wise error rates and standard deviation on the set $\T$ ($y$-axis) as a function of the parameter $\alpha$ ($x$-axis) with \Cref{cor:mca-dis}. 
  Each class is represented by different markers and colors.
  }
  \label{fig:appendix-risk-cvar-neural}
\end{figure}
\end{landscape}

\newpage 
\thispagestyle{empty}
\begin{landscape}
\begin{figure}[p]
  \centering
   \begin{subfigure}{0.48\linewidth} 
    \centering
    \includegraphics[width=\linewidth,height=0.9\textheight,keepaspectratio]{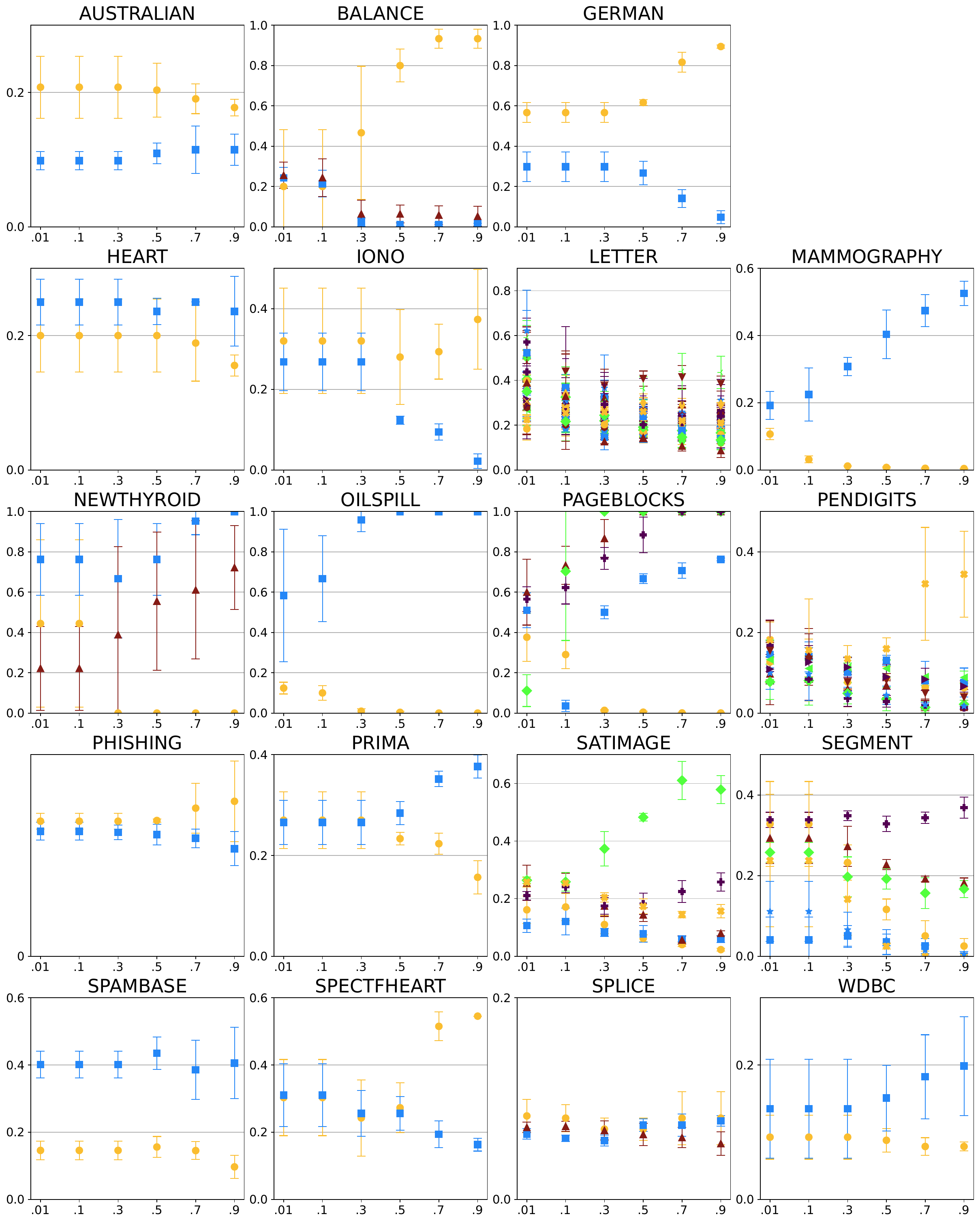}
    \caption{$\pi=$ Class ratio}
  \end{subfigure} \hfill
    \begin{subfigure}{0.48\linewidth} 
    \centering
    \includegraphics[width=\linewidth,height=0.9\textheight,keepaspectratio]{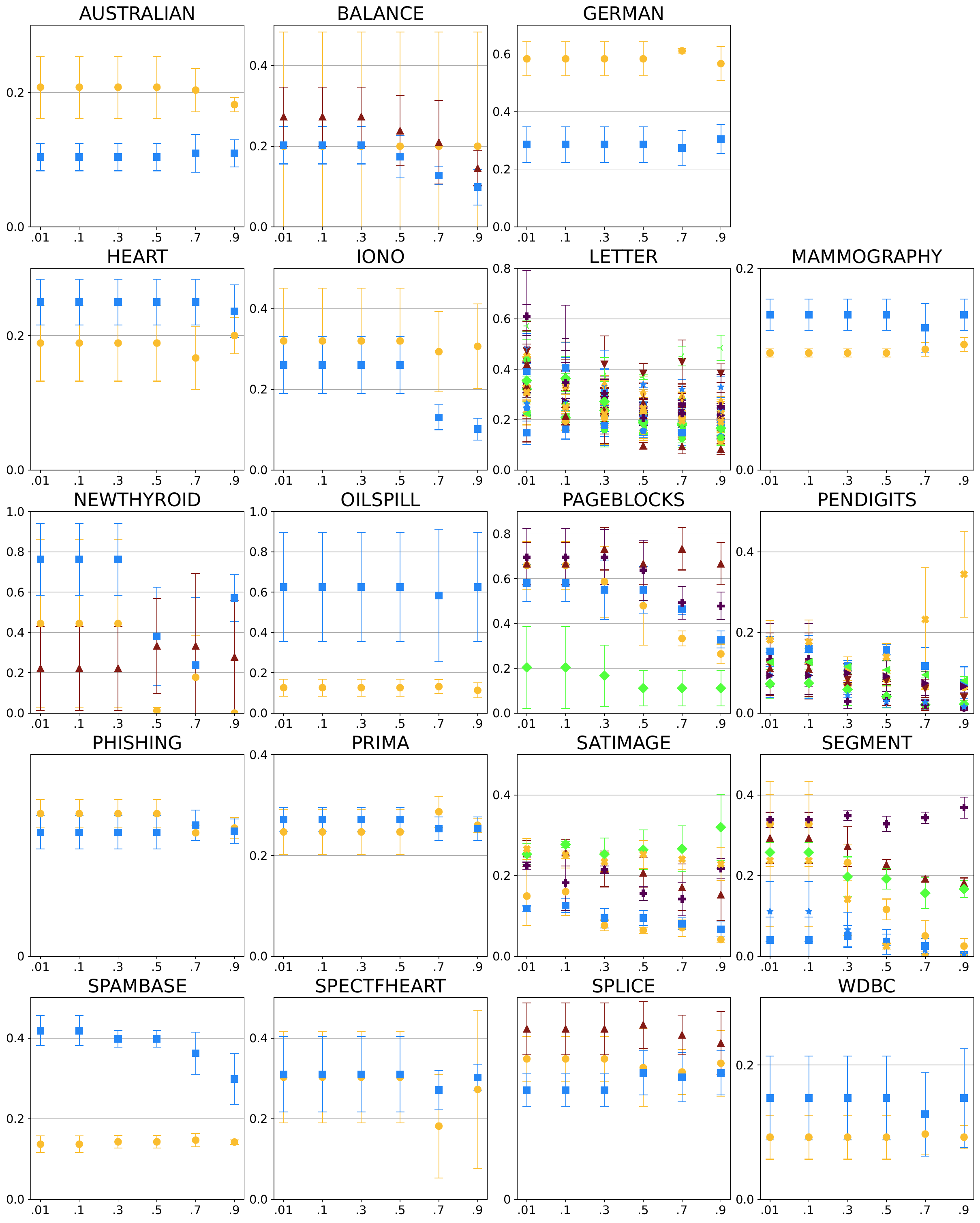}
    \caption{$\pi=$ Uniform}
  \end{subfigure}

  \caption{\textbf{Perceptron with CVaR.} Evolution of the class-wise error rates and standard deviation on the set $\T$ ($y$-axis) as a function of the parameter $\alpha$ ($x$-axis) with \Cref{cor:mca-dis}. 
  Each class is represented by different markers and colors.
  }
  \label{fig:appendix-risk-cvar-perceptron}
\end{figure}
\end{landscape}

\newpage 
\thispagestyle{empty}
\begin{landscape}
\begin{figure*}[p]
\includegraphics[width=\linewidth, keepaspectratio]{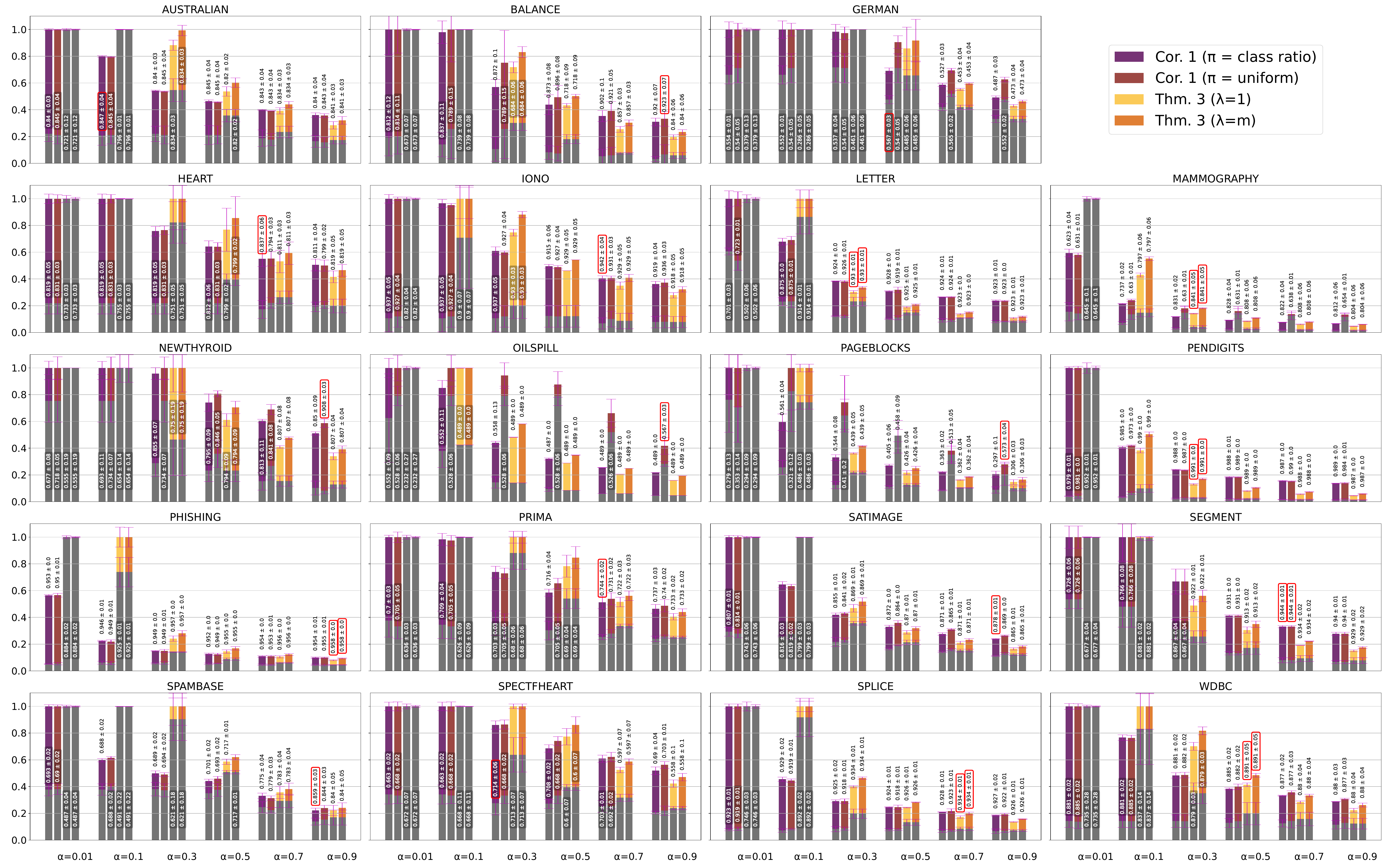}
\caption{\textbf{2-hidden layer MLP with EVaR.} Bound values (in color), test risk $\risk_{\T}$ (in grey), and F-score value on $\T$ (with their standard deviations) for \Cref{thm:bound-for-one-example}, \Cref{cor:mca-dis}, and \Cref{thm:mhammedi}, as a function of $\alpha$ (on the $x$-axis).
    The $y$-axis corresponds to the value of the bounds and test risks.
    The highest F-score for each dataset is emphasized with a red frame.
    }
    \label{fig:appendix-evar-neural}
\end{figure*}
\end{landscape}

\newpage 
\thispagestyle{empty}
\begin{landscape}
\begin{figure*}[p]
\includegraphics[width=\linewidth, keepaspectratio]{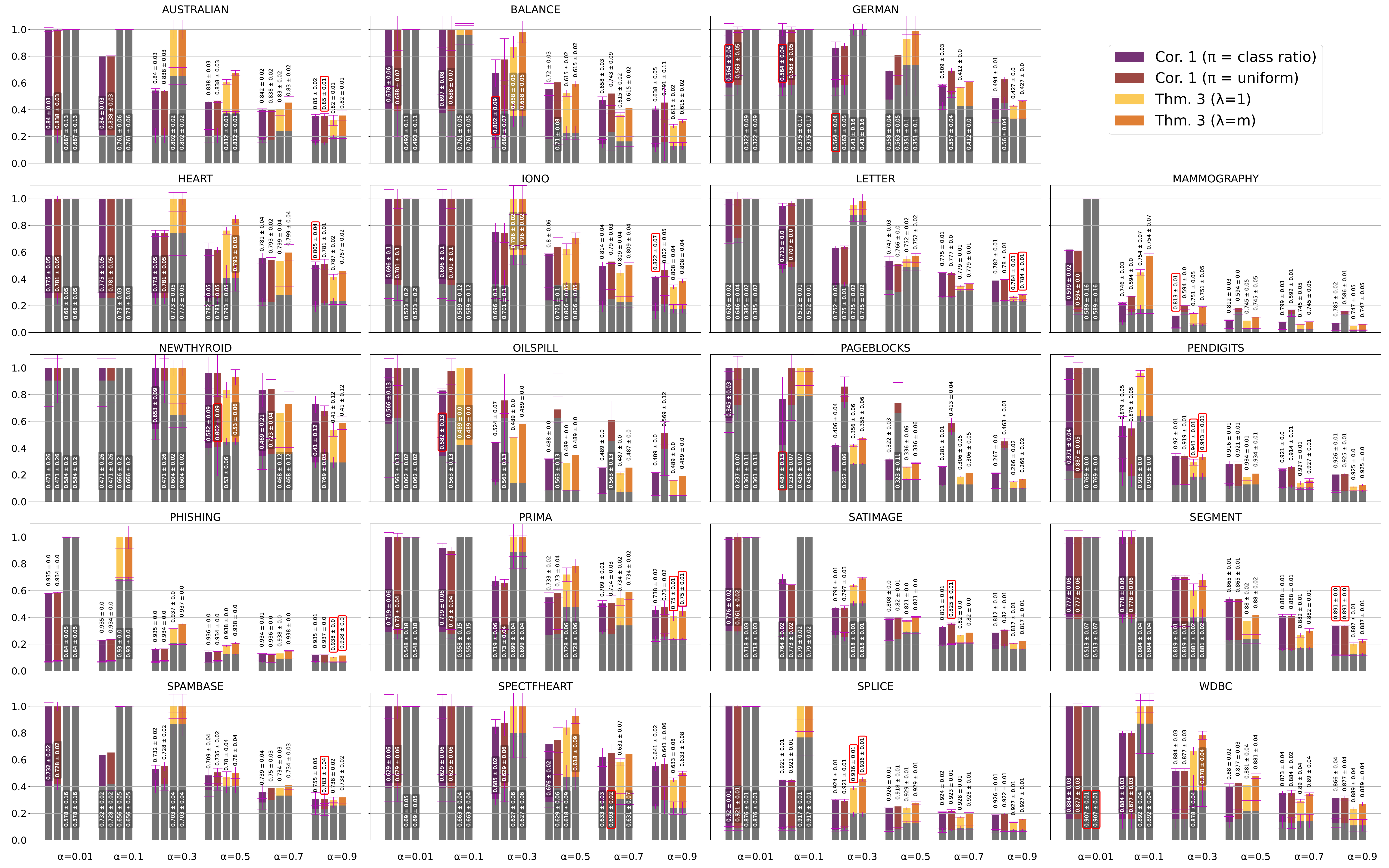}
\caption{\textbf{Perceptron with EVaR.} Bound values (in color), test risk $\risk_{\T}$ (in grey), and F-score value on $\T$ (with their standard deviations) for \Cref{thm:bound-for-one-example}, \Cref{cor:mca-dis}, and \Cref{thm:mhammedi}, as a function of $\alpha$ (on the $x$-axis).
    The $y$-axis corresponds to the value of the bounds and test risks.
    The highest F-score for each dataset is emphasized with a red frame.
    }
    \label{fig:appendix-evar-linear}
\end{figure*}
\end{landscape}

\newpage 
\thispagestyle{empty}
\begin{landscape}

\begin{figure}[p]
  \centering
  \begin{subfigure}{0.48\linewidth} 
    \centering
    \includegraphics[width=\linewidth,height=0.9\textheight,keepaspectratio]{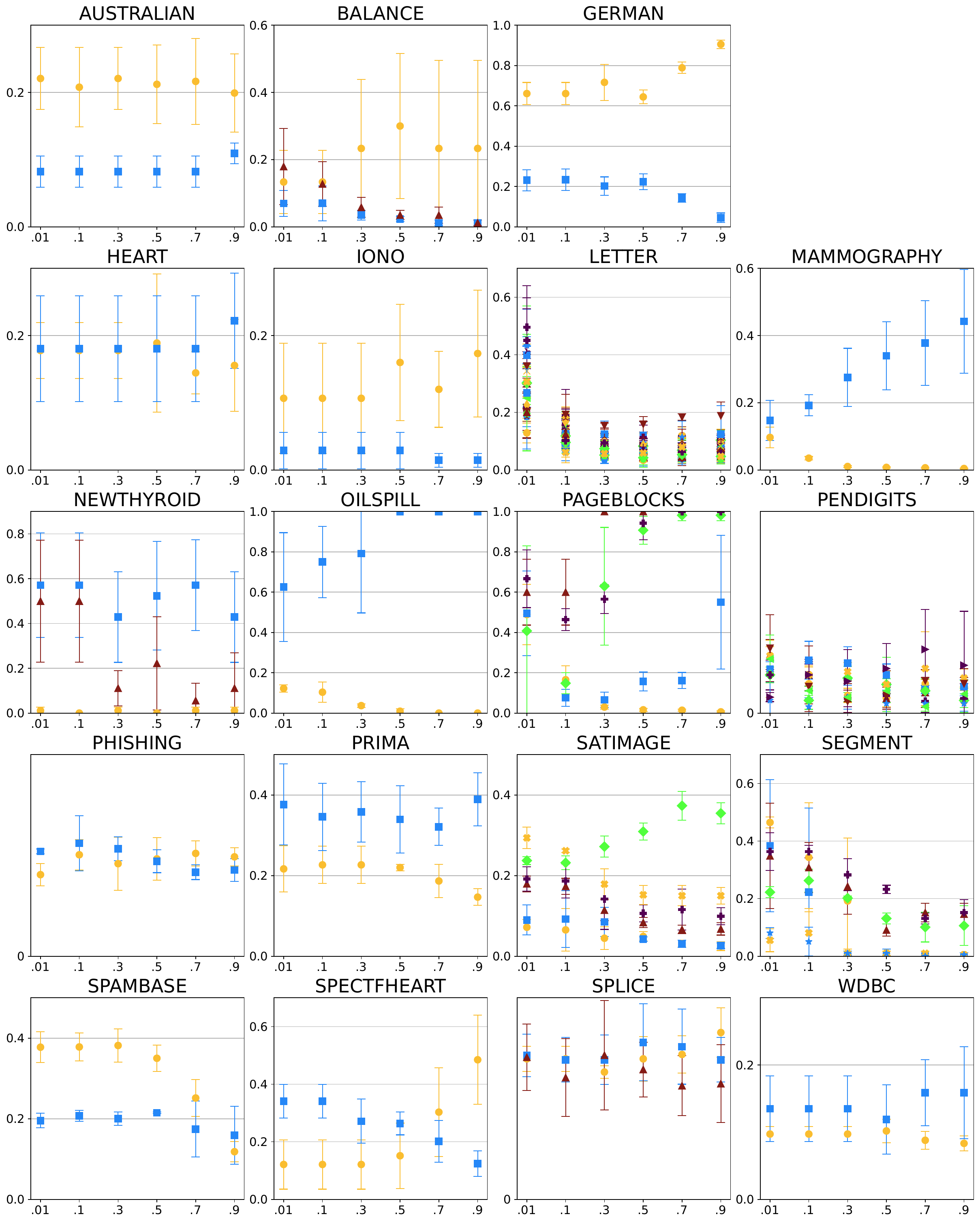}
    \caption{$\pi=$ class ratio}
  \end{subfigure}\hfill
  \begin{subfigure}{0.48\linewidth} 
    \centering
    \includegraphics[width=\linewidth,height=0.9\textheight,keepaspectratio]{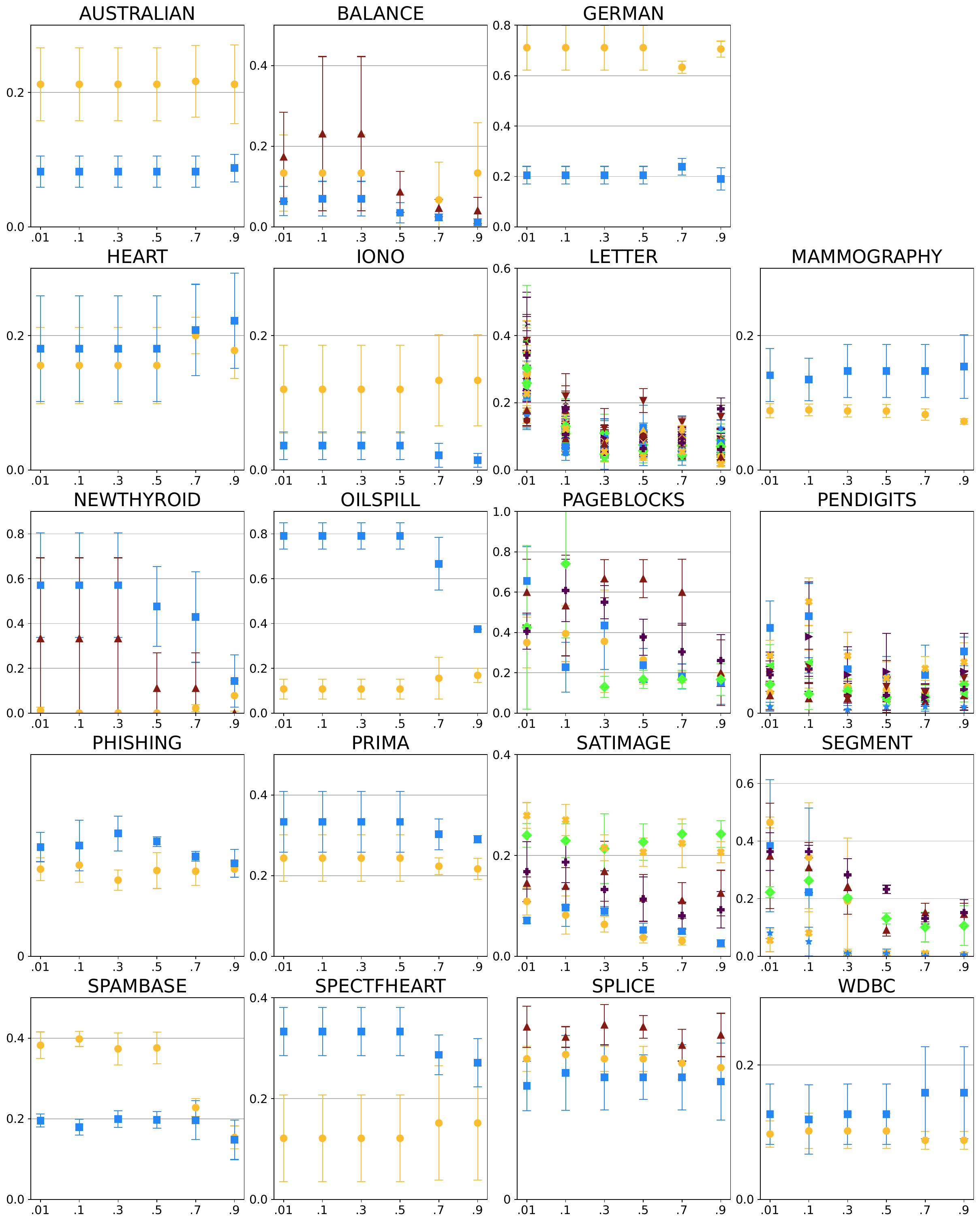}
    \caption{$\pi=$ uniform}
  \end{subfigure}
  \caption{\textbf{2-hidden layer MLP with EVaR.} Evolution of the class-wise error rates and standard deviation on the set $\T$ ($y$-axis) as a function of the parameter $\alpha$ ($x$-axis) with \Cref{cor:mca-dis}. 
  Each class is represented by different markers and colors.
  }
  \label{fig:appendix-risk-evar-neural}
\end{figure}
\end{landscape}

\newpage 
\thispagestyle{empty}
\begin{landscape}
\begin{figure}[p]
  \centering
  \begin{subfigure}{0.48\linewidth} 
    \centering
    \includegraphics[width=\linewidth,height=0.9\textheight,keepaspectratio]{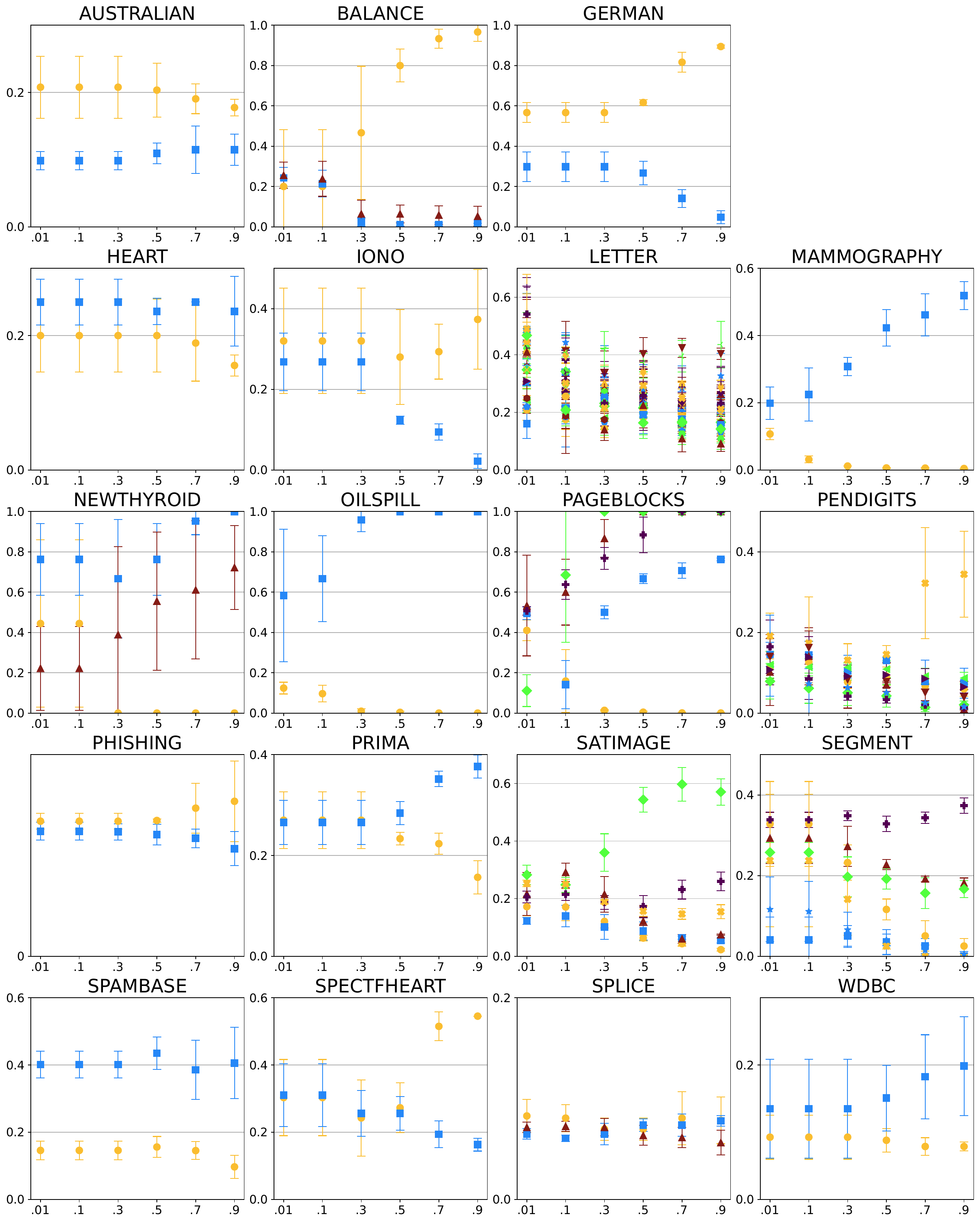}
    \caption{$\pi=$ class ratio}\hfill
  \end{subfigure}\hfill
  \begin{subfigure}{0.48\linewidth} 
    \centering
    \includegraphics[width=\linewidth,height=0.9\textheight,keepaspectratio]{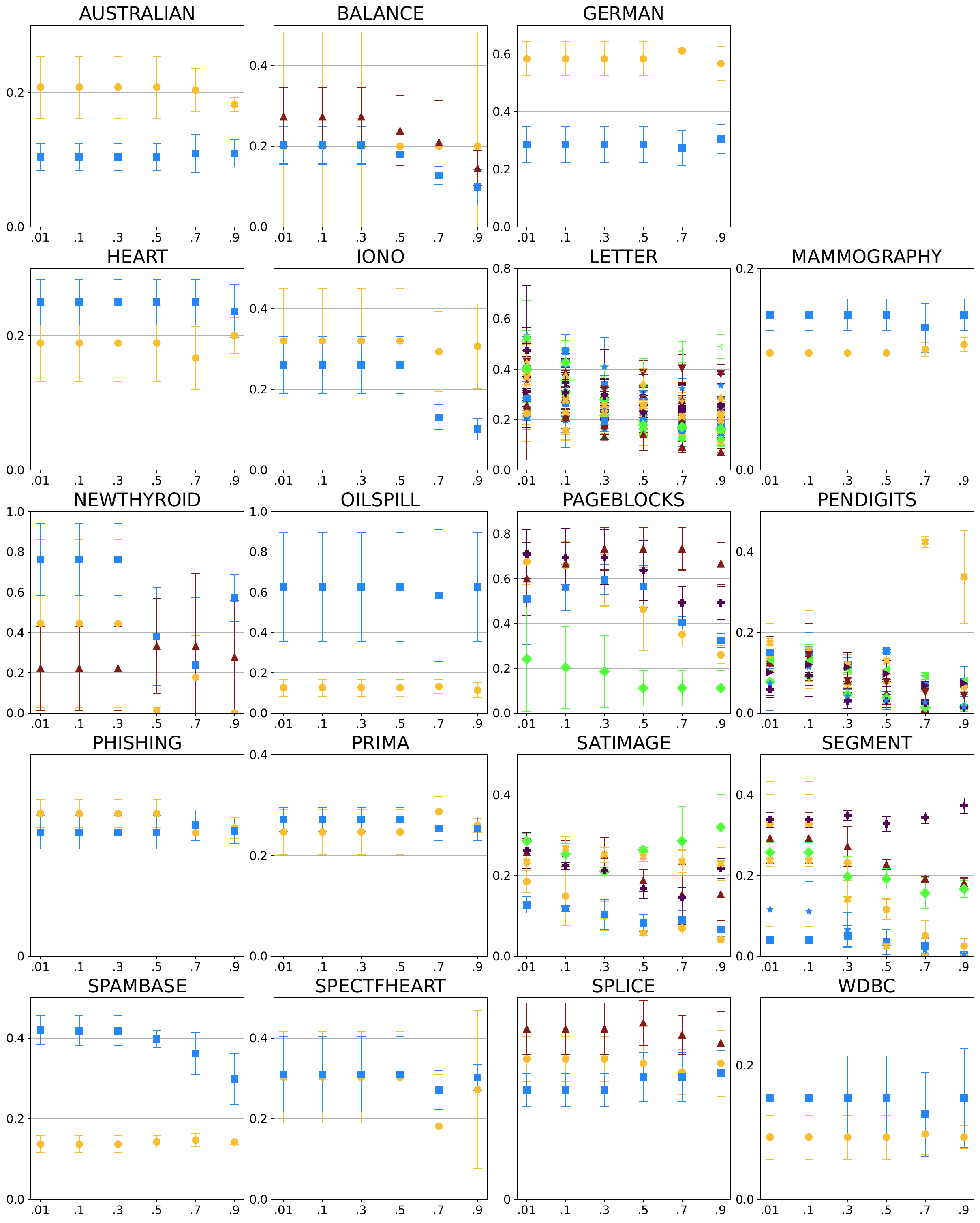}
    \caption{$\pi=$ uniform}
  \end{subfigure}
  \caption{\textbf{Perceptron MLP with EVaR.} Evolution of the class-wise error rates and standard deviation on the set $\T$ ($y$-axis) as a function of the parameter $\alpha$ ($x$-axis) with \Cref{cor:mca-dis}. 
  Each class is represented by different markers and colors.
  }
  \label{fig:appendix-risk-evar-linear}
\end{figure}
\end{landscape}

\end{document}